\documentclass{article}

\usepackage[final]{neurips_2020}

\usepackage[utf8]{inputenc} % allow utf-8 input
\usepackage[T1]{fontenc}    % use 8-bit T1 fonts
\usepackage{hyperref}       % hyperlinks
\usepackage{url}            % simple URL typesetting
\usepackage{booktabs}       % professional-quality tables
\usepackage{amsfonts}       % blackboard math symbols
\usepackage{nicefrac}       % compact symbols for 1/2, etc.
\usepackage{microtype}      % microtypography
\usepackage{subcaption,booktabs}

\usepackage{booktabs}
\usepackage{caption}
\usepackage{color}
\usepackage{amsfonts}
\usepackage{amsmath}
\usepackage{algorithm}
\usepackage{algorithmic}
\usepackage{graphicx}
\usepackage{nicefrac}
\usepackage{bbm}
\usepackage{amsthm}
\usepackage{wrapfig}
\usepackage{tikz}
\usepackage[titletoc,title]{appendix}
\usepackage{stmaryrd}
\usepackage{amsmath}
\usepackage{enumitem}
\usepackage{mathtools}
\usepackage{subcaption,booktabs}

\newtheorem{lemma}{Lemma}
\newtheorem{corollary}{Corollary}
\newtheorem{theorem}{Theorem}
\newtheorem{proposition}{Proposition}

\newtheorem{claim}{Claim}

\def\be{\begin{equation}}
\def\ee{\end{equation}}
\def\beas{\begin{eqnarray*}}
	\def\eeas{\end{eqnarray*}}
\def\bea{\begin{eqnarray}}
\def\eea{\end{eqnarray}}

\newcommand{\inprod}[2]  {\left\langle{#1},{#2}\right\rangle}
\newcommand*{\SQUEEZE}{}

\newcommand{\x}{{\mathbf x}}
\newcommand{\y}{{\mathbf y}}

\newcommand{\uu}{{\mathbf u}}
\newcommand{\vv}{{\mathbf v}}

\newcommand{\aaa}{{\mathbf a}}
\newcommand{\bb}{{\mathbf b}}

\newcommand{\0}{{\mathbf 0}}
\newcommand{\A}{{\mathcal A}}

\newcommand{\M}{{\mathcal M}}

\newcommand{\C}{{\mathbb C}}

\newcommand{\R}{{\mathbb R}}
\newcommand{\N}{{\mathbb N}}
\newcommand{\nocontentsline}[3]{}
\newcommand{\tocless}[2]{\bgroup\let\addcontentsline=\nocontentsline#1{#2}\egroup}

\newcommand{\abs}[1]{\left\lvert#1 \right\rvert}
\newcommand{\norm}[1]{\left\|#1 \right\|}

\newcommand{\cupdot}{\mathbin{\mathaccent\cdot\cup}}

\newcommand{\mat}[1]{\llbracket#1\rrbracket}
\newcommand{\sep}[2]{\mathrm{sep}_{(#1)}\left( #2 \right)}
\newcommand{\rank}[1]{\mathrm{rank}\left( #1 \right)}

\def\multiset#1#2{\ensuremath{\left(\kern-.3em\left(\genfrac{}{}{0pt}{}{#1}{#2}\right)\kern-.3em\right)}}
\newcommand{\q}{{\mathbf q}}
\newcommand{\eg}{\emph{e.g.}}
\newcommand{\ie}{\emph{i.e.}}

\newcommand{\wrt}{w.r.t.}

\title{The Depth-to-Width Interplay in Self-Attention}

\author{Yoav Levine, Noam Wies, Or Sharir, Hofit Bata, and Amnon Shashua\\
	The Hebrew University of Jerusalem\\
}
\begin{document}
	
	\maketitle
	\ifdefined\SQUEEZE \vspace{-4mm} \fi
	\begin{abstract}
		
		Self-attention architectures, which are rapidly pushing the frontier in natural language processing, demonstrate a surprising depth-\textbf{in}efficient behavior:
		previous works indicate that increasing the internal representation (network width) is just as useful as increasing the number of self-attention layers (network depth). 
		We theoretically predict a width-dependent transition between depth-efficiency and depth-\textbf{in}efficiency in self-attention. We conduct systematic empirical ablations on networks of depths $6$ to $48$ that clearly reveal the theoretically predicted behaviors, and provide explicit quantitative suggestions regarding the optimal depth-to-width allocation for a given self-attention network size.
		The race towards beyond 1-Trillion parameter language models renders informed guidelines for increasing self-attention depth and width in tandem an essential ingredient. 
		Our guidelines elucidate the depth-to-width trade-off in self-attention networks of sizes up to the scale of GPT3 (which we project to be too deep for its size), and beyond, marking an unprecedented width of 30K as optimal for a 1-Trillion parameter network.
	\end{abstract}
	
	\ifdefined\SQUEEZE \vspace{-4mm} \fi
	\tocless\section{Introduction \label{sec:intro}}
	\ifdefined\SQUEEZE \vspace{-2mm} \fi
	The golden age of deep learning has popularized the depth-efficiency notion: From an expressiveness standpoint, increasing a neural network's size by adding more layers (deepening) is advantageous relatively to other parameter increase alternatives, such as increasing the dimension of the internal representation (widening). 
	Beyond overwhelming empirical signals for this notion \citep{simonyan2014very,he2016deep}, depth-efficiency was theoretically supported from a variety of angles  \citep{cohen2016expressive,eldan2016power,raghu2017expressive,daniely2017depth}.
	
	Diminishing returns in the case of very deep networks were mainly attributed to optimization issues, and indeed the alleviation of these issues has allowed network depths to mount from $10$s to $100$s and beyond \citep{he2016deep}, enabling deep convolutional networks (ConvNets) to advance the state-of-the-art in computer vision applications.  
	However, as the field matured, a more nuanced perspective emerged. Empirical \citep{zagoruyko2016wide,wu2019wider} and theoretical \citep{lu2017expressive} studies suggest that the interplay between depth and width may be more subtle. Recently, a method for increasing width and depth in tandem (``EfficientNet" by  \citet{tan2019efficientnet}) has lead to the state-of-the-art on ImageNet while using a ConvNet with a fraction of the parameters used by previous leaders. Our work provides principled guidelines for increasing width and depth in tandem in self-attention networks.
	
	\begin{figure}
		\centering
		\includegraphics[width=\linewidth]{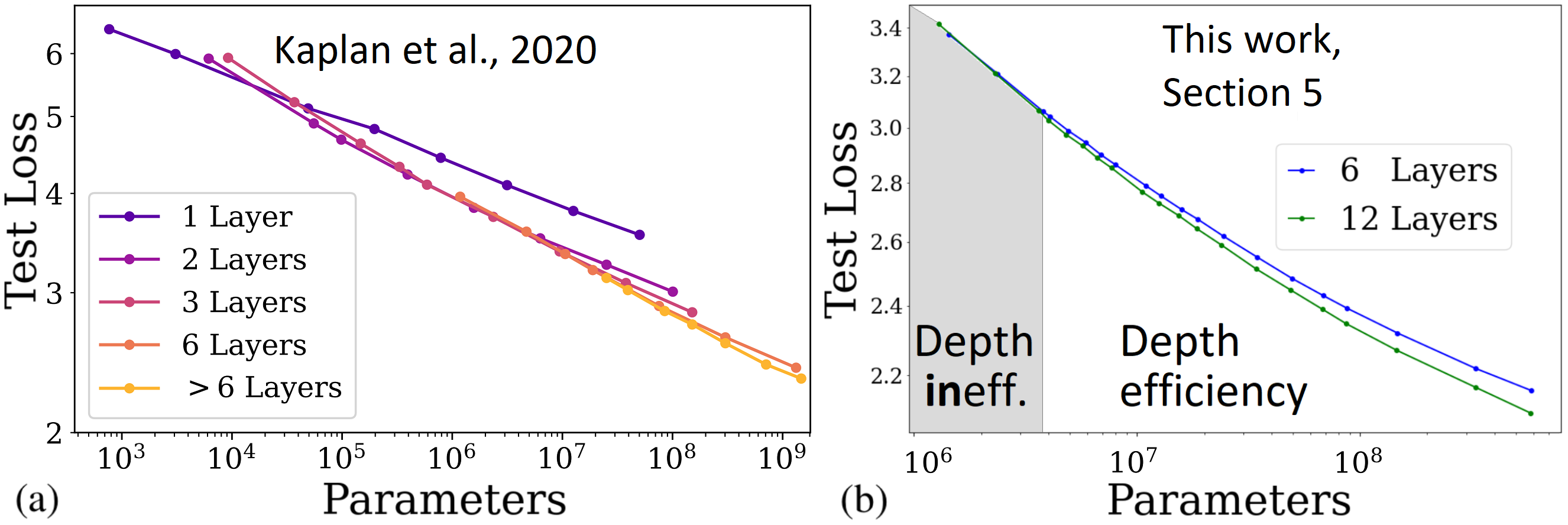}
		\vspace{-2mm} 
		\caption{
			\textbf{(a)} An ablation taken from figure~6 in \cite{kaplan2020scaling}, examining the perplexity scores of self-attention networks of varying depths and widths.
			 Experiments on the $L>6$ curve (yellow) all approximately obey the same improvement trend which depends only on the number of network parameters and not on the depth-to-width ratio. For $L\leq6$, depth-efficiency is clearly demonstrated, but due to the $L>6$ curve the authors conclude ``depth {\textbf{in}}efficiency" of self attention.
			\textbf{(b)} A representative of our experimental plots, which shows that a transition between depth-efficiency and inefficiency takes place, and that both regimes affect the behavior also at $L>6$.  Figure~\ref{fig:our_graph} shows that this trend continues at least up to depth $48$, and figure~\ref{fig:exponential} shows that the transition between regimes grows exponentially with depth, as predicted by our theory.    
		}
		\label{fig:no_depth_efficiency}
		\ifdefined\SQUEEZE \vspace{-4mm} \fi
	\end{figure}

	Since the introduction of the Transformer \citep{Transformer}, along with its encoder-only variant, BERT \citep{BERT}, self-attention based deep learning architectures have taken over the field of natural language processing~\citep{RoBERTa,GPT2,XLNet,T5,ELECTRA}. 
	However, in contrast to the depth ``arms race" that took place in the ConvNet case, the leading self-attention networks are not much deeper than the original BERT model. 
	In fact, even the strongest self-attention models trained to date, which increased the $0.3$B parameter count of BERT-large by factors of $100$s to $11$B~\citep{T5} and $175$B~\citep{GPT3}, have only increased its depth by factors of $2$ and $4$, respectively. 
	The remaining size increase stems from an increase in layer widths, clearly countering the depth-efficiency notion. 
	
	A recent empirical ablation study by \citet{kaplan2020scaling} provides support for the above signal. 
	Figure~\hyperref[fig:no_depth_efficiency]{~\ref{fig:no_depth_efficiency}(a)}, taken from this study, leads the authors to conclude that the overall (non-embedding) network size, given by $12\cdot L\cdot d_x^2$ where $L$ is the number of self-attention layers (network depth) and $d_x$ is the hidden representation dimension (network width), is the main predictor of performance regardless of the depth-to-width ratio. 
	This suggests that depth may not play as crucial a role in self-attention networks as it does in convolutional networks. 
	
	In this paper, we address the above question of the depth-to-width interplay in self-attention networks, and reveal fundamental subtleties in the above picture. 
	We predict that self-attention will exhibit two qualitatively different depth-efficiency and depth-\textbf{in}efficiency behaviors,  in two distinct parameter regimes, as depicted in figure~\hyperref[fig:no_depth_efficiency]{~\ref{fig:no_depth_efficiency}(b)}.
	After presenting our theoretical analysis in sections~\ref{sec:sa}-\ref{sec:results}, we provide a thorough empirical evaluation in section~\ref{sec:exp}, which validates our predicted trends for self-attention networks of depths $6$ to $48$. 
	Importantly, our theoretical and empirical results provide quantitative guidelines for optimal depth-to-width parameter allocation given a fixed parameter budget (see for example table~\ref{table:tab}). 
	The current challenge of reaching beyond $1$-Trillion parameter language models renders informed guidelines of how to increase self-attention depth and width in tandem a mandatory ingredient.
	Our results clearly show that the optimal path towards the $1$-Trillion parameter mark includes massive widening. 
	
	\tocless\subsection{An overview of our theoretical approach and findings}
	
	We analyze self-attention networks in which all non-linear activations and normalization operations are removed.
	Otherwise, the analyzed class (presented in section~\ref{sec:sa}) has the regular deep multi-headed Key/Query/Value structure of common self-attention.
	After presenting this class in detail, we point to recent studies which demonstrate that normalization and position-wise activations are much less pertinent to the ability of self-attention to correlate inputs than its core connectivity, described in full by our analyzed model. 
	More generally, removing non-linearities for analysis of deep network connectivity traits is a common simplifying assumption: results on expressiveness and optimization of fully-connected~\citep{saxe2013exact,kawaguchi2016deep,hardt2016identity}, convolutional~\citep{cohen2016expressive}, and recurrent~\citep{khrulkov2018expressive,levine2018benefits} networks have been attained via this technique. 
		Trade-offs between the depth and width of fully-connected neural networks have been recently examined from theoretical~\citep{fan2020quasi,bu2020depth} and empirical~\citep{nguyen2020wide} perspectives. 
	To the best of our knowledge, our theoretical analysis is the first to address the question of parameter allocation between depth and width in self-attention networks.

	\begin{table}

		\begin{center}
			
			\begin{tabular}{l c| c c | c c }
				\toprule
				\multicolumn{2}{c}{{\textit{Borrowed from~\citet{GPT3}}}} &\multicolumn{2}{c}{{\textit{\textbf{Trained in practice}}}} &\multicolumn{2}{c}{{\textit{\textbf{Optimal by our fit}}}} \\
				Model Name & Size in params& Depth ($L$) & Width ($d_x$)& Depth ($L$) & Width ($d_x$)  \\
				\midrule
				GPT-3 Small & 125M & 12 & 768 & 23 &555 \\                              
				GPT-3 Medium & 350M & 24 & 1024 & 32 &886 \\                             
				GPT-3 Large & 760M & 24 & 1536 & 38 & 1220 \\                             
				GPT-3 XL & 1.3B & 24 & 2048 & 42 & 1550 \\                            
				GPT-3 2.7B   & 2.7B & 32 & 2560 & 47 & 2110 \\                             
				GPT-3 6.7B   & 6.7B & 32 & 4096 & 54 & 3150 \\                            
				GPT-3 13B  & 13.0B & 40 & 5140 & 60 & 4200 \\                           
				GPT-3 175B or ``GPT-3'' & 175.0B & 96 & 12288 & 80 & 13500 \\
				\midrule
				Optimal $1$-Trillion arch& 1T&--&--&95&30100\\
				\bottomrule
			\end{tabular}
		\end{center}
		\caption{Our projections regarding optimal depth-to-width parameter distribution at self-attention sizes corresponding to huge language models trained in~\citet{GPT3}, according to the fit in section~\ref{sec:exp} (see figure~\ref{fig:proj} for the statistical uncertainty in these predictions). 
			Up to the scale of~$10$B, the trained GPT3 architectures were generally too shallow per their parameter count, meaning that they under-performed relatively to the optimal architecture at that size (similarly to the depth $6$ network in the white regime of figure~\hyperref[fig:no_depth_efficiency]{~\ref{fig:no_depth_efficiency}(a)}.
			Coversely, the largest model trained to date, GPT3-175B, is too deep given its size, and could have benefited from widening at the expense of depth (similarly to the depth $48$ network in the gray regime of figure~\hyperref[fig:our_graph]{~\ref{fig:our_graph}(c)}. We project that the strongest $1$-Trillion parameter model would entail widening to an unprecedented width of $30$K.
			\label{table:tab}}
		\vspace{-3mm}
	\end{table}

	We employ the tool of a function's separation rank with respect to~subsets of its inputs, which quantifies its ability to model input dependencies (presented in section~\ref{sec:sep_rank}). The separation rank was employed for attaining theoretical insights on the dependencies modeled by convolutional and recurrent networks~\citep{cohen2017inductive,levine2018benefits}.
	
	\textbf{Rather than reinforcing the seemingly plausible hypothesis for the trend in figure~\hyperref[fig:no_depth_efficiency]{~\ref{fig:no_depth_efficiency}(a)}, by which widening a self-attention network is as effective as deepening it, we confirm the contrary. }
	We show that the operation of stacking self-attention layers is so effective that it quickly saturates a capacity of the network's width. 
	We establish in section~\ref{sec:results} the existence of a depth threshold which depends logarithmically on the width $d_x$, denoted $L_{\textrm{th}}(d_x)\sim \log(d_x)$. 
	Below the threshold, we prove that depth-efficiency takes place in self-attention networks: a network of depth $L\leq L_{\textrm{th}}(d_x)$ cannot be replicated by a shallower network, unless the latter's width grows double-exponentially with $L$. 
	We prove the above by showing that the separation rank of functions realized by self-attention networks grows double-exponentially with depth, but only polynomially with width, shedding light on the effectiveness of the self-attention mechanism in modeling input interactions when recursively repeated. 
	However, we show that this overwhelming advantage of depth is quickly replaced by a balanced growth. 
	We prove that for self-attention networks with $L> L_{\textrm{th}}(d_x)$ 
	the ability to model input dependencies, as modeled by the separation rank, increases similarly with depth and width.
	{We corroborate our theoretical findings empirically, as shown in the example of figure~\hyperref[fig:no_depth_efficiency]{~\ref{fig:no_depth_efficiency}(b)} and more extensively in section~\ref{sec:exp}}.
	
	\vspace{3mm}
	\tocless\subsection{An overview of our empirical approach and findings}
		\vspace{-2mm}
		
	For two networks with the same parameter count but of different depths $L_1<L_2$ and widths $d_2<d_1$, our theory indicates that:
	(1) there is no advantage to the deeper network when its dimension $d_2$ is too small (width caps the benefit of the added layers of depths $L_1+1,...,L_2$), but 
	(2) the deeper network should outperform the shallower one when its width $d_2$ is large enough such that the added layers are in the depth-efficiency regime.

	Traces of this predicted phenomenon appear in existing literature:
	A closer look at depths $L\leq 6$ in the experiment of \citet{kaplan2020scaling} in figure~\hyperref[fig:no_depth_efficiency]{~\ref{fig:no_depth_efficiency}(a)} reveals depth-efficiency, though the conclusion of that study was of an overall depth \textbf{in}efficiency of self-attention, based on the behavior of their $L>6$ curve.
	In section~\ref{sec:exp} we demonstrate empirically that the more nuanced transition between depth-efficiency and \textbf{in}efficiency, as predicted by our theory, affects commonly used self-attention depths of $L=6,12,18,24,30,36,48$ (a representative plot from our experiments is given in figure~\hyperref[fig:no_depth_efficiency]{~\ref{fig:no_depth_efficiency}(b)}). 
	The experiments reveal a third regime of ``width-efficiency": a network can be too deep for a given parameter budget, and under-perform relatively to a shallower and wider network of the same size (see figure~\ref{fig:our_graph}).   
	
	We fit the network sizes at which a transition between the different depth-efficiency regimes occur to an exponential form, predicted by our theory (see figure~\ref{fig:exponential}). This allows us to extrapolate the depth-efficiency behavior of larger architectures, and project practical guidelines for the architectural design of contemporary huge language models. Table~\ref{table:tab} shows our suggested depths and widths for models of sizes used in the recent GPT3 paper~\citep{GPT3}.     
	It seems that popular self-attention architectures at all sizes trained up to GPT3's crossing of the 100B parameter threshold, could generally benefit from deepening, with the appropriate widening (indicated by our guidelines).  
	With that, \textbf{our results clearly indicate the importance of widening self-attention networks when aiming for the 1 Trillion parameter mark}. We project the optimal architecture at that size to have depth 95 and width 30K, wider than any self-attention network trained to date.

	\ifdefined\SQUEEZE \vspace{-0mm} \fi
	\tocless\section{The self-attention mechanism \label{sec:sa}}
	\ifdefined\SQUEEZE \vspace{-0mm} \fi
	Differentiable attention models in which the output attends over all LSTM-based input representations have been introduced in the context of machine translation~\citep{bahdanau2014neural}. Self-attention (also referred to as intra-attention), which relates different inputs to each other, was first employed for machine reading~\citep{cheng2016long}, and soon thereafter shown to be useful for a variety of language applications when operating over LSTM-based representations~\citep{parikh2016decomposable,paulus2017deep,lin2017structured}.
	\cite{Transformer} were the first to demonstrate that a model based solely on attention, the Transformer, can be better than LSTM based networks. 
	The Transformer's encoder, BERT~\citep{BERT}, based entirely on self-attention, has demonstrated unprecedented performance across natural language understanding tasks.

	\ifdefined\SQUEEZE \vspace{-1mm} \fi
	\tocless\subsection{The Transformer encoder architecture \label{sec:sa:bert}}
	\ifdefined\SQUEEZE \vspace{-1mm} \fi
	We begin by describing the self-attention operation of the original Transformer, and then in the next subsection we present the modifications made in our analyzed model. 
	Each layer $l\in[L]:=\{1,...,L\}$ of a depth-$L$ Transformer encoder is comprised of two sub-layers. The $H$-headed self-attention sublayer of layer $l$ computes the following function at position $i\in[N]$, over its $N$ inputs $\{\x^{l,j}\in\R^{d_x}\}_{j=1}^N$:

	\ifdefined\SQUEEZE \vspace{-7mm} \fi
	\begin{align}\label{eq:bert_sa}
	\mathbf{f}_{\textrm{SA}}^{l,i}\left(\x^{l,1,},...,\x^{l,N}\right)=
	\sum_{j=1}^{N}\sum_{h=1}^H &SM_j\left\{\nicefrac{1}{\sqrt{d_a}}\left\langle W^{\textrm{Q},l,h} \x^{l,i},W^{\textrm{K},l,h} \x^{l,j}\right\rangle \right\}W^{\textrm{O},l,h}W^{\textrm{V},l,h}\x^{l,j}
	\end{align}
	\ifdefined\SQUEEZE \vspace{-3mm} \fi

	where $SM_j\left\{f(j)\right\}:=e^{f(j)}/\sum_{j'}e^{f(j')}$ is the softmax operation and $\forall h\in[H]$ the learned weights matrices $W^{\textrm{K},l,h},W^{\textrm{Q},l,h},W^{\textrm{V},l,h}\in\R^{d_a\times d_x}$ convert the representation from its dimension $d_x$ into the attention dimension $d_a=\nicefrac{d_x}{H}$, creating Key, Query, and Value representations, respectively. The learned weights matrix $W^{\textrm{O},l,h}\in\R^{d_x\times d_a}$ converts the attention result back into the representation dimension.
	The multi-headed self-attention sublayer output in eq.~\eqref{eq:bert_sa}, followed by a residual connection and layer-norm~\citep{ba2016layer}, is inserted into a position-wise feed-forward + ReLU sublayer, such that each layer's output at position $i\in[N]$ is:
	\begin{align}\label{eq:bert_layer}
	\mathbf{f}_{\textrm{Layer}}^{l,i}\left(\x^{l,1},...,\x^{l,N}\right)&=
	W^{\textrm{FF,2}}ReLU\left(W^{\textrm{FF,1}}LayerNorm\left(\mathbf{f}_{\textrm{SA}}^{l,i}+\x^{l,i}\right)\right),
	\end{align}
	where the feed-forward matrices are usually taken to be $W^{\textrm{FF,1}}\in\R^{4d_x\times d_x},W^{\textrm{FF,2}}\in\R^{d_x\times 4d_x}$, such that the parameter count for an entire layer is $12\cdot d^2_x$.
	Finally, the depth-$L$ multi-headed self-attention operation of the Transformer encoder is obtained by a composition of $L$ such layers, \ie, when setting $\forall l\in\{2,...,L\},j\in[N]:~\x^{l,j}=LayerNorm\left(\mathbf{f}_{\textrm{Layer}}^{l-1,j}\right)$, with $\x^{1,j}$ denoting the input to the deep self-attention network at position $j$.\footnote{Focusing on the self-attention operation, we omit a description of the input embedding matrix, as well as of the positional embeddings added at the input, which do not affect our analysis given realistic vocabulary sizes.}
	\ifdefined\SQUEEZE \vspace{0mm} \fi
	\tocless\subsection{The analyzed architecture \label{sec:sa:model}}
	\ifdefined\SQUEEZE \vspace{-1mm} \fi
	We analyze a deep multi-headed self-attention network variant which excludes the layer-norm operation, the softmax normalization, and the ReLU activation (see a thorough discussion on the effect of these relaxations in the next subsection).
	For cleanliness of presentation, we defer the analysis of the residual connection to the appendix (it bears insignificant impact on our bounds).
	Specifically, in the analyzed network, each layer $l\in[L]$ computes the following function at position $i\in[N]$ over its inputs $\{\x^{l,j}\in\R^{d_x}\}_{j=1}^N$:
	
	\ifdefined\SQUEEZE \vspace{-5mm} \fi
	\begin{align}\label{eq:our_layer}
	\y^{l,i}\left(\x^{l,1},...,\x^{l,N}\right)&=
	\sum_{j=1}^{N}\sum_{h=1}^H \left\langle W^{\textrm{Q},l,h} \x^{l,i},W^{\textrm{K},l,h} \x^{l,j}\right\rangle W^{\textrm{O},l,h} W^{\textrm{V},l,h} \x^{l,j},
	\end{align}
	\ifdefined\SQUEEZE \vspace{-3mm} \fi
	
	where the Feed-Forward matrices can be now effectively embedded within $W^{\textrm{O},l,h}$.  
	Our analysis below treats a deep multi-headed self-attention network that is attained by a concatenation of $L$ such layers.
	Importantly, the resultant ``linearized" network form, where activations and normalizations are removed, is by no means a linear mapping over the network input -- every layer integrates $3$ copies of its input in the above non-linear fashion.

	By recursively applying eq.~\eqref{eq:our_layer} $L$ times we attain the analyzed depth-$L$ self-attention network. We denote the function realized by a network with embedding dimension $d_x$ and $H$ attention heads per layer at output location $i\in[N]$ by: 
	\ifdefined\SQUEEZE \vspace{-2mm} \fi
	\begin{equation}
	\y^{i, L, d_x, H, \Theta}\left(\x^1,...,\x^N\right):=\sum_{j_1,...,j_C=1}^{N}\mathbf{g}^L\left(\x^i,\x^{j_1},...,\x^{j_C}\right),\label{eq:our_network}
	\end{equation}
	\ifdefined\SQUEEZE \vspace{-3mm} \fi
	
	where
	$\Theta$ denotes all $4LH$
	learned weight matrices: $\forall (l,h)\in[L]\otimes[{H}]:$$W^{\textrm{K},l,h},W^{\textrm{Q},l,h},W^{\textrm{V},l,h} \in\R^{d_a\times d_x},$ and $W^{\textrm{O},l,h} \in\R^{d_x\times d_a}$,
	and the function $\mathbf{g}^L$ is a placeholder, fully detailed in the appendix, which integrates $C=\frac{3^L-1}{2}$ different input vectors.
	Network connectivity implies that the number of summed position indices is also $C$. 
	Comparing the form of eq.~\eqref{eq:our_network} to the operation of a single layer in eq.~\eqref{eq:our_layer}, it can be seen schematically that while a single layer mixes the output position $i$ with every input position $j$ once and aggregates the result, depth brings forth an exponential enhancement to the amount of inputs mixed at once as well as to the amount of summed terms. 
	In section~\ref{sec:results}, we quantify this effect and analyze the limitations posed by the dimension of the internal representation (the width) on the network's ability to make use of this exponential growth with depth. 
	In the following subsection, we comment on the differences between the Transformer encoder architecture described in eqs.~\eqref{eq:bert_sa} and~\eqref{eq:bert_layer} and the self-attention architecture presented in eqs.~\eqref{eq:our_layer} and~\eqref{eq:our_network}.
	
	\ifdefined\SQUEEZE \vspace{0mm} \fi
	\tocless\subsection{Relaxations \label{sec:sa:relaxations}}
	\ifdefined\SQUEEZE \vspace{-1mm} \fi
	
	Empirical evidence indicates that while the ReLU activations and softmax normalization contribute to performance (layer-norm mainly contributes to optimization), the basic mechanism in eqs.~\eqref{eq:our_layer} and~\eqref{eq:our_network} above captures the defining self-attention characteristic of integrating the inputs with each other in a flexible manner:

	\emph{The ReLU activation relaxation}: \cite{press2019improving} demonstrate that a ``self-attention first" BERT variant that first performs all of the  self-attention operations (eq.~\eqref{eq:bert_sa}) consecutively, and only then performs all of the position-wise feed-forward+ReLU operations, achieves comparable language modeling performance relatively to the Baseline, which takes the regular approach of interleaving these functionalities (\ie, concatenating the BERT's layer described in~eq.~\eqref{eq:bert_layer}). They report that the interleaved Baseline achieves a perplexity score of $18.63\pm 0.26$ on the WikiText-103 test \citep{merity2016pointer} when averaged over $5$ random seeds, while the ``self-attention first" model  achieves a perplexity score of $18.82$ on this test set. 
	The best pre-Transformer perplexity result on the WikiText-103 test, reported by an LSTM-based architecture, was $29.2$~\citep{rae2018fast}. 
	Since ReLU and feed-forward do not mix different locations, this outcome directly implies that the self-attention mechanism itself provides all of the elaborate input integration which differentiates BERT from previous architectures.

	\emph{The softmax normalization relaxation}: Initially, an intuitive interpretation of attention as distributing ``fractions" of an overall attention budget among inputs was given to its actual operation of dynamically linking input and output locations.
	The intuitive interpretation, tightly linked to the need to transform the Key/Query similarity score into a distribution, has been recently challenged, as
	a growing body of work shows that the attention weights distribution does not directly correlate with predictions~\citep{AttentionIsNotExplanation,pruthi2019learning,brunner2020identifiability}.
	Moreover, \cite{richter2020normalized} recently point out undesirable traits of the softmax operation, demonstrating that its property of confining the outcome to the convex hull of its inputs unnecessarily limits the expressibility of the self-attention mechanism. 
	They experiment on a suite of synthetic tasks with a BERT variant in which the softmax normalization is removed, and find it to perform on par on almost all examined tasks. When replacing the softmax with other normalizations they report improvements.
	Finally, completely linearized attention (softmax removed) was employed on real tasks as means of reducing costs, since the softmax operation cost scales with the input size~\citep{de2016cheap,wang2020off}.

	The goal of the above points is not to advocate modifications in BERT's non-linearity or normalization operations (we leave that to other works), but to note that while these are under examination and are susceptible for alteration, the connectivity of self-attention, manifested by eqs.~\eqref{eq:our_layer} and~\eqref{eq:our_network} , is the core mechanism driving its functionality. 
	Our results, to be presented in section~\ref{sec:results}, demonstrate how conclusions drawn by directly analyzing this mechanism accord with the operation of commonly employed self-attention networks.
	
	\ifdefined\SQUEEZE \vspace{-1mm} \fi
	% SEPARATION RANK
	\tocless\section{A measure of capacity for modeling input dependencies  \label{sec:sep_rank}}
	\ifdefined\SQUEEZE \vspace{-1mm} \fi
	In this section, we introduce the separation rank of the function realized by a self-attention network as a measure that quantifies its ability to model dependencies between subsets of its variable set $\{\x^j\}_{j=1}^N$.
	We will use this measure in order to establish the two depth-efficiency/ inefficiency regimes  in self-attention.
	The separation rank, introduced in \cite{beylkin2002numerical} for high-dimensional numerical analysis, 
	was employed for various applications, \eg,~chemistry~\citep{harrison2003multiresolution}, particle engineering~\citep{hackbusch2006efficient}, and machine learning~\citep{beylkin2009multivariate}. 
	Importantly,
	the separation rank has been established as a measure of  dependencies modeled by deep convolutional and recurrent networks \wrt~their inputs~\citep{cohen2017inductive,levine2018benefits,levine2018deep}.

	Let $(A,B)$ be a partition of the input locations, \ie,~$A$ and~$B$ are disjoint subsets of~$[N]$ whose union gives~$[N]$.
	The separation rank of a function $y(\x^1,\ldots,\x^N)$~\wrt~partition $(A,B)$, is the minimal number of summands that together sum up to equal $y$, where each summand is \emph{multiplicatively separable \wrt~$(A,B)$}, \ie,~is equal to a product of two functions~--~one that intakes only inputs from one subset $\{\x^{j}:j\in A\}$, and another that intakes only inputs from the other subset $\{\x^{j}:j\in B\}$. 
	Formally, the \emph{separation rank} of $y:(\R^{d_x})^N\to\R$ \wrt~the partition $(A,B)$ is defined as follows:
	\ifdefined\SQUEEZE \vspace{-1.5mm} \fi
	\bea\label{eq:sep}
	sep(y;A,B):=\min\left\{R\in\N\cup\{0\}:\exists{g_1{\ldots}g_R:(\R^{d_x})^{\abs{A}}\to\R,g'_1{\ldots}g'_R:(\R^{d_x})^{\abs{B}}\to\R}~~s.t.\right.
	\quad~\label{eq:sep_rank}\\
	\left.y\left(\x^1,\ldots,\x^N\right)=\sum\nolimits_{r=1}^{R}g_{r}\left(\{\x^{j}:j\in A\}\right)g'_r\left(\{\x^{j}:j\in B\}\right)
	\right\}
	\nonumber
	\eea
	
	\ifdefined\SQUEEZE \vspace{-3mm} \fi

	If the separation rank of a function \wrt~a partition of its input is equal to~$1$, the function is separable, meaning it cannot take into account consistency between the values of $\{\x^{j}\}_{j\in A}$ and those of $\{\x^{j}\}_{j\in B}$.
	In a statistical setting, if~$y$ is a probability density function, this would mean that $\{\x^{j}\}_{j\in A}$ and $\{\x^{j}\}_{j\in B}$ are statistically independent.
	The higher $sep(y;A,B)$ is, the farther~$y$ is from this situation, \ie~the more it models dependency between $\{\x^{j}\}_{j\in A}$ and $\{\x^{j}\}_{j\in B}$, or equivalently, the stronger the correlation it induces between the inputs indexed by~$A$ and those indexed by~$B$.
	
	The fixed connectivity of ConvNets has been shown to yield high separation ranks \wrt~partitions which separate neighboring inputs (\eg, where all odd positions are in $A$ and all even positions are in $B$), while suffering from low separation ranks \wrt~partitions which separate distant inputs (\eg, where $A=1,...,\nicefrac{N}{2}$ and $B=\nicefrac{N}{2}+1,...,N$).
	Our analysis establishes a qualitatively different trait for {self-attention networks, which treat all balanced partitions alike}:
	\begin{proposition}
		For $p\in[d_x]$, let $y^{i, L, d_x, H, \Theta}_p$ be the scalar function computing the $p$th entry of an output vector at position $i\in[N]$ of the depth-$L$ self-attention network with embedding dimension $d_x$ and $H$ attention heads per layer, defined in eqs.~\eqref{eq:our_layer} and~\eqref{eq:our_network}.
		Then,  its 
		separation rank \wrt~balanced partitions, which obey $A\cupdot B=[N], \abs{A},\abs{B}=\nicefrac{N}{2}$, is invariant to the identity of the partition, \ie, $\forall A\cupdot B=[N], \tilde{A}\cupdot \tilde{B}=[N],~~s.t.~ \abs{A},\abs{B},|{\tilde{A}}|,|{\tilde{B}}|=\nicefrac{N}{2}$:
		\begin{equation}
		sep(y^{i, L, d_x,H,\Theta}_p;A,B)=sep(y^{i, L, d_x, H, \Theta}_p;\tilde{A},\tilde{B})	
		\end{equation}
		
		Accordingly, we will omit the specification of the partition in future uses, denoting $sep(y^{i, L, d_x, H, \Theta}_p)$ as the separation rank of $y^{i, L, d_x, H, \Theta}_p$ \wrt~any balanced partition of the inputs.
	\end{proposition} 
	
	This result accords with the intuition regarding the flexibility of the attention mechanism -- it does not integrate the input in a predefined pattern like convolutional networks, but dynamically learns to correlate any inter-dependent subsets of the inputs.
	Natural text exhibits non-smooth non-local dependency structures, as correlations between input segments can abruptly rise and decay with distance. 
	The fact that self-attention facilitates all correlation patterns equally poses it as a more natural architecture for language modeling related tasks. 
	Convolutional networks, with their local connectivity, may have the right inductive bias for imagery data, but partitions unfavored by them may reflect more erratic correlations that are nonetheless relevant for natural language inputs. 
	
	However, the above property of indifference to the input partition is not enough for succeeding at  tasks with elaborate input dependencies, since a function with equally low separation ranks for all input partitions has limited ability to model such dependencies.
	In the following section, we analyze how different architectural parameters affect the ability of self-attention networks to correlate their inputs, and by bounding their separation ranks, we establish the different depth-efficiency regimes in self-attention networks.
	
	\ifdefined\SQUEEZE \vspace{-1mm} \fi
	\tocless\section{The effect of depth in self-attention networks \label{sec:results}}
	In this section, we present tight bounds on the separation rank of self-attention networks, which reveal two qualitatively different regimes. In the first regime of $L<\log_3(d_x)$, analyzed in subsection~\ref{sec:results:depth_eff}, we establish that deepening is clearly preferable to widening.
	In the second regime of $L>\log_3(d_x)$, analyzed in subsection~\ref{sec:results:limit}, we show that deepening and widening play a similar role in enhancing the expressiveness self-attention networks.  
	\ifdefined\SQUEEZE \vspace{-1mm} \fi
	\tocless\subsection{Depth-efficiency in self-attention\label{sec:results:depth_eff}}
	\ifdefined\SQUEEZE \vspace{-1mm} \fi
	The recursive structure of deep self-attention hints at an exponential increase of input mixing with depth: The output of each layer is introduced $3$ times into the Key/Query/Value computation made by the subsequent layer. In this subsection, we formalize this intuition for self-attention networks of sufficient width, $d_x>3^L$.
	Theorem~\ref{theorem:depth_efficiency} below bounds the separation rank of such networks. Subsequent to its statement and brief outline of its proof, we explicitly show in corollary~\ref{cor:depth_efficiency} the implied double-exponential requirement from a bounded depth network attempting to replicate a deeper one.
	\begin{theorem}\label{theorem:depth_efficiency}
		For $p\in[d_x]$, let $y^{i, L, d_x, H, \Theta}_p$ be the scalar function computing the $p$th entry of an output vector at position $i\in[N]$ of the depth-$L$ self-attention network with embedding dimension $d_x$ and $H$ attention heads per layer, defined in eqs.~\eqref{eq:our_layer} and~\eqref{eq:our_network}.
		Let $sep(y^{i, L, d_x, H, \Theta}_p)$ be its 
		separation rank (section~\ref{sec:sep_rank}). 
		If $L,d_x$ obey $L<\log_3\left(d_{x}\right)$, then the following holds almost everywhere in the network's learned parameter space, \ie~for all values of the weight matrices (represented by $\Theta$) but a set of Lebesgue measure zero:
		\begin{equation}	\label{eq:theorem_1}
		3^{L-2}\left(\log_{3}\left(d_{x}-H\right)+a\right)\leq
		\log_{3}\left(sep(y_{p}^{i,L,d_{x},H,\Theta})\right)\leq\frac{3^{L}-1}{2}\log_{3}\left(d_{x}+H\right)
		\end{equation}
		with $a=-L+\left[2-\log_{3}2\right]$.
		(note that $\log_{3}\left(d_{x}-H\right)+a>0$ in this regime of $L<\log_3(d_x)$).
	\end{theorem}
	We provide below a short proof sketch of the lower bound in the above theorem. The derivation of the upper bound is more straightforward, and is left for the appendix, along with a formal proof of the lower bound.
	
	\emph{	\underline{Proof sketch for the lower bound in theorem~\ref{theorem:depth_efficiency}}}: We make use of grid tensor based function discretization~\citep{hackbusch2012tensor}  -- The function realized by a self-attention network is evaluated for a set of points on an exponentially large grid in the input space, and the outcomes are stored in a matrix $\M\left(y_{p}^{i,L,d_{x},H,\Theta}\right)$, which we prove upholds: $\textrm{rank}\left[\M\left(y_{p}^{i,L,d_{x},H,\Theta}\right)\right]\leq sep(y_{p}^{i,L,d_{x},H,\Theta})$, \ie, its rank lower bounds the separation rank.
	Since the entries of  
	$\M\left(y_{p}^{i,L,d_{x},H,\Theta}\right)$ 
	vary polynomially with the self-attention network's
	weights,
	we show that it suffices to find a single network weights assignment $\Theta$ for which the rank of the
	matrix is greater than the desired lower bound, in order to prove the case for almost all of the configurations of
	the network's learned weights (but a set of measure zero).
	Thus, we prove the lower bound in theorem~\ref{theorem:depth_efficiency} by choosing a simple weight assignment that still represents the self-attention connectivity, and showing that for this value of $\Theta$,  $\textrm{rank}\left[\M\left(y_{p}^{i,L,d_{x},H,\Theta}\right)\right]$ achieves the lower bound, in turn lower bounding the separation rank.
	$\qed$

	Theorem~\ref{theorem:depth_efficiency} bounds the separation rank of a deep  self-attention network of sufficient width between two functions that grow double-exponentially with depth and polynomially with width, tightly describing its behavior \wrt~depth and width.
	Because equivalence cannot hold between two functions of different separation ranks, the above result implies a double-exponential requirement from the width of a shallow network attempting to replicate the deep one, and clear depth efficiency holds:

	\begin{corollary}\label{cor:depth_efficiency}
		With probability $1$, the function realized upon randomization of the weights of a deep self-attention network defined in eqs.~\eqref{eq:our_layer} and~\eqref{eq:our_network} with depth $L^{\textrm{deep}}$ and width $d_x^{\textrm{deep}}>3^{L^{\textrm{deep}}}$, may only be realized by a shallower network with depth  $L^{\textrm{shallow}}= \nicefrac{L^{\textrm{deep}}}{d}$ and width $d_x^{\textrm{shallow}}=w d_x^{\textrm{shallow}}$, where $d>1,w>1$ (\ie, the deep network is deeper by a factor of $d$ and the shallow network is wider by a factor of $w$), if the following holds:
		\vspace{-1mm}  $$w\propto\exp(\exp(d)).$$
	\end{corollary}
	\vspace{-1mm}

	\ifdefined\SQUEEZE \vspace{-1mm} \fi
	\tocless\subsection{Depth in-efficiency in self-attention \label{sec:results:limit}}
	\ifdefined\SQUEEZE \vspace{-1mm} \fi
	Beyond establishing depth-efficiency in early self-attention layers, the above analysis sheds light on the contribution of a self-attention network's depth to its ability to correlate input subsets. 
	The separation rank (\wrt~any partition) of a single layer, given by eq.~\eqref{eq:our_layer}, is only linear in $H$ and $d_x$, showcasing a limitation  of the class of functions realized by single self-attention layers to model elaborate input dependencies. Theorem~\ref{theorem:depth_efficiency} quantifies the double exponential growth of this capacity measure with the number of stacked self-attention layers. 
	The following theorem shows that this growth is capped by the dimension of the internal representation:
	
	\begin{theorem}\label{theorem:width_bound}
		For $y^{i, L, d_x, H, \Theta}_p$ as defined in theorem~\ref{theorem:depth_efficiency}, if $L>\log_3\left(d_{x}\right)$, then the following  holds almost everywhere in the network's learned parameter space, \ie~for all values of the weight matrices (represented by $\Theta$) but a set of Lebesgue measure zero:
		\begin{equation}
		\frac{1}{2}d_{x}\cdot L+b_{1}+b_{2}\leq\log_{3}\left(sep(y^{i, L, d_x, H, \Theta}_p)\right)\leq
		2d_{x}\cdot L+c_{1}+c_{2}
		\end{equation}	
		with corrections on the order of $L$:
		$b_{1}=-L\left(\frac{H}{2}+1\right)$, $c_{1}=L$,
		and on the order of $d_x\log_3(d_x)$:
		$b_{2}=-d_{x}\left(1+\frac{1}{2}\log_{3}\left(\frac{d_{x}-H}{2}\right)\right)$, $c_{2}=-2d_{x}\cdot\log_{3}\nicefrac{d_{x}}{2\sqrt{2e}}+\log_{3}d_{x}$.
	\end{theorem}
	
	We provide below a proof sketch of the upper bound in the above theorem. The formal proof, along with the proof of the lower bound, which is similar to the one illustrated above for the lower bound in theorem~\ref{theorem:depth_efficiency}, are left for the appendix.
	
	\emph{	\underline{Proof sketch for the upper bound in theorem~\ref{theorem:width_bound}}}: By observing that $y^{i, L, d_x, H, \Theta}_p$ is a polynomial of degree $2C=3^L-1$ ($C$ is introduced in eq.~\eqref{eq:our_network}), we find a kernel $\psi\left(\x^1,...,\x^N\right)$ that maps the input into a space where each of the output monomials is a linear functional.
	We find a basis for the subspace $V$ spanned by the output monomials, and bound the separation rank of each element in that basis by a constant. 
	The dimension of $V$ is exponential in $Nd_x$ and polynomial in $3^L-1$, providing equal groundings for depth and width. 
	A careful analysis that exploits the sums over the indices $j_1,...,j_C$ in eq.~\eqref{eq:our_network}, removes the dependence on $N$.
	$\qed$
	
	Theorem~\ref{theorem:width_bound} states that when the network's depth passes a width dependent threshold, the separation rank turns from increasing polynomially with width and double-exponentially with depth to increasing-exponentially with width and depth together.
	Thus, while an increase in network size increases its capacity to model input dependencies, our result shows that there is no longer a clear cut advantage of depth in this respect:
	\begin{corollary}
		Let $\y^{\textrm{deep}}$ denote the function realized by a deep self-attention network at any output location $i\in[N]$, defined in eqs.~\eqref{eq:our_layer} and~\eqref{eq:our_network} with depth and width denoted $L^{\textrm{deep}},d_{x}^{\textrm{deep}}$ such that $L^{\textrm{deep}}>\log_{3}d_{x}^{\textrm{deep}}$. Denote  $\beta_{1}:=\frac{\log_{3}d_{x}^{\textrm{deep}}}{L^{\textrm{deep}}}<1$.
		Then, there exists $\beta_2=O(log(H)\cdot log(d_{x}^{\textrm{deep}})\cdot log(L^{\textrm{deep}}))$ such that the function realized by a network of depth: $L^{\textrm{shallow}}=\beta_{1}\cdot L^{\textrm{deep}}+\beta_{2}$, and width:  $d_{x}^{\textrm{shallow}}=3^{\beta_{2}}d_{x}^{\textrm{deep}}$, denoted $\y^{\textrm{shallow}}$, has higher separation rank, \ie:
		\begin{equation}
		sep(y_p^{\textrm{shallow}})>sep(y_{p'}^{\textrm{deep}})~~~;~~~~\textrm{where}~ p,p'\in[d_x]
		\end{equation}
	\end{corollary}
	\ifdefined\SQUEEZE \vspace{-1mm} \fi
	
	The above corollary, which follows from theorems~\ref{theorem:depth_efficiency} and~\ref{theorem:width_bound}, shows that the separation rank of a function realized by a self-attention network of arbitrary depth $L>\log_3(d_x)$
	can be surpassed by a shallower network of polynomial width, contrarily to the established behavior for networks of depth  $L<\log_3(d_x)$. 
	
	We leave it as an open conjecture that a polynomially sized shallower network can exactly replicate the operation of a deeper network in this regime. 
	With that, we point out that a variety of results which directly bound different complexity measures of deep networks have been put forward, shedding light on their operation~\citep{montufar2014number,bianchini2014complexity,raghu2017expressive,serra2017bounding,8951658}.
	Bounds on the separation rank have been used to explain the operation of more veteran architectures, and we find them to be particularly relevant in the case of self-attention: this complexity measure quantifies the amount of input inter-dependency induced by the network, directly reflecting a widespread intuition on the success behind the self-attention mechanism.

	\ifdefined\SQUEEZE \vspace{-2mm} \fi
	\tocless\section{Depth-efficiency regimes in common self-attention networks \label{sec:exp}}
	\ifdefined\SQUEEZE \vspace{-2mm} \fi
	In the previous sections, we analyzed a simplified version of self-attention networks (described in  section~\ref{sec:sa}). For this class, we proved the existence of the two different depth-efficiency/\textbf{in}efficiency regimes in self-attention networks, and further quantified the transition point between regimes to be exponential in network width (and accordingly in network size).
	In this section, we demonstrate that our theoretical predictions are manifested in common self-attention networks: the experiments below were conducted over common self-attention architectures which include all operations that were omitted in our theoretical analysis.
	We describe the training setup in section~\ref{sec:exp:setup}, the experiments in section~\ref{sec:exp:exp}, and the projection regrading optimal depth to width ratios (see table~\ref{table:tab}) in section~\ref{sec:exp:proj}.

		\tocless\subsection{The training setup\label{sec:exp:setup}}
	
	We trained common self-attention architectures of depths $L=6,12,18,24,30,36,48$ and varying widths, such that the network sizes range between $10^6$ and $6\cdot10^{8}$ (full details on the widths of the trained architectures are given in the appendix). We trained decoder-only (unidirectional) models, by optimizing the autoregressive log-likelihood of the training examples.
	We used a smaller than usual vocabulary size of $2000$ so that the vocabulary embedding parameters, given by $d_x\cdot V$ for a vocabulary of size $V$, would constitute a small fraction of the learned parameters for all data points. 
	Autoregressive models were shown to work well even on character level vocabularies (\eg,~\citep{peters2018deep}); due to modeling a joint distribution over the text, they are less sensitive to vocabulary size than bidirectional models~\citep{levine2021pmimasking}.

	Our training set was English Wikipedia, BookCorpus and OpenWebText. 
	We report the loss on a held out test set of size~$170$K sequences. 
	Notably, we estimated the variance of the pretraining and evaluation procedure by rerunning $11$ of the trained architectures three times each, and found it to be very low -- the reported test loss is stable up to its third digit.
	The remainder of the training details are given in the appendix. 
	
		\tocless\subsection{Experiments\label{sec:exp:exp}}
		\tocless\subsubsection{Distinct depth-efficiency regimes in self-attention\label{sec:exp:exp:regimes}}
	\begin{figure}
		\centering
		\includegraphics[width=\linewidth]{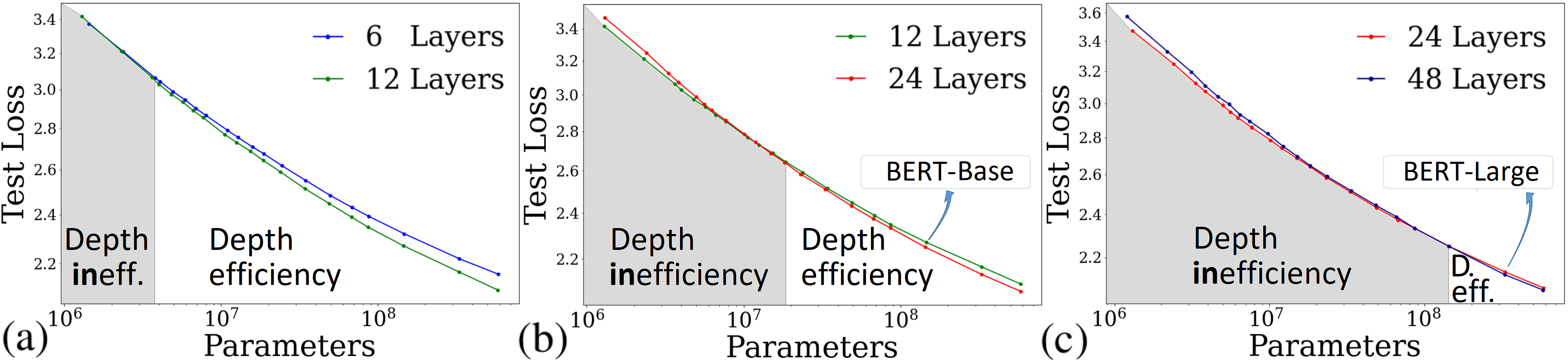}
		\vspace{-5mm} 
		\caption{An experimental validation of the existence of the two depth-efficiency/\textbf{in}efficiency regimes for common self-attention networks.  
			The transition between regimes
			%marked by $N_{\textrm{Transition}}(L^{\textrm{shallow}})$, 
			occurs in exponentially larger network sizes as the networks gets deeper, in  agreement with our theory (see figure~\ref{fig:exponential}).
		}\vspace{0mm}
		\label{fig:our_graph}
	\end{figure}

		Figure~\ref{fig:our_graph} shows that the predicted devision into two depth-efficiency/\textbf{in}efficiency regimes indeed takes place in common self-attention architectures. When comparing depths $(L^{\textrm{shallow}},L^{\textrm{deep}})=\{(6,12),(12,24),(24,48)\}$, 
	a qualitatively different depth-efficiency behavior is observed as the network size varies. 
	For smaller network sizes, deepening is not favorable over widening.
	Our theoretical analysis predicts this, showing that when the width of the deeper network is not large enough it can not use its excess layers efficiently.  
	However, when the network's size is increased by widening, a transition into the depth-efficiency regime is clearly demonstrated: for the same parameter budget the deeper network performs better. Once the deeper network becomes wide enough, such that the depth threshold for depth-efficiency surpasses $L^{\textrm{shallow}}$, it is significantly more expressive.
	
	\tocless\subsubsection{Transition between regimes depends exponentially on depth}\label{sec:exp:exp:transition}
	Importantly, beyond a qualitative match to the two predicted depth-efficiency/\textbf{in}efficiency behaviors, the experiments corroborate our prediction for an exponential dependence of the ``depth-efficiency width" --- the width for which a network becomes depth-efficient --- on the network's depth. 
	By quantifying this exponential behavior (figure~\ref{fig:proj}), we attain practical guidelines for depth-to-width parameter allocation in a self-attention network of a given size.  
	
	Per network depth, we examine the width in which it diverges from the subsequent trained depth, 
	 \ie, we examine the following pairs of trained adjacent depths: $(6,12),(12,18),(18,24),(24,30),(30,36),(36,48)$. 
	For each pair, we estimate the shallower network's transition width (marking the crossing between gray and white areas in figure~\ref{fig:our_graph}) as the average of its width in two points: the first point in which the shallower network under-performs in a statistically significant manner (see standard deviation estimation in the appendix), and the point to its left in which the performance of the two is not distinguishable. We take the empirical error of this estimation to be the distance between the two points.

	Our theoretical results in section~\ref{sec:results} predict that the above empirically estimated transition should occur when the shallower network's width $d_x$ is exponential in its depth $L$. Accordingly, we fit a linear dependence of the log of the width on the depth and receive the fit coefficients $(a,b)$: $\log\left(d_x\right)=a+b\cdot L$. 
	The linear fit, shown in  Figure~\hyperref[fig:exponential]{~\ref{fig:exponential}(a)} yields measures of 
	$R^2=0.998$ and $\chi_{\textrm{red}}^2=0.854$ (see further details in the appendix). 
	These measures imply a good compatibility of the theoretically predicted dependence to the measurements, and further reinforce the practical use we make of the fit parameters $a$ and $b$ hereinafter, for predicting the network size for which the regime transition occurs per depth. 
	
		\begin{figure}
		\centering
		\includegraphics[width=\linewidth]{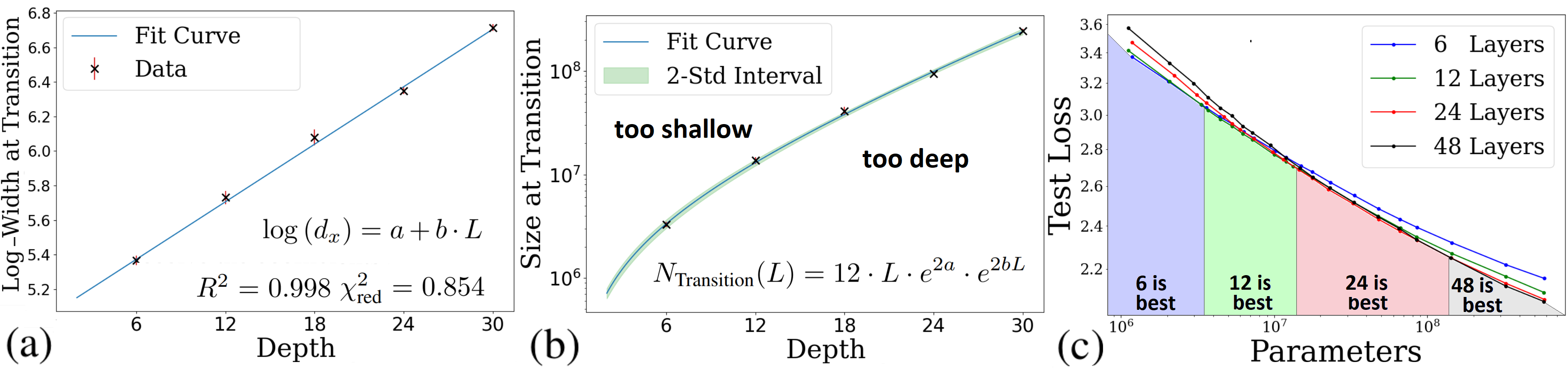}
		\caption{\textbf{(a)} A fit of the predicted exponential dependence of network width on its depth at the transition point between the depth-efficiency/inefficiency regimes.
			The experimental points are marked by black crosses and their empirical errors by red vertical lines.  
			\textbf{(b)} The network size at the transition between regimes $N_{\textrm{Transition}}$ as a function of network depth. The green area marks an interval of $2\Delta N_{\textrm{Transition}}$ as calculated in eq.~\ref{eq:fit} with the fit parameters and their variance given in eq.~\ref{eq:fit_params}.  Architectures to the top-left of the curve are too shallow relative to their size, and can be improved by deepening. \textbf{(c)} 
			The color in each range of network sizes corresponds to the color of the depth reaching the minimal loss in this range. This implies that architectures to the bottom-right of the curve in figure (b) are too deep relative to their size, and can be improved by widening.}\vspace{0mm}
		\label{fig:exponential}
	\end{figure}
	
	Specifically, we insert $d_x=e^a\cdot e^{bL}$ into the dependence 
$N=12\cdot L\cdot d_x^2$ and calculate the transition size and its propagated uncertainty as:
\begin{align}\label{eq:fit}
N_{\textrm{Transition}}(L)&=12\cdot L\cdot e^{2a}\cdot e^{2bL}\\\nonumber
\Delta N_{\textrm{Transition}}(L)&=\sqrt{\left(\nicefrac{dN}{da}\right)^{2}\sigma^2_{a}+\left(\nicefrac{dN}{db}\right)^{2}\sigma^2_{b}+2\left(\nicefrac{dN}{da}\right)\left(\nicefrac{dN}{db}\right)\sigma_{ab}}
\end{align}
with the fit parameters given by:
\begin{align}\label{eq:fit_params}
\begin{pmatrix}
a & b 
\end{pmatrix} ~~~~~&=~~~ 
\begin{pmatrix}
5.039\pm 0.030 & 5.55\cdot 10^{-2}\pm 1.3\cdot 10^{-3}
\end{pmatrix}\\\nonumber
\begin{pmatrix}
\sigma^2_{a} & \sigma_{ab} \\
\sigma_{ab} & \sigma^2_{b}
\end{pmatrix} &=~~~~~~~~~~~ 
\begin{pmatrix}
9.4\cdot 10^{-4}&-3.74 \cdot 10^{-5}\\
-3.74 \cdot 10^{-5}& 1.7\cdot 10^{-6}
\end{pmatrix}
\end{align}
\ifdefined\SQUEEZE \vspace{-2mm} \fi

Figure~\hyperref[fig:exponential]{~\ref{fig:exponential}(b)} shows the empirical transition sizes per depth on top of the projection and its error, calculated  by eq.~\ref{eq:fit} with the fit parameters in eq.~\ref{eq:fit_params}. 
Networks to the left of the curve are too shallow given their parameter budget, and can be improved by deepening at the expense of their width.

	\tocless\subsubsection{``Width-efficiency" in small network sizes}\label{sec:exp:exp:wifth}
	
	Our experiments reveal an empirical phenomenon that was not predicted by our theory. We established in section~\ref{sec:results} that depth does not have an advantage when the width is too small, but our bounds do not separate wider networks from deeper ones in this depth-inefficiency regime. 
	A surprising phenomenon is seen in  figures~\hyperref[fig:our_graph]{~\ref{fig:our_graph}(b,c)}: for small enough network sizes, deeper self-attention networks perform {\textit{worse}} than shallow ones. We leave a theoretical treatment of this regime for future work. 
	
	The above ``width-efficiency" empirical phenomenon leads to an important observation: for a given network size, a certain network can be too shallow, as we predicted theoretically and corroborated empirically above, but it can also be \textbf{too deep}. 
	In other words, the region to the right of the fit curve in figure~\hyperref[fig:exponential]{~\ref{fig:exponential}(b)} includes networks that can be improved by widening at the expense of their depth. 
	This implies that rather than representing a minimal depth per given self-attention network size, the curve in figure~\hyperref[fig:exponential]{~\ref{fig:exponential}(b)} represents the area of an \textbf{optimal depth} per network size. 
	We provide a demonstration of this idea in figure~\hyperref[fig:exponential]{~\ref{fig:exponential}(c)}, which clearly shows that when comparing networks of depths $L=6,12,24,48$, each one would be best to use in a different range of network sizes (the color in each range corresponds to  the best performing depth in that range).

\tocless\subsection{Projecting to larger networks\label{sec:exp:proj}}
	Beyond reflecting our theoretical predictions, the fit in figure~\ref{fig:exponential} can be used to project beyond the scope of our experiments in order to shed light on architecture design decisions made for much larger self-attention networks, like the contemporary huge Transformer-based language models~\citep{GPT3,raffel2019exploring,TuringNLG}. 
	Figure~\ref{fig:proj} shows the extrapolation of the fitted function and the uncertainty up to networks of depth $100$. 
	Notably, despite the uncertainty growing as the scope extends, $\frac{\Delta N_{\textrm{Transition}}(L=100)}{N_{\textrm{Transition}}(L=100)}=0.2$, \ie, the predictions for the optimal network size in the $L=100$ case are likely to be accurate within $20\%$ of the predicted size, yielding meaningful and unforeseen practical implications.
	
		\begin{wrapfigure}{r}{0.6\textwidth}
		\vspace{-5pt}
		\begin{center}
			\includegraphics[scale=0.35,clip=false,trim=0 40 0 20]{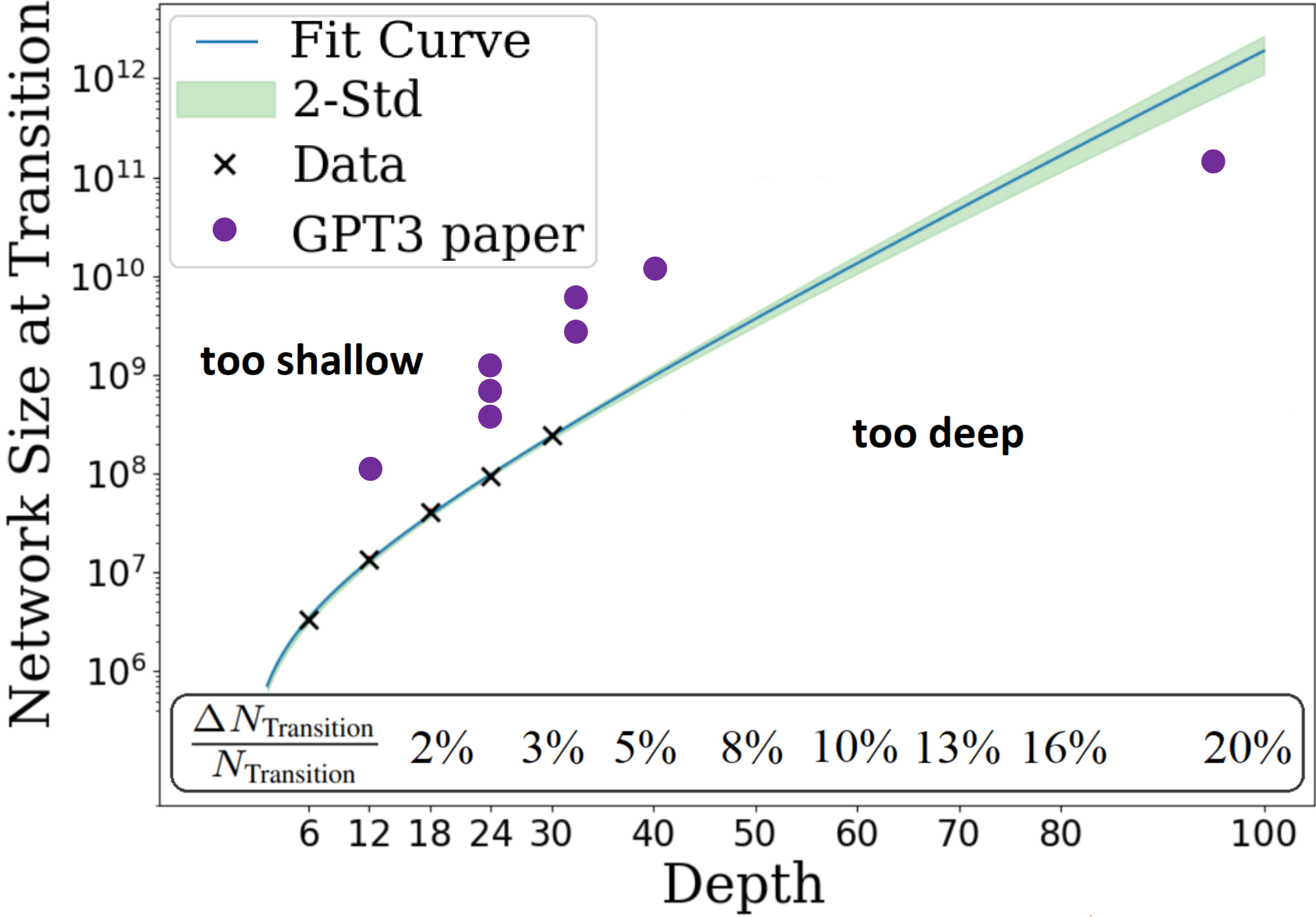}
		\end{center}
		\vspace{9pt}
		\caption{An extrapolation of the optimal size per depth (eq.~\ref{eq:fit}) beyond the scope of our experiments. The purple circles mark the GPT3 experiments detailed in table~\ref{table:tab}. \label{fig:proj}\vspace{-2em}}\end{wrapfigure}
	For example, when examining the architecture of GPT3, the deepest self-attention network trained to date with $96$ layers, we get $N_{\textrm{Transition}}(96)=1.17\pm0.23\cdot 10^{12}$, or over a Trillion parameters. This  places GPT3 with its 175B parameters significantly below our fit, suggesting that it may be too deep given its parameter budget. 
	In fact, the optimal depth  for GPT3's size is predicted to be $L=80$, since  $N_{\textrm{Transition}}(80	)=1.65\pm 0.25\cdot 10^{11}$. 
	Table~\ref{table:tab} includes further suggestion for huge models following our fit, including a suggestion to deepen networks on the left of the curve in figure~\ref{fig:proj}. With high certainty given our experimental data, the optimal model size increase towards 1 Trillion parameter models and beyond is via widening.
	
	\ifdefined\SQUEEZE \vspace{0mm} \fi
	\tocless\section{Discussion \label{sec:discussion}}
	\ifdefined\SQUEEZE \vspace{-2mm} \fi
	An apparent ``depth-\emph{\textbf{in}}efficiency" of self-attention networks was pointed out by prior works~-- in contrast to other successful deep learning architectures, in the case of self-attention there does not seem to be a clear advantage to deepening vs. widening.
	Our theoretical analysis clearly reflects this behavior in one parameter setting, but suggests an important nuance regarding its origins, while predicting a separate ``depth-efficiency" regime in another parameter setting. 
	Rather than an obvious explanation for the observed depth inefficiency, by which the self-attention mechanism does not benefit much from the operation of compounding, our analysis strongly points at the converse:
	self-attention is so effective at integrating its inputs, that it very quickly reaches saturation in the amount of dependencies that can be supported by the representation dimension.
	
	Thus, for early self-attention compounding, we prove a rapid growth in expressiveness with depth, and specifically in the ability to flexibly correlate between any input locations, which can not be accounted for by any reasonable widening.  
	However, our analysis pinpoints a transition in which the capacity of width to support the above rapid growth exhausts.
	Thus, when the width of a self-attention network is not large enough, the above depth-efficiency disappears -- deepening and widening become equivalent in terms of expressiveness. 
	
	We did not find a result which directly upper bounds depth-efficiency in other architecture classes. 
	Works by \cite{sharir2018expressive,PhysRevLett.122.065301} show an \textit{exponential} growth with depth of a measure related to the separation rank in certain classes of convolutional networks. 
	Comparing this with the \textit{double-exponential} growth shown in theorem~\ref{theorem:depth_efficiency} for early self-attention layers, it may be conjectured that convolutional networks seemingly benefit more from depth than self-attention does because their separation rank grows less rapidly, so they do not saturate some width dependent threshold as quickly as self-attention does. 
	We leave these investigations for future work.
	
	The experiments presented in section~\ref{sec:exp} reveal a qualitative and quantitative match to our theoretical predictions. Beyond reinforcing the validity of our theoretical interpretation, our comprehensive experimental setup allowed us to extrapolate and project depth-to-width trade-offs in huge self-attention networks, that are currently being trained as powerful language models. 
	For example, GPT3, the deepest self-attention network trained to date with $96$ layers, has matched this depth with an unprecedented width of $12$K. However, our projections clearly show that for this number of layers the network should be much wider. 
	\textbf{In fact, the logarithmic dependence that we establish between the optimal depth and width clearly dictates that size increase should be mainly via widening from this point ($\mathbf{\sim100B}$ models) onwards.}   
	This is good news from an engineering perspective: width can be increased more efficiently than depth in terms of parallelization.
	The multi-million dollar price tag on these architectures, along with the race to push the envelope towards $1$-Trillion parameter models and beyond,   
	make such informed guidelines an essential ingredient. 
	
	Beyond elucidating the behavior of vanilla self-attention architectures, our work theoretically motivates architectural changes that can provide the next leap in self-attention network expressiveness.   
	By indicating the network width as the limiting factor for depth-efficiency, our analysis encourages the development of methods for increasing network width with low expenses.
	For example, we point at the concept of ShuffleNet~\citep{ma2018shufflenet}, which has proven to be efficient for convolutional networks. They increase the representation dimension while using only a fraction of it for computation in each layer. This way, the computational costs are contained, but the width related theoretical limitations, posed by our work, are relaxed.
	Recently, \citet{fedus2021switch} trained a $1$-Trillion parameter model via a related approach which learns to choose the subset of parameters to apply in each layer.
	Indeed, we view our work as part of an effort to provide timely interpretations as feedback for the tremendous empirical pull in our field.
	
	\section*{Acknowledgments}
	We thank Daniel Jannai for assistance in the experiments, and Jared Kaplan for the permission to use the figure in~\cite{kaplan2020scaling}. This research was supported by the ERC (European Research Council) and the ISF (Israel Science Foundation). Experiments were performed with Cloud TPUs and supported by Google's TensorFlow Research Cloud (TFRC).
	Yoav Levine was supported by the Israel Academy of Sciences Adams fellowship.

	\small{
		\bibliographystyle{plainnat}
		\bibliography{refs}

\begin{thebibliography}{67}
\providecommand{\natexlab}[1]{#1}
\providecommand{\url}[1]{\texttt{#1}}
\expandafter\ifx\csname urlstyle\endcsname\relax
  \providecommand{\doi}[1]{doi: #1}\else
  \providecommand{\doi}{doi: \begingroup \urlstyle{rm}\Url}\fi

\bibitem[Amini et~al.(2012)Amini, Karbasi, and Marvasti]{amini2012low}
Arash Amini, Amin Karbasi, and Farokh Marvasti.
\newblock Low-rank matrix approximation using point-wise operators.
\newblock \emph{IEEE Transactions on Information Theory}, 58\penalty0
  (1):\penalty0 302--310, 2012.

\bibitem[Ba et~al.(2016)Ba, Kiros, and Hinton]{ba2016layer}
Jimmy~Lei Ba, Jamie~Ryan Kiros, and Geoffrey~E Hinton.
\newblock Layer normalization.
\newblock \emph{arXiv preprint arXiv:1607.06450}, 2016.

\bibitem[Bahdanau et~al.(2014)Bahdanau, Cho, and Bengio]{bahdanau2014neural}
Dzmitry Bahdanau, Kyunghyun Cho, and Yoshua Bengio.
\newblock Neural machine translation by jointly learning to align and
  translate.
\newblock \emph{arXiv preprint arXiv:1409.0473}, 2014.

\bibitem[Beylkin and Mohlenkamp(2002)]{beylkin2002numerical}
Gregory Beylkin and Martin~J Mohlenkamp.
\newblock Numerical operator calculus in higher dimensions.
\newblock \emph{Proceedings of the National Academy of Sciences}, 99\penalty0
  (16):\penalty0 10246--10251, 2002.

\bibitem[Beylkin et~al.(2009)Beylkin, Garcke, and
  Mohlenkamp]{beylkin2009multivariate}
Gregory Beylkin, Jochen Garcke, and Martin~J Mohlenkamp.
\newblock Multivariate regression and machine learning with sums of separable
  functions.
\newblock \emph{SIAM Journal on Scientific Computing}, 31\penalty0
  (3):\penalty0 1840--1857, 2009.

\bibitem[Bianchini and Scarselli(2014)]{bianchini2014complexity}
Monica Bianchini and Franco Scarselli.
\newblock On the complexity of neural network classifiers: A comparison between
  shallow and deep architectures.
\newblock \emph{Neural Networks and Learning Systems, IEEE Transactions on},
  25\penalty0 (8):\penalty0 1553--1565, 2014.

\bibitem[Brown et~al.(2020)Brown, Mann, Ryder, Subbiah, Kaplan, Dhariwal,
  Neelakantan, Shyam, Sastry, Askell, et~al.]{GPT3}
Tom~B Brown, Benjamin Mann, Nick Ryder, Melanie Subbiah, Jared Kaplan, Prafulla
  Dhariwal, Arvind Neelakantan, Pranav Shyam, Girish Sastry, Amanda Askell,
  et~al.
\newblock Language models are few-shot learners.
\newblock \emph{arXiv preprint arXiv:2005.14165}, 2020.

\bibitem[Brunner et~al.(2020)Brunner, Liu, Ortiz, Richter, Ciaramita, and
  Wattenhofer]{brunner2020identifiability}
Gino Brunner, Yang Liu, Damian~Pascual Ortiz, Oliver Richter, Massimiliano
  Ciaramita, and Roger Wattenhofer.
\newblock On identifiability in transformers.
\newblock 2020.

\bibitem[Bu et~al.(2020)Bu, Zhang, and Luo]{bu2020depth}
Kaifeng Bu, Yaobo Zhang, and Qingxian Luo.
\newblock Depth-width trade-offs for neural networks via topological entropy.
\newblock \emph{arXiv preprint arXiv:2010.07587}, 2020.

\bibitem[Caron and Traynor(2005)]{caron2005zero}
Richard Caron and Tim Traynor.
\newblock {The zero set of a polynomial}.
\newblock \emph{WSMR Report 05-02}, 2005.

\bibitem[Cheng et~al.(2016)Cheng, Dong, and Lapata]{cheng2016long}
Jianpeng Cheng, Li~Dong, and Mirella Lapata.
\newblock Long short-term memory-networks for machine reading.
\newblock \emph{arXiv preprint arXiv:1601.06733}, 2016.

\bibitem[Clark et~al.(2020)Clark, Luong, Le, and Manning]{ELECTRA}
Kevin Clark, Minh-Thang Luong, Quoc~V. Le, and Christopher~D. Manning.
\newblock Electra: Pre-training text encoders as discriminators rather than
  generators.
\newblock In \emph{International Conference on Learning Representations}, 2020.
\newblock URL \url{https://openreview.net/forum?id=r1xMH1BtvB}.

\bibitem[Cohen and Shashua(2017)]{cohen2017inductive}
Nadav Cohen and Amnon Shashua.
\newblock Inductive bias of deep convolutional networks through pooling
  geometry.
\newblock In \emph{5th International Conference on Learning Representations
  (ICLR)}, 2017.

\bibitem[Cohen et~al.(2016)Cohen, Sharir, and Shashua]{cohen2016expressive}
Nadav Cohen, Or~Sharir, and Amnon Shashua.
\newblock On the expressive power of deep learning: A tensor analysis.
\newblock \emph{Conference On Learning Theory (COLT)}, 2016.

\bibitem[Daniely(2017)]{daniely2017depth}
Amit Daniely.
\newblock Depth separation for neural networks.
\newblock \emph{arXiv preprint arXiv:1702.08489}, 2017.

\bibitem[de~Br{\'e}bisson and Vincent(2016)]{de2016cheap}
Alexandre de~Br{\'e}bisson and Pascal Vincent.
\newblock A cheap linear attention mechanism with fast lookups and fixed-size
  representations.
\newblock \emph{arXiv preprint arXiv:1609.05866}, 2016.

\bibitem[Devlin et~al.(2019)Devlin, Chang, Lee, and Toutanova]{BERT}
Jacob Devlin, Ming{-}Wei Chang, Kenton Lee, and Kristina Toutanova.
\newblock {BERT:} pre-training of deep bidirectional transformers for language
  understanding.
\newblock In Jill Burstein, Christy Doran, and Thamar Solorio, editors,
  \emph{Proceedings of the 2019 Conference of the North American Chapter of the
  Association for Computational Linguistics: Human Language Technologies,
  {NAACL-HLT} 2019, Minneapolis, MN, USA, June 2-7, 2019, Volume 1 (Long and
  Short Papers)}, pages 4171--4186. Association for Computational Linguistics,
  2019.
\newblock \doi{10.18653/v1/n19-1423}.
\newblock URL \url{https://doi.org/10.18653/v1/n19-1423}.

\bibitem[Eldan and Shamir(2016)]{eldan2016power}
Ronen Eldan and Ohad Shamir.
\newblock The power of depth for feedforward neural networks.
\newblock In \emph{Conference on learning theory}, pages 907--940, 2016.

\bibitem[Fan et~al.(2020)Fan, Lai, and Wang]{fan2020quasi}
Fenglei Fan, Rongjie Lai, and Ge~Wang.
\newblock Quasi-equivalence of width and depth of neural networks.
\newblock 2020.

\bibitem[Fedus et~al.(2021)Fedus, Zoph, and Shazeer]{fedus2021switch}
William Fedus, Barret Zoph, and Noam Shazeer.
\newblock Switch transformers: Scaling to trillion parameter models with simple
  and efficient sparsity.
\newblock \emph{arXiv preprint arXiv:2101.03961}, 2021.

\bibitem[Hackbusch(2006)]{hackbusch2006efficient}
Wolfgang Hackbusch.
\newblock On the efficient evaluation of coalescence integrals in population
  balance models.
\newblock \emph{Computing}, 78\penalty0 (2):\penalty0 145--159, 2006.

\bibitem[Hackbusch(2012)]{hackbusch2012tensor}
Wolfgang Hackbusch.
\newblock \emph{Tensor spaces and numerical tensor calculus}, volume~42.
\newblock Springer Science \& Business Media, 2012.

\bibitem[Hardt and Ma(2016)]{hardt2016identity}
Moritz Hardt and Tengyu Ma.
\newblock Identity matters in deep learning.
\newblock \emph{arXiv preprint arXiv:1611.04231}, 2016.

\bibitem[Hardy et~al.(1952)Hardy, Littlewood, and
  P{\'o}lya]{hardy1952inequalities}
Godfrey~Harold Hardy, John~Edensor Littlewood, and George P{\'o}lya.
\newblock \emph{Inequalities}.
\newblock Cambridge university press, 1952.

\bibitem[Harrison et~al.(2003)Harrison, Fann, Yanai, and
  Beylkin]{harrison2003multiresolution}
Robert~J Harrison, George~I Fann, Takeshi Yanai, and Gregory Beylkin.
\newblock Multiresolution quantum chemistry in multiwavelet bases.
\newblock In \emph{Computational Science-ICCS 2003}, pages 103--110. Springer,
  2003.

\bibitem[He et~al.(2016)He, Zhang, Ren, and Sun]{he2016deep}
Kaiming He, Xiangyu Zhang, Shaoqing Ren, and Jian Sun.
\newblock Deep residual learning for image recognition.
\newblock In \emph{Proceedings of the IEEE Conference on Computer Vision and
  Pattern Recognition}, pages 770--778, 2016.

\bibitem[{Inoue}(2019)]{8951658}
K.~{Inoue}.
\newblock Expressive numbers of two or more hidden layer relu neural networks.
\newblock In \emph{2019 Seventh International Symposium on Computing and
  Networking Workshops (CANDARW)}, pages 129--135, 2019.

\bibitem[Jain and Wallace(2019)]{AttentionIsNotExplanation}
Sarthak Jain and Byron~C. Wallace.
\newblock Attention is not explanation.
\newblock In Jill Burstein, Christy Doran, and Thamar Solorio, editors,
  \emph{Proceedings of the 2019 Conference of the North American Chapter of the
  Association for Computational Linguistics: Human Language Technologies,
  {NAACL-HLT} 2019, Minneapolis, MN, USA, June 2-7, 2019, Volume 1 (Long and
  Short Papers)}, pages 3543--3556. Association for Computational Linguistics,
  2019.
\newblock \doi{10.18653/v1/n19-1357}.
\newblock URL \url{https://doi.org/10.18653/v1/n19-1357}.

\bibitem[Kaplan et~al.(2020)Kaplan, McCandlish, Henighan, Brown, Chess, Child,
  Gray, Radford, Wu, and Amodei]{kaplan2020scaling}
Jared Kaplan, Sam McCandlish, Tom Henighan, Tom~B Brown, Benjamin Chess, Rewon
  Child, Scott Gray, Alec Radford, Jeffrey Wu, and Dario Amodei.
\newblock Scaling laws for neural language models.
\newblock \emph{arXiv preprint arXiv:2001.08361}, 2020.

\bibitem[Kawaguchi(2016)]{kawaguchi2016deep}
Kenji Kawaguchi.
\newblock Deep learning without poor local minima.
\newblock In \emph{Advances in neural information processing systems}, pages
  586--594, 2016.

\bibitem[Khrulkov et~al.(2018)Khrulkov, Novikov, and
  Oseledets]{khrulkov2018expressive}
Valentin Khrulkov, Alexander Novikov, and Ivan Oseledets.
\newblock Expressive power of recurrent neural networks.
\newblock In \emph{6th International Conference on Learning Representations
  (ICLR)}, 2018.

\bibitem[Levine et~al.(2018{\natexlab{a}})Levine, Sharir, Ziv, and
  Shashua]{levine2018benefits}
Yoav Levine, Or~Sharir, Alon Ziv, and Amnon Shashua.
\newblock {Benefits of depth for long-term memory of recurrent networks}.
\newblock \emph{(ICLR 2018) International Conference on Learning
  Representations workshop}, 2018{\natexlab{a}}.

\bibitem[Levine et~al.(2018{\natexlab{b}})Levine, Yakira, Cohen, and
  Shashua]{levine2018deep}
Yoav Levine, David Yakira, Nadav Cohen, and Amnon Shashua.
\newblock Deep learning and quantum entanglement: Fundamental connections with
  implications to network design.
\newblock In \emph{6th International Conference on Learning Representations
  (ICLR)}, 2018{\natexlab{b}}.

\bibitem[Levine et~al.(2019)Levine, Sharir, Cohen, and
  Shashua]{PhysRevLett.122.065301}
Yoav Levine, Or~Sharir, Nadav Cohen, and Amnon Shashua.
\newblock Quantum entanglement in deep learning architectures.
\newblock \emph{Phys. Rev. Lett.}, 122:\penalty0 065301, Feb 2019.
\newblock \doi{10.1103/PhysRevLett.122.065301}.
\newblock URL \url{https://link.aps.org/doi/10.1103/PhysRevLett.122.065301}.

\bibitem[Levine et~al.(2021)Levine, Lenz, Lieber, Abend, Leyton-Brown,
  Tennenholtz, and Shoham]{levine2021pmimasking}
Yoav Levine, Barak Lenz, Opher Lieber, Omri Abend, Kevin Leyton-Brown, Moshe
  Tennenholtz, and Yoav Shoham.
\newblock Pmi-masking: Principled masking of correlated spans.
\newblock In \emph{International Conference on Learning Representations}, 2021.
\newblock URL \url{https://openreview.net/forum?id=3Aoft6NWFej}.

\bibitem[Lin et~al.(2017)Lin, Feng, Santos, Yu, Xiang, Zhou, and
  Bengio]{lin2017structured}
Zhouhan Lin, Minwei Feng, Cicero Nogueira~dos Santos, Mo~Yu, Bing Xiang, Bowen
  Zhou, and Yoshua Bengio.
\newblock A structured self-attentive sentence embedding.
\newblock \emph{arXiv preprint arXiv:1703.03130}, 2017.

\bibitem[Liu et~al.(2019)Liu, Ott, Goyal, Du, Joshi, Chen, Levy, Lewis,
  Zettlemoyer, and Stoyanov]{RoBERTa}
Yinhan Liu, Myle Ott, Naman Goyal, Jingfei Du, Mandar Joshi, Danqi Chen, Omer
  Levy, Mike Lewis, Luke Zettlemoyer, and Veselin Stoyanov.
\newblock Roberta: {A} robustly optimized {BERT} pretraining approach.
\newblock \emph{CoRR}, abs/1907.11692, 2019.
\newblock URL \url{http://arxiv.org/abs/1907.11692}.

\bibitem[Lu et~al.(2017)Lu, Pu, Wang, Hu, and Wang]{lu2017expressive}
Zhou Lu, Hongming Pu, Feicheng Wang, Zhiqiang Hu, and Liwei Wang.
\newblock The expressive power of neural networks: A view from the width.
\newblock In \emph{Advances in neural information processing systems}, pages
  6231--6239, 2017.

\bibitem[Ma et~al.(2018)Ma, Zhang, Zheng, and Sun]{ma2018shufflenet}
Ningning Ma, Xiangyu Zhang, Hai-Tao Zheng, and Jian Sun.
\newblock Shufflenet v2: Practical guidelines for efficient cnn architecture
  design.
\newblock In \emph{Proceedings of the European conference on computer vision
  (ECCV)}, pages 116--131, 2018.

\bibitem[Merity et~al.(2016)Merity, Xiong, Bradbury, and
  Socher]{merity2016pointer}
Stephen Merity, Caiming Xiong, James Bradbury, and Richard Socher.
\newblock Pointer sentinel mixture models.
\newblock \emph{arXiv preprint arXiv:1609.07843}, 2016.

\bibitem[Michel et~al.(2019)Michel, Levy, and Neubig]{michel2019sixteen}
Paul Michel, Omer Levy, and Graham Neubig.
\newblock Are sixteen heads really better than one?
\newblock In \emph{Advances in Neural Information Processing Systems}, pages
  14014--14024, 2019.

\bibitem[Montufar et~al.(2014)Montufar, Pascanu, Cho, and
  Bengio]{montufar2014number}
Guido~F Montufar, Razvan Pascanu, Kyunghyun Cho, and Yoshua Bengio.
\newblock On the number of linear regions of deep neural networks.
\newblock In \emph{Advances in neural information processing systems}, pages
  2924--2932, 2014.

\bibitem[Nguyen et~al.(2020)Nguyen, Raghu, and Kornblith]{nguyen2020wide}
Thao Nguyen, Maithra Raghu, and Simon Kornblith.
\newblock Do wide and deep networks learn the same things? uncovering how
  neural network representations vary with width and depth.
\newblock \emph{arXiv preprint arXiv:2010.15327}, 2020.

\bibitem[Parikh et~al.(2016)Parikh, T{\"a}ckstr{\"o}m, Das, and
  Uszkoreit]{parikh2016decomposable}
Ankur~P Parikh, Oscar T{\"a}ckstr{\"o}m, Dipanjan Das, and Jakob Uszkoreit.
\newblock A decomposable attention model for natural language inference.
\newblock \emph{arXiv preprint arXiv:1606.01933}, 2016.

\bibitem[Paulus et~al.(2017)Paulus, Xiong, and Socher]{paulus2017deep}
Romain Paulus, Caiming Xiong, and Richard Socher.
\newblock A deep reinforced model for abstractive summarization.
\newblock \emph{arXiv preprint arXiv:1705.04304}, 2017.

\bibitem[Peters et~al.(2018)Peters, Neumann, Iyyer, Gardner, Clark, Lee, and
  Zettlemoyer]{peters2018deep}
Matthew~E Peters, Mark Neumann, Mohit Iyyer, Matt Gardner, Christopher Clark,
  Kenton Lee, and Luke Zettlemoyer.
\newblock Deep contextualized word representations.
\newblock \emph{arXiv preprint arXiv:1802.05365}, 2018.

\bibitem[Press et~al.(2019)Press, Smith, and Levy]{press2019improving}
Ofir Press, Noah~A Smith, and Omer Levy.
\newblock Improving transformer models by reordering their sublayers.
\newblock \emph{arXiv preprint arXiv:1911.03864}, 2019.

\bibitem[Pruthi et~al.(2019)Pruthi, Gupta, Dhingra, Neubig, and
  Lipton]{pruthi2019learning}
Danish Pruthi, Mansi Gupta, Bhuwan Dhingra, Graham Neubig, and Zachary~C
  Lipton.
\newblock Learning to deceive with attention-based explanations.
\newblock \emph{arXiv preprint arXiv:1909.07913}, 2019.

\bibitem[Radford et~al.(2019)Radford, Wu, Child, Luan, Amodei, and
  Sutskever]{GPT2}
Alec Radford, Jeff Wu, Rewon Child, David Luan, Dario Amodei, and Ilya
  Sutskever.
\newblock Language models are unsupervised multitask learners.
\newblock 2019.

\bibitem[Rae et~al.(2018)Rae, Dyer, Dayan, and Lillicrap]{rae2018fast}
Jack~W Rae, Chris Dyer, Peter Dayan, and Timothy~P Lillicrap.
\newblock Fast parametric learning with activation memorization.
\newblock \emph{arXiv preprint arXiv:1803.10049}, 2018.

\bibitem[Raffel et~al.(2019{\natexlab{a}})Raffel, Shazeer, Roberts, Lee,
  Narang, Matena, Zhou, Li, and Liu]{T5}
Colin Raffel, Noam Shazeer, Adam Roberts, Katherine Lee, Sharan Narang, Michael
  Matena, Yanqi Zhou, Wei Li, and Peter~J Liu.
\newblock Exploring the limits of transfer learning with a unified text-to-text
  transformer.
\newblock \emph{arXiv preprint arXiv:1910.10683}, 2019{\natexlab{a}}.

\bibitem[Raffel et~al.(2019{\natexlab{b}})Raffel, Shazeer, Roberts, Lee,
  Narang, Matena, Zhou, Li, and Liu]{raffel2019exploring}
Colin Raffel, Noam Shazeer, Adam Roberts, Katherine Lee, Sharan Narang, Michael
  Matena, Yanqi Zhou, Wei Li, and Peter~J Liu.
\newblock Exploring the limits of transfer learning with a unified text-to-text
  transformer.
\newblock \emph{arXiv preprint arXiv:1910.10683}, 2019{\natexlab{b}}.

\bibitem[Raghu et~al.(2017)Raghu, Poole, Kleinberg, Ganguli, and
  Dickstein]{raghu2017expressive}
Maithra Raghu, Ben Poole, Jon Kleinberg, Surya Ganguli, and Jascha~Sohl
  Dickstein.
\newblock On the expressive power of deep neural networks.
\newblock In \emph{Proceedings of the 34th International Conference on Machine
  Learning-Volume 70}, pages 2847--2854. JMLR. org, 2017.

\bibitem[Richter and Wattenhofer(2020)]{richter2020normalized}
Oliver Richter and Roger Wattenhofer.
\newblock Normalized attention without probability cage.
\newblock \emph{arXiv preprint arXiv:2005.09561}, 2020.

\bibitem[Rosset(2020)]{TuringNLG}
Corby Rosset.
\newblock {Turing-NLG}: A 17-billion-parameter language model by microsoft.
\newblock
  \url{https://www.microsoft.com/en-us/research/blog/turing-nlg-a-17-billion-parameter-language-model-by-microsoft/},
  2020.
\newblock Accessed: 2020-04-12.

\bibitem[Saxe et~al.(2013)Saxe, McClelland, and Ganguli]{saxe2013exact}
Andrew~M Saxe, James~L McClelland, and Surya Ganguli.
\newblock Exact solutions to the nonlinear dynamics of learning in deep linear
  neural networks.
\newblock \emph{arXiv preprint arXiv:1312.6120}, 2013.

\bibitem[Serra et~al.(2017)Serra, Tjandraatmadja, and
  Ramalingam]{serra2017bounding}
Thiago Serra, Christian Tjandraatmadja, and Srikumar Ramalingam.
\newblock Bounding and counting linear regions of deep neural networks.
\newblock \emph{arXiv preprint arXiv:1711.02114}, 2017.

\bibitem[Sharir and Shashua(2018)]{sharir2018expressive}
Or~Sharir and Amnon Shashua.
\newblock On the expressive power of overlapping architectures of deep
  learning.
\newblock In \emph{6th International Conference on Learning Representations
  (ICLR)}, 2018.

\bibitem[Sharir et~al.(2016)Sharir, Tamari, Cohen, and
  Shashua]{sharirtractable}
Or~Sharir, Ronen Tamari, Nadav Cohen, and Amnon Shashua.
\newblock Tractable generative convolutional arithmetic circuits.
\newblock 2016.

\bibitem[Simonyan and Zisserman(2014)]{simonyan2014very}
Karen Simonyan and Andrew Zisserman.
\newblock Very deep convolutional networks for large-scale image recognition.
\newblock \emph{arXiv preprint arXiv:1409.1556}, 2014.

\bibitem[Tan and Le(2019)]{tan2019efficientnet}
Mingxing Tan and Quoc~V Le.
\newblock Efficientnet: Rethinking model scaling for convolutional neural
  networks.
\newblock \emph{arXiv preprint arXiv:1905.11946}, 2019.

\bibitem[Vaswani et~al.(2017)Vaswani, Shazeer, Parmar, Uszkoreit, Jones, Gomez,
  Kaiser, and Polosukhin]{Transformer}
Ashish Vaswani, Noam Shazeer, Niki Parmar, Jakob Uszkoreit, Llion Jones,
  Aidan~N. Gomez, Lukasz Kaiser, and Illia Polosukhin.
\newblock Attention is all you need.
\newblock In Isabelle Guyon, Ulrike von Luxburg, Samy Bengio, Hanna~M. Wallach,
  Rob Fergus, S.~V.~N. Vishwanathan, and Roman Garnett, editors, \emph{Advances
  in Neural Information Processing Systems 30: Annual Conference on Neural
  Information Processing Systems 2017, 4-9 December 2017, Long Beach, CA,
  {USA}}, pages 5998--6008, 2017.
\newblock URL \url{http://papers.nips.cc/paper/7181-attention-is-all-you-need}.

\bibitem[Veit et~al.(2016)Veit, Wilber, and Belongie]{veit2016residual}
Andreas Veit, Michael~J Wilber, and Serge Belongie.
\newblock Residual networks behave like ensembles of relatively shallow
  networks.
\newblock In \emph{Advances in neural information processing systems}, pages
  550--558, 2016.

\bibitem[Wang et~al.(2020)Wang, Zhou, Chen, Hu, and Chen]{wang2020off}
Chengwei Wang, Tengfei Zhou, Chen Chen, Tianlei Hu, and Gang Chen.
\newblock Off-policy recommendation system without exploration.
\newblock In \emph{Pacific-Asia Conference on Knowledge Discovery and Data
  Mining}, pages 16--27. Springer, 2020.

\bibitem[Wu et~al.(2019)Wu, Shen, and Van Den~Hengel]{wu2019wider}
Zifeng Wu, Chunhua Shen, and Anton Van Den~Hengel.
\newblock Wider or deeper: Revisiting the resnet model for visual recognition.
\newblock \emph{Pattern Recognition}, 90:\penalty0 119--133, 2019.

\bibitem[Yang et~al.(2019)Yang, Dai, Yang, Carbonell, Salakhutdinov, and
  Le]{XLNet}
Zhilin Yang, Zihang Dai, Yiming Yang, Jaime~G. Carbonell, Ruslan Salakhutdinov,
  and Quoc~V. Le.
\newblock Xlnet: Generalized autoregressive pretraining for language
  understanding.
\newblock In Hanna~M. Wallach, Hugo Larochelle, Alina Beygelzimer, Florence
  d'Alch{\'{e}}{-}Buc, Emily~B. Fox, and Roman Garnett, editors, \emph{Advances
  in Neural Information Processing Systems 32: Annual Conference on Neural
  Information Processing Systems 2019, NeurIPS 2019, 8-14 December 2019,
  Vancouver, BC, Canada}, pages 5754--5764, 2019.
\newblock URL
  \url{http://papers.nips.cc/paper/8812-xlnet-generalized-autoregressive-pretraining-for-language-understanding}.

\bibitem[Zagoruyko and Komodakis(2016)]{zagoruyko2016wide}
Sergey Zagoruyko and Nikos Komodakis.
\newblock Wide residual networks.
\newblock \emph{arXiv preprint arXiv:1605.07146}, 2016.

\end{thebibliography}
	}
	
	% APPENDIXES
	\clearpage
	\appendix
	\tableofcontents
	\section{Upper bounds on the separation rank}
	\subsection{The function realized by a deep multi-headed self-attention network}\label{sec:sa_function}
	In this subsection, we prove facts  on the general structure of the function realized by the analyzed self-attention architecture that will be of use to us in the upcoming proofs.
	For a cleaner presentation, we will rewrite eq.~\eqref{eq:our_layer} in vectorized notation:
	\begin{align}\label{eq:our_layer_matrix_notations}
	Y=\sum_{h=1}^{H}W^{\textrm{O},h}W^{\textrm{V},h}XX^{T}\left(W^{\textrm{K},h}\right)^{T}W^{\textrm{Q},h}X
	\end{align}
	where $X,Y\left(X\right)\in\mathbb{R}^{d_{x}\times N}$ denote matrices
	respectively holding $\x^{j}$,$\y^{j}\left(\x^{1},...,\x^{N}\right)$ in their $j$'th column.
	Similarly treating eq.~\eqref{eq:our_network}, we will denote by $Y^{L,d_{x},H,\Theta}\left(X\right)\in\mathbb{R}^{d_{x}\times N}$
	the matrix holding $y^{j,L,d_{x},H,\Theta}\left(x^{1},...,x^{N}\right)$  in its $j$'th column.
	
	We begin by proving a lemma that reveals the structure of  $\mathbf{g}^L$ presented in eq.~\eqref{eq:our_network}:
	\begin{lemma}
		\label{lemma:sa_abstract_structre}Defining $C\left(L\right)\coloneqq\frac{3^{L}-1}{2}$, any depth $L$
		composition of the self-attention layers defined in eq.~\eqref{eq:our_layer} can be written as:
		
		\begin{align}\label{eq:sa_abstract_structre}
		Y^{L,d_{x},H,\Theta}=\sum_{h\in\left[H\right]^{\left[C\left(L\right)\right]}}B^{\left(0,h\right)T}M^{\left(1,h\right)}\cdots M^{\left(C\left(L\right),h\right)}A^{\left(0,h\right)}X
		\end{align}
		
		where $\forall h\in\left[H\right]^{\left[C\right]}\,0\leq c\leq C\left(L\right):M^{\left(c,h\right)}=A^{\left(c,h\right)}XX^{T}B^{\left(c,h\right)T}$
		and $A^{\left(c,h\right)},B^{\left(c,h\right)}\in\mathbb{R}^{d_{a}\times d_{x}}$.
	\end{lemma}	
	\begin{proof}
		By Induction on $L$. Base case:
		
		\begin{align*}
		Y^{\left(1\right)}\left(X\right)=\sum_{h=1}^{H}\underbrace{W^{\textrm{O},h}}_{B^{T}}\underbrace{W^{\textrm{V},h}XX^{T}\left(W^{\textrm{K},h}\right)^{T}}_{M}\underbrace{W^{\textrm{Q},h}}_{A}X
		\end{align*}
		\begin{align*}
		Y^{\left(L+1\right)}\left(X\right)=\sum_{h=1}^{H}W^{\textrm{O},h}W^{\textrm{V},h}Y^{\left(L\right)}\left(X\right)Y^{\left(L\right)}\left(X\right)^{T}\left(W^{\textrm{K},h}\right)^{T}W^{\textrm{Q},h}Y^{\left(L\right)}\left(X\right)
		\end{align*}
		Now, substituting in the induction hypothesis on the structure of  $Y^{\left(L\right)}\left(X\right)$ yields:
		\begin{align*}
		& =\sum_{h=1}^{H}W^{\textrm{O},h}W^{\textrm{V},h}\left(\sum_{h_{1}\in\left[H\right]^{\left[C\left(L\right)\right]}}B^{\left(0,h_{1}\right)T}M^{\left(1,h_{1}\right)}\cdots M^{\left(C\left(L\right),h_{1}\right)}A^{\left(0,h_{1}\right)}X\right)\\
		& \left(\sum_{h_{2}\in\left[H\right]^{\left[C\left(L\right)\right]}}X^{T}A^{\left(0,h_{2}\right)T}M^{\left(C\left(L\right),h_{2}\right)T}\cdots M^{\left(1,h_{2}\right)T}B^{\left(0,h_{2}\right)}\right)\left(W^{\textrm{K},h}\right)^{T}W^{\textrm{Q},h}\\
		& \left(\sum_{h_{3}\in\left[H\right]^{\left[C\left(L\right)\right]}}^{H}B^{\left(0,h_{3}\right)T}M^{\left(1,h_{3}\right)}\cdots M^{\left(C\left(L\right),h_{3}\right)}A^{\left(0,h_{3}\right)}X\right)
		\end{align*}
		Finally unifying the summations over $h,h_{1}h_{2},h_{3}$ to single
		sum over $\left[H\right]^{\left[C\left(L\right)\cdot3+1=C\left(L+1\right)\right]}$ gives
		\begin{align}\label{eq:sa_abstract_structr_induction_step}
		\sum_{h\in\left[H\right]^{\left[C\left(L+1\right)\right]}}\underbrace{W^{\textrm{O},h}W^{\textrm{V},h}B^{\left(0,h\right)T}}_{\in\mathbb{R}^{d_{x}\times d_{a}}}M^{\left(1,h\right)}\cdots M^{\left(C\left(L\right),h\right)}\underbrace{A^{\left(0,h\right)}XX^{T}A^{\left(0,h\right)T}}_{\text{in the desired form of }M}M^{\left(C\left(L\right),h\right)T}\cdots M^{\left(2,h\right)T}\\
		\underbrace{M^{\left(1,h\right)T}B^{\left(0,h\right)}\left(W^{\textrm{K},h\left(0\right)}\right)^{T}W^{\textrm{Q},h}B^{\left(0,h\right)T}}_{\text{in the desired form of }M}M^{\left(1,h\right)}\cdots M^{\left(C\left(L\right),h\right)}A^{\left(0,h\right)}X
		\end{align}
		Note that the number of $M$ units, each with a summation on a different index $j\in[N]$, is $3C\left(L\right)+1=C\left(L+1\right)$, implying $C\left(L\right)=\frac{3^{L}-1}{2}$ as needed.
	\end{proof}
	\begin{corollary}
		Defining $C\left(L\right)\coloneqq\frac{3^{L}-1}{2}$, any depth $L$
		composition of $L$ self-attention layers can be written as:
		\begin{align}\label{eq:gl_explicit_form}
		\y^{i, L, d_x, H, \Theta}\left(\x^1,...,\x^N\right)=\sum_{j_1,...,j_C=1}^{N}\mathbf{g}^L\left(\x^i,\x^{j_1},...,\x^{j_C}\right)
		\end{align}	
		Where
		\begin{align*}
		\mathbf{g}^L\left(\x^i,\x^{j_1},...,\x^{j_C}\right)\coloneqq\sum_{h\in\left[H\right]^{\left[C\left(L\right)\right]}}\sum_{r_{1},\dots,r_{C\left(L\right)+1}=1}^{d_{a}} \left[B^{\left(0,h\right)}\right]_{r_{1},p}\left(\prod_{c=1}^{C\left(L\right)}\left\langle A_{r_{c}}^{\left(c,h\right)},\x^{\left(j_{c}\right)}\right\rangle \left\langle B_{r_{c+1}}^{\left(c,h\right)},\x^{\left(j_{c}\right)}\right\rangle \right)\left\langle A_{r_{C\left(L\right)+1}}^{\left(0,h\right)},\x^{\left(i\right)}\right\rangle 
		\end{align*}
	\end{corollary}
	\begin{proof}
		To get the required form, we will use lemma~\ref{lemma:sa_abstract_structre} above and write the matrix multiplication in eq.~\eqref{eq:sa_abstract_structre} explicitly.
		$$
		M^{\left(c,h\right)}_{r_{1},r_{2}}=\sum_{j=1}^{N}\left[A^{\left(c,h\right)}X\right]_{r_{1},j}\left[X^{T}B^{\left(c,h\right)T}\right]_{j,r_{2}}=\sum_{j=1}^{N}\left\langle A_{r_{1}}^{\left(c,h\right)},\x^{\left(j\right)}\right\rangle \left\langle B_{r_{2}}^{\left(c,h\right)},\x^{\left(j\right)}\right\rangle 
		$$
		
		Therefore
		\begin{align*}
		y_{p}^{i,L,d_{x},H,\Theta}\left(\x^{\left(1\right)},...,\x^{\left(N\right)}\right) & =\sum_{h\in\left[H\right]^{\left[C\left(L\right)\right]}}B_{p}^{\left(0,h\right)T}M^{\left(1,h\right)}\cdots M^{\left(C\left(L\right),h\right)}A^{\left(0,h\right)}\x^{\left(i\right)}\\
		& =\sum_{j_{1},\dots,j_{C\left(L\right)}=1}^{N}\sum_{h\in\left[H\right]^{\left[C\left(L\right)\right]}}\sum_{r_{1},\dots,r_{C\left(L\right)+1}=1}^{d_{a}}\\
		& \left[B^{\left(0,h\right)}\right]_{r_{1},p}\left(\prod_{c=1}^{C\left(L\right)}\left\langle A_{r_{c}}^{\left(c,h\right)},\x^{\left(j_{c}\right)}\right\rangle \left\langle B_{r_{c+1}}^{\left(c,h\right)},\x^{\left(j_{c}\right)}\right\rangle \right)\left\langle A_{r_{C\left(L\right)+1}}^{\left(0,h\right)},\x^{\left(i\right)}\right\rangle 
		\end{align*}

	\end{proof}
	In the next two subsections, we will use the above lemma~\ref{lemma:sa_abstract_structre} to prove the two competing upper bounds on the separation rank of self-attention networks.
	
	\subsection{Proof of the upper bound in theorem~1}
	In the following theorem, we show how an upper bound on the separation rank is implied by the form of eq.~\eqref{eq:sa_abstract_structre} in the statement of lemma~\ref{lemma:sa_abstract_structre}.
	\begin{theorem}
		Defining $C\left(L\right)\coloneqq\frac{3^{L}-1}{2}$, for any depth $L\ge1$ input size $N>1$ partition $P\cupdot Q=\left[N\right]$ and
		output locations $i\in\left[N\right],\:p\in\left[d_{x}\right]$, the following holds:
		\[
		sep\left(y_{p}^{i,L,d_{x},H,\Theta},P,Q\right)\ensuremath{\leq\left(H\left(d_a+1\right)\right)^{C\left(L\right)}}
		\]
	\end{theorem}
	\begin{proof}
		We begin by writing the matrix multiplication in eq.~\eqref{eq:sa_abstract_structre} explicitly.
		$$
		M^{\left(c,h\right)}_{r_{1},r_{2}}=\sum_{j=1}^{N}\left[A^{\left(c,h\right)}X\right]_{r_{1},j}\left[X^{T}B^{\left(c,h\right)T}\right]_{j,r_{2}}=\sum_{j\in P}\left\langle A_{r_{1}}^{\left(c,h\right)},\x^{\left(j\right)}\right\rangle \left\langle B_{r_{2}}^{\left(c,h\right)},\x^{\left(j\right)}\right\rangle+\sum_{j\in Q}\left\langle A_{r_{1}}^{\left(c,h\right)},\x^{\left(j\right)}\right\rangle \left\langle B_{r_{2}}^{\left(c,h\right)},\x^{\left(j\right)}\right\rangle
		$$
		Therefore, rewriting the summation to be over $\{P_c\in\{P,Q\}\}_{c=1}^{C(L)}$ that correspond to the two partition segments $P/Q$.
		\begin{align*}
		y_{p}^{i,L,d_{x},H,\Theta}\left(\x^{\left(1\right)},...,\x^{\left(N\right)}\right) & =\sum_{h\in\left[H\right]^{\left[C\left(L\right)\right]}}B_{p}^{\left(0,h\right)T}M^{\left(1,h\right)}\cdots M^{\left(C\left(L\right),h\right)}A^{\left(0,h\right)}\x^{\left(i\right)}\\
		& =\sum_{h\in\left[H\right]^{\left[C\left(L\right)\right]}}\sum_{r_{1},\dots,r_{C\left(L\right)+1}=1}^{d_{a}}\sum_{P_{1},\dots,P_{C\left(L\right)}\in\left\{ P,Q\right\} }\\
		& B^{\left(0,h\right)}_{r_{1},p}\left(\prod_{c=1}^{C\left(L\right)}\sum_{j\in P_{c}}\left\langle A_{r_{c}}^{\left(c,h\right)},\x^{\left(j\right)}\right\rangle \left\langle B_{r_{c+1}}^{\left(c,h\right)},\x^{\left(j\right)}\right\rangle \right)\left\langle A_{r_{C\left(L\right)+1}}^{\left(0,h\right)},\x^{\left(i\right)}\right\rangle 
		\end{align*}
		Now we reorder the above sum by summing over indices of swaps between $P$ and $Q$, i.e. $\beta\in[C]$ such that  $P_{\beta}\ne P_{\beta+1}$, and split the multiplication $\prod_{c=1}^{C\left(L\right)}$ according to the crossing indices:
		\begin{align*}
		& =\sum_{h\in\left[H\right]^{\left[C\left(L\right)\right]}}\sum_{r_{1},\dots,r_{C\left(L\right)+1}=1}^{d_{a}}\sum_{b=0}^{C\left(L\right)}\sum_{0=\beta_{b+1}<\beta_{b}\le\beta_{b-1}\ldots\le\beta_{1}<\beta_{0}=C\left(L\right)}B_{r_{1},p}^{\left(0,h\right)}\\
		& \left(\left(\prod_{m=0}^{\left\lfloor \frac{b}{2}\right\rfloor }\prod_{c=\beta_{2m+1}+1}^{\beta_{2m}}\sum_{j\in P}\left\langle A_{r_{c}}^{\left(c,h\right)},\x^{\left(j\right)}\right\rangle \left\langle B_{r_{c+1}}^{\left(c,h\right)},\x^{\left(j\right)}\right\rangle \right)\left\langle A_{r_{C\left(L\right)+1}}^{\left(0,h\right)},\x^{\left(i\right)}\right\rangle \right)\\
		& \left(\prod_{m=0}^{\left\lceil \frac{b}{2}\right\rceil -1}\prod_{c=\beta_{2m+2}+1}^{\beta_{2m+1}}\sum_{j\in Q}\left\langle A_{r_{c}}^{\left(c,h\right)},\x^{\left(j\right)}\right\rangle \left\langle B_{r_{c+1}}^{\left(c,h\right)},\x^{\left(j\right)}\right\rangle \right)
		\end{align*}
		Where we assume w.l.o.g that $i\in P$ and therefore $P_{\beta_{1}},P_{\beta_{1}+1},\dots,P_{\beta_{0}-1},P_{\beta_{0}}=P$.
		The above reordering allows pushing the summation of non swapping $r_c$ indices into the $P,Q$ parentheses:
		\begin{align}\label{eq:crossing_idices}
		& =\sum_{h\in\left[H\right]^{\left[C\left(L\right)\right]}}\sum_{b=0}^{C\left(L\right)}\sum_{0=\beta_{b+1}<\beta_{b}\le\beta_{b-1}\ldots\le\beta_{1}\leq\beta_{0}=C\left(L\right)}\sum_{r_{\beta_{1}+1},\dots,r_{\beta_{b}+1}=1}^{d_{a}}B_{r_{1},p}^{\left(0,h\right)}\\\nonumber
		&\underbrace{\left(\underbrace{\sum_{r_{C\left(L\right)+1}=1}^{d_{a}}}_{\substack{\text{just for }\beta_{1}<C\\
					\text{ otherwise ignore}
				}
			}\underbrace{\sum_{r_{1}=1}^{d_{a}}}_{\substack{\text{used either}\\
					\text{in }P\text{ or }Q
				}
			}\left(\prod_{m=0}^{\left\lfloor \frac{b}{2}\right\rfloor }\sum_{\substack{r_{\beta_{2m+1}+2}\\
					\vdots\\
					r_{\beta_{2m}}
				}
				=1}^{d_{a}}\prod_{c=\beta_{2m+1}+1}^{\beta_{2m}}\sum_{j\in P}\left\langle A_{r_{c}}^{\left(c,h\right)},\x^{\left(j\right)}\right\rangle \left\langle B_{r_{c+1}}^{\left(c,h\right)},\x^{\left(j\right)}\right\rangle \right)\left\langle A_{r_{C\left(L\right)+1}}^{\left(0,h\right)},\x^{\left(i\right)}\right\rangle \right)}_{\text{function of }P}\\\nonumber
		&\underbrace{\left(\underbrace{\sum_{r_{1}=1}^{d_{a}}}_{\substack{\text{used either}\\
					\text{in }P\text{ or }Q
				}
			}\prod_{m=0}^{\left\lceil \frac{b}{2}\right\rceil -1}\sum_{r_{\beta_{2m+2}+2},\dots,r_{\beta_{2m+1}}=1}^{d_{a}}\prod_{c=\beta_{2m+2}+1}^{\beta_{2m+1}}\sum_{j\in Q}\left\langle A_{r_{c}}^{\left(c,h\right)},\x^{\left(j\right)}\right\rangle \left\langle B_{r_{c+1}}^{\left(c,h\right)},\x^{\left(j\right)}\right\rangle \right)}_{\text{function of }Q}
		\end{align}

		Since the separation rank of each term in the above summation is $1$, we proved the following upper bound on the separation rank:
		
		\begin{align*}
		sep\left(y_{p}^{i,L,d_{x},H,\Theta},P,Q\right)\ensuremath{\leq\sum_{h\in\left[H\right]^{\left[C\left(L\right)\right]}}\sum_{b=0}^{C\left(L\right)}\sum_{0=\beta_{b+1}<\beta_{b}\le\beta_{b-1}\ldots\le\beta_{1}\leq\beta_{0}=C\left(L\right)}\sum_{r_{\beta_{1}+1},\dots,r_{\beta_{b}+1}=1}^{d_{a}}1\\=H^{C\left(L\right)}\sum_{b=0}^{C\left(L\right)}\binom{C\left(L\right)}{b}\left(d_{a}\right)^{b}=H^{C\left(L\right)}\left(d_{a}+1\right)^{C\left(L\right)}=\left(H\left(d_a+1\right)\right)^{C\left(L\right)}}
		\end{align*}
		
		We  note that unlike the $d_{a}$ case, the same $H$ index can affect nonconsecutive $M^{\left(c_{1},h\right)},M^{\left(c_{2},h\right)}$, therefore we can't simply push the $h$ indices as done for the $r$ indices in eq.~\eqref{eq:crossing_idices}.
	\end{proof}
	From here, the upper bound in theorem~\ref{theorem:depth_efficiency} follows by
	
	\begin{align}\label{eq:depth_upper_sep}
	\log_{3}\left(sep\left(y_{p}^{i,L,d_{x},H,\Theta},P,Q\right)\right)\leq\log_{3}\left(\left(H\left(d_{a}+1\right)\right)^{C\left(L\right)}\right)=\frac{3^{L}-1}{2}\log_{3}\left(d_{x}+H\right)
	\end{align}
	\subsection{Proof of the upper bound in theorem~2}

	In the following theorem, we show how an upper bound on the separation rank is implied by the polynomial degree of $y_p^{i,L,d_{x},H,\Theta}$ in eq.~\eqref{eq:gl_explicit_form}. We will use the notation of $\multiset{n}{k}$ -- the multiset coefficient, given in the binomial form by
	$\binom{n+k-1}{k}$. We will use the identity $\abs{\left\{a_1\ldots a_n\in\mathbb{Z}\geq0:\sum_{r=1}^{n}a_{r}=k\right\} } = \multiset{n}{k}$.
	
	\begin{theorem}\label{theorem:dx_upper_bound}
		Defining $C\left(L\right)\coloneqq\frac{3^{L}-1}{2}$, for any depth $L\ge1$ input size $N>1$ partition $P\cupdot Q=\left[N\right]$ and
		output locations $i\in\left[N\right],\:p\in\left[d_{x}\right]$, the following holds:
		\begin{align}\label{eq:dx_upper_bound}
		sep\left(y_{p}^{i,L,d_{x},H,\Theta},P,Q\right)\ensuremath{\leq d_{x}\left(C\left(L\right)+1\right)\left(\binom{d_{x}}{2C\left(L\right)}\right)\left(\frac{2C\left(L\right)}{d_{x}}+1\right)^{d_{x}}}
		\end{align}
	\end{theorem}
	\begin{proof}
		We begin by opening the inner products in eq.~\eqref{eq:gl_explicit_form}, explicitly writing the indices:
		\begin{align*}
		y_{p}^{i,L,d_{x},H,\Theta}=\sum_{j_{1},\dots,j_{C\left(L\right)}=1}^{N}\sum_{h\in\left[H\right]^{\left[C\left(L\right)\right]}}\sum_{r_{1},\dots,r_{C\left(L\right)+1}=1}^{d_{a}}
		B^{\left(0,h\right)}_{r_{1},p}\left(\prod_{c=1}^{C\left(L\right)}\left\langle A_{r_{c}}^{\left(c,h\right)},\x^{\left(j_c\right)}\right\rangle \left\langle B_{r_{c+1}}^{\left(c,h\right)},\x^{\left(j_c\right)}\right\rangle \right)\left\langle A_{r_{C\left(L\right)+1}}^{\left(0,h\right)},\x^{\left(i\right)}\right\rangle\\
		=\sum_{\alpha_{1},\dots,\alpha_{C\left(L\right)+1},\beta_{1},\dots,\beta_{C\left(L\right)}=1}^{d_{x}}\sum_{j_{1},\dots,j_{C\left(L\right)}=1}^{N}\sum_{h\in\left[H\right]^{\left[C\left(L\right)\right]}}\sum_{r_{1},\dots,r_{C\left(L\right)+1}=1}^{d_{a}}\\
		\left(B^{\left(0,h\right)}_{r_{1},p}A_{r_{C\left(L\right)+1},\alpha_{C\left(L\right)}+1}^{\left(0,h\right)}\x_{\alpha_{C\left(L\right)+1}}^{\left(i\right)}\prod_{c=1}^{C\left(L\right)}A_{r_{c},\alpha_{c}}^{\left(c,h\right)}\x_{\alpha_{c}}^{\left(j_{c}\right)}B_{r_{c},\beta_{c}}^{\left(c,h\right)}\x_{\beta_{c}}^{\left(j_{c}\right)}\right)
		\end{align*}
		And separating between  coefficients and $\x$'s:
		
		\begin{align*}
		=\sum_{\alpha_{1},\dots,\alpha_{C\left(L\right)+1},\beta_{1},\dots,\beta_{C\left(L\right)}=1}^{d_{x}}\underbrace{\left(\sum_{h\in\left[H\right]^{\left[C\left(L\right)\right]}}\sum_{r_{1},\dots,r_{C\left(L\right)+1}=1}^{d_{a}}B^{\left(0,h\right)}_{r_{1},p}
			A_{r_{C\left(L\right)+1},\alpha_{C\left(L\right)}+1}^{\left(0,h\right)}\prod_{c=1}^{C\left(L\right)}A_{r_{c},\alpha_{c}}^{\left(c,h\right)}B_{r_{c},\beta_{c}}^{\left(c,h\right)}\right)}_{\coloneqq\mathcal{T}_{\alpha_{1},\dots,\alpha_{C\left(L\right)+1},\beta_{1},\dots,\beta_{C\left(L\right)}}}\\\left(\sum_{j_{1},\dots,j_{C\left(L\right)}=1}^{N}\x_{\alpha_{C\left(L\right)+1}}^{\left(i\right)}\prod_{c=1}^{C\left(L\right)}\x_{\alpha_{c}}^{\left(j_{c}\right)}\x_{\beta_{c}}^{\left(j_{c}\right)}\right)
		\end{align*}
		
		Now we can group monomials by the powers $n_1,\dots,n_{d_x}$ of each coordinate

		\begin{align*}
		=\sum_{\alpha_{C\left(L\right)+1}}^{d_x}\underbrace{\sum_{n_1+\dots n_{d_x}=2C\left(L\right)}}_{\text{The powers}}\left(\overbrace{\sum_{\substack{\alpha_{!},\dots,\alpha_{C\left(L\right)},\beta_{1},\dots,\beta_{C\left(L\right)}\in\left[d_{x}\right]\\\forall m\in\left[d_{x}\right]\,\left|\left\{ c\in\left[C\left(L\right)\,:\alpha_{c}=m\right]\right\} \right|+\left|\left\{ c\in\left[C\left(L\right)\,:\beta_{c}=m\right]\right\} \right|=n_{m}\\}}}^{\text{How to distributethe powers between the }c\text{'s}}\mathcal{T}_{\alpha_{1},\dots,\alpha_{C\left(L\right)+1},\beta_{1},\dots,\beta_{C\left(L\right)}}\right)\\\chi_{n_{1}.\dots,n_{d_{x}},\alpha_{C\left(L\right)+1}}\left(\x^{\left(1\right)},\dots,\x^{\left(N\right)}\right)
		\end{align*}
		Where:
		$$\chi_{n_{1}.\dots,n_{d_{x}},\alpha_{C\left(L\right)+1}}\left(\x^{\left(1\right)},\dots,\x^{\left(N\right)}\right)\coloneqq\underbrace{\sum_{o_{1}+\dots+o_{N}=C\left(L\right)}}_{\substack{\text{How many }j\text{ indices}\\
				\text{ \text{equal to each} }\left[N\right]
			}
		}\underbrace{\sum_{\substack{0\leq n_{1,1},\dots,n_{d_{x},N}\leq2C\left(L\right)\\
					\forall m\in\left[d_{x}\right]\,\sum_{j=1}^{N}n_{m,j}=n_{m}\\
					\forall j\in\left[N\right]\,\sum_{m=1}^{d_{x}}n_{m},j=2o_{j}
				}
		}}_{\text{How to distribute the powers between  }\left[N\right]}\x_{\alpha_{C\left(L\right)+1}}^{\left(i\right)}\prod_{j=1}^{N}\prod_{m=1}^{d_{x}}\left(\x_{m}^{\left(j\right)}\right)^{n_{m,j}}$$

		Finally, we need to bound the separation rank of $\chi_{n_{1}.\dots,n_{d_{x}},\alpha_{C\left(L\right)+1}}$.
		W.l.o.g we choose the partition $P=\left\{ 1,\dots,\frac{N}{2}\right\} ,Q=\left\{ \frac{N}{2}+1,\dots,N\right\} $
		and $i\in P$ then we can divide the powers between $P,Q$ in the following way:

		\begin{align*}
		\chi_{n_{1}.\dots,n_{d_{x}},\alpha_{C\left(L\right)+1}}\left(\x^{\left(1\right)},\dots,\x^{\left(N\right)}\right)=\sum_{\substack{0\leq r_{1,P},\dots,r_{d_{x},P}\leq2C\left(L\right)\\
				0\leq r_{1,Q},\dots,r_{d_{x},Q}\leq2C\left(L\right)\\
				\forall m\in\left[d_{x}\right]\,r_{m,P}+r_{m,Q}=n_{m}
			}
		}\sum_{E=0}^{C\left(L\right)}\\
		\underbrace{\left(\sum_{o_{1}+\dots+o_{\frac{N}{2}}=E}\sum_{\substack{0\leq n_{1,1},\dots,n_{d_{x},\frac{N}{2}}\leq2C\left(L\right)\\
					\forall m\in\left[d_{x}\right]\,\sum_{j\in P}^{N}n_{m,j}=r_{m,P}\\
					\forall j\in\left[N\right]\,\sum_{m=1}^{d_{x}}n_{m},j=2o_{j}
				}
			}\x_{\alpha_{C\left(L\right)+1}}^{\left(i\right)}\prod_{j\in P}\prod_{m=1}^{d_{x}}\left(\x_{m}^{\left(j\right)}\right)^{n_{m,j}}\right)}_{\text{function of }P}\\	
		\underbrace{\left(\sum_{o_{\frac{N}{2}+1}+\dots+o_{N}=C\left(L\right)-E}\sum_{\substack{0\leq n_{1,1},\dots,n_{d_{x},\frac{N}{2}}\leq2C\left(L\right)\\
					\forall m\in\left[d_{x}\right]\,\sum_{j\in Q}^{N}n_{m,j}=r_{m,Q}\\
					\forall j\in\left[N\right]\,\sum_{m=1}^{d_{x}}n_{m},j=2o_{j}
				}
			}\prod_{j\in Q}\prod_{m=1}^{d_{x}}\left(\x_{m}^{\left(j\right)}\right)^{n_{m,j}}\right)}_{\text{function of }Q}	
		\end{align*}
		Thus, since each summand is of separation rank $1$, the separation rank of $\chi_{n_{1}.\dots,n_{d_{x}},\alpha_{C\left(L\right)+1}}$
		is bounded by the number of summands:
		\[
		\left(C\left(L\right)+1\right)\prod_{\beta=1}^{d_{x}}\left(\binom{2}{r_{\beta}}\right)\overbrace{\leq}^{\text{lemma~\ref{means_inequality}}}\left(C\left(L\right)+1\right)\left(\frac{2C\left(L\right)}{d_{x}}+1\right)^{d_{x}}
		\]
		
		where the inequality followed from lemma \ref{means_inequality}.
		Since we have at most $d_{x}\left(\binom{d_{x}}{2C\left(L\right)}\right)$ different $\chi$ we conclude that:
		\[
		\text{\ensuremath{sep\left(y_{p}^{i,L,d_{x},H,\Theta},P,Q\right)\leq \underbrace{d_{x}\left(\binom{d_{x}}{2C\left(L\right)}\right)}_{\text{number of }\chi}\left(C\left(L\right)+1\right)\left(\frac{2C\left(L\right)}{d_{x}}+1\right)^{d_{x}}}}
		\]
		\begin{align*}
		\end{align*}
	\end{proof}
	
	From here, theorem~\ref{theorem:width_bound} follows by the multiset identity in lemma~\ref{lemma:multiset_bound}: 
	\begin{align}\label{eq:upper2}\log_{3}\left[sep\left(y_{p}^{i,L,d_{x},H,\Theta},P,Q\right)\right] & \leq\log_{3}\left[d_{x}\left(C\left(L\right)+1\right)\left(\binom{d_{x}}{2C\left(L\right)}\right)\left(\frac{2C\left(L\right)}{d_{x}}+1\right)^{d_{x}}\right]\\\nonumber
	& \le\log_{3}\left[d_{x}\left(C\left(L\right)+1\right)\left(\frac{2e\left(d_{x}+2C\left(L\right)\right)}{d_{x}}\right)^{d_{x}}\left(\frac{2C\left(L\right)}{d_{x}}+1\right)^{d_{x}}\right]\\\nonumber
	& \le\log_{3}\left[3^{L}d_{x}\left(2e\right)^{d_{x}}\left(\frac{3^{L}-1}{d_{x}}+1\right)^{2d_{x}}\right]\\\nonumber
	&\overbrace{\le}^{\textrm{only for $3^L> d_x$}}
	\log_{3}[3^{L}d_{x}\left(2e\right)^{d_{x}}(2\cdot\frac{3^{L}-1}{d_{x}})^{2d_{x}}]\\\nonumber
	& \le L+\log_{3}d_{x}+d_{x}\log_{3}2e+2d_{x}\log_{3}\left\{ \left(\frac{2\cdot3^{L}}{d_{x}}\right)\right\} \\\nonumber
	& \le\left(2d_{x}+1\right)L+\log_{3}d_{x}+2d_{x}\left(\log_{3}2\sqrt{2e}-\log_{3}d_{x}\right)\qed
	\end{align}
	
	\subsection{{The effect of residual connections}}
	
	\begin{figure}[h]
		\centering
		\includegraphics[width=\linewidth]{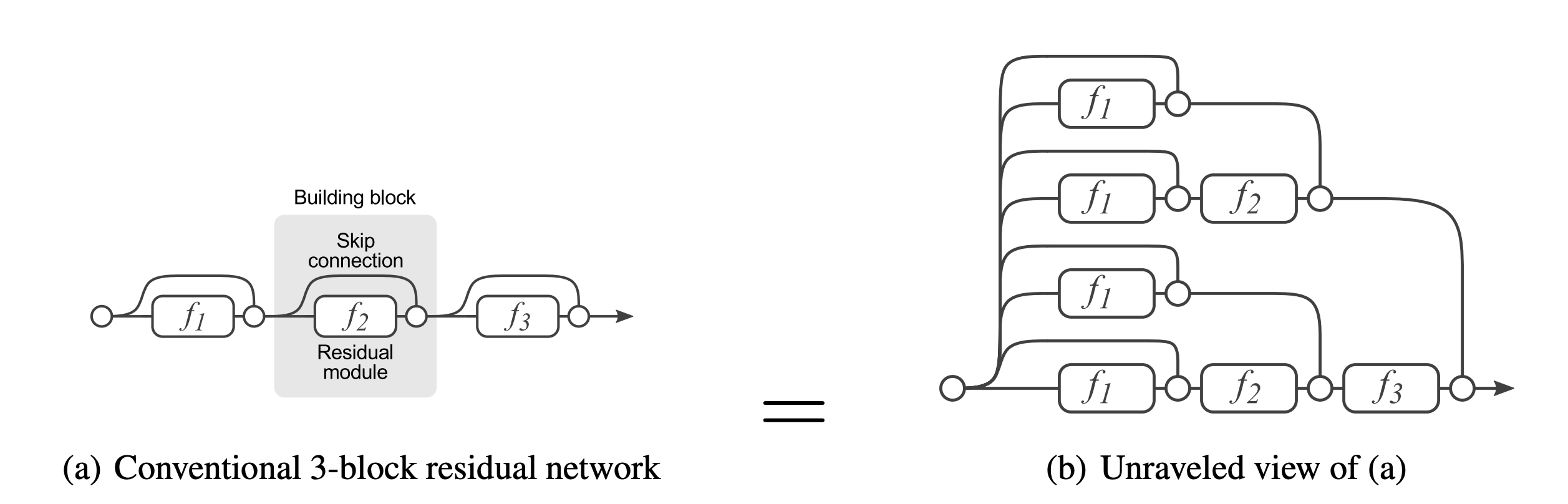}
		\vspace{-2mm} 
		\caption{A residual network in  its compressed and unraveled form, taken from~\cite{veit2016residual}.}
		\label{fig:residual}
		%\ifdefined\SQUEEZE \vspace{-4mm} \fi
	\end{figure}

	Having upper-bounded the separation rank of the deep self-attention network defined in section 2.2 of the main text, we comment on the effect of adding residual connections over each layer, as is done in the regular network (described in section 2.1 of the main text).
	Consider a network composed of a concatenation of the building blocks shown in figure~\ref{fig:residual}(a), taken from \citet{veit2016residual}. 
	A building block in layer $l$ includes a module $f_l$, which in our case is the self-attention layer given in eq.~(3) of the main text,\footnote{We have embedded the Feed-Forward layer within $W^{\textrm{O}}$ due to the linearity of the analyzed model.} and a skip connection which adds $f_l$'s input to its output (circles denote addition). \citet{veit2016residual} propose an unraveled view of such a network, shown in figure~\ref{fig:residual}(b), which we will employ in the proof of theorem~\ref{th:res_th_1} for clarity of presentation.
	
	We begin by proving a lemma that quantifies how the separation rank of the composition of a self-attention layer over a function is related to the function's separation rank:
	
	\begin{lemma} \label{lemma:sep_layer}
		Let $g^{j}\in\R^{d_x}$ be an input vector at position $j$ to a self-attention layer defined by eq.~(3) of the main text, and let $K $ be an upper bound to the separation rank of any of the entries $p\in [d_x]$ of any input $g^{j}\in\R^{d_x}$, i.e., $\forall p\in[d_x],j\in[N]:$ $Sep\left(g_{p}^{j}\right)\le K$.
		Let $y^{i}_p$ be the $p$th entry of the self-attention layer output at position $i$. Then, an upper bound to the separation rank of $y^{i}_p\in\R$ is given by:
		$$Sep\left(y_{p}^{i}\right)\le \frac{Nd_x^4}{H}K^3$$
	\end{lemma} 
	\begin{proof}
		Denote by $G\in\R^{d_x\times N}$ the matrix holding $g^{j}\in\R^{d_x}$ in its $j$th
		column. Note that by the conditions of the lemma any entry of $G$ upholds $Sep(G_{\alpha\beta})\le K$. Writing eq.~(3) of the main text as a summation over matrix indices:
		\begin{equation}
		y^{i}_p= \sum_{h=1}^{H}\sum_{\alpha_1=1}^{d_x/H}\sum_{\alpha_2=1}^{d_x}\sum_{j=1}^{N}\sum_{\alpha_4=1}^{d_x}\sum_{\alpha_5=1}^{d_x/H}\sum_{\alpha_6=1}^{d_x}W^{\textrm{O},h}_{p\alpha_1}W^{\textrm{V},h}_{\alpha_1\alpha_2}G_{\alpha_2j}(G^\top)_{j\alpha_4}(W^{\textrm{K},h})^\top_{\alpha_4\alpha_5}W^{\textrm{Q},h}_{\alpha_5\alpha_6}G_{\alpha_6i}\\
		\end{equation}
		The lemma follows by multiplying the number of summed terms by an upper bound on the separation rank of each summed term, $K^3$.
	\end{proof}
	
	We now prove a theorem which establishes that the integration of skip connections modifies the upper bound in theorem~1 of the main text by a small factor. 
	\begin{theorem}\label{th:res_th_1}
		For $p\in[d_x]$, let $y^{i, L}_p$ be the scalar function computing the $p$th entry of an output vector at position $i\in[N]$ of the depth-$L$
		residual network depicted in figure~\ref{fig:residual}, where $f_l$ is the self-attention layer in eq.~(3) of the main text.
		Then:
		$$\log_3 Sep(y_p^{L,i})\le L\log_3L+\left(4\log_3d_x+\log_3N-\log_3H\right)\cdot\frac{3^L-1}{2}$$
	\end{theorem}
	Comparing this dependence to the upper bound in the theorem~1 of the main text, given in eq.~\eqref{eq:depth_upper_sep}, this theorem implies that the effect of residual connections is insignificant to our analysis.  
	
	\begin{proof}
		Observing figure~\ref{fig:residual}(b) which gives the $L=3$ example, we upper bound the separation rank of the entire unraveled network by noting that its output is composed from $L+1$ additions of outputs from branches of depth $l=0,....,L$ ($0$ being the direct link of the input to the output), such that schematically the separation rank at the output of the entire network can be upper bounded by: $$Sep(y_p^{L,i})\le(L+1)Sep(\textrm{longest branch}(L))$$	
		where we denoted $\textrm{longest branch}(L)$ as the function at the output of $f_L$, before the addition with the other branches.
		Noting that the input to $f_L$ can be recursively viewed as an output of  an unraveled network of depth $L-1$, we bound the separation rank of the function at the input to $f_L$ by $L\cdot Sep(\textrm{longest branch}(L-1))$. 
		Since $f_L$ is a self-attention layer, 
		Lemma~\ref{lemma:sep_layer} implies that $Sep(\textrm{longest branch}(L))\le \frac{Nd_x^4}{H} \left(L\cdot Sep(\textrm{longest branch}(L-1))\right)^3$. Continuing recursively, and inserting the stopping condition $Sep(\textrm{longest branch}(L=1))=\frac{Nd_x^4}{H}$ (since the input to $f_1$ is a specific entry of the input to the entire network, of separation rank 1), we attain:
		$$Sep(y_p^{L,i})\le
		\prod_{l=1}^L(l+1)\left(\frac{Nd_x^4}{H}\right)^{3^{l-1}},$$
		satisfying the theorem.
	\end{proof}
	
	We now prove a theorem which establishes that the integration of skip connections modifies the upper bound in theorem~2 of the main text by a small factor. 
	\begin{theorem}
		Defining $C\left(L\right)\coloneqq\frac{3^{L}-1}{2}$, for $p\in[d_x]$, let $y_{p,\text{residual}}^{i,L,d_{x},H,\Theta}$ be the scalar function computing the $p$th entry of an output vector at position $i\in[N]$ of the depth-$L$
		residual network depicted in figure~\ref{fig:residual}, where $f_l$ is the self-attention layer in eq.~(3) of the main text. 
		Then for any partition $P\cupdot Q=\left[N\right]$, the following holds:
		\[
		sep\left(y_{p,\text{residual}}^{i,L,d_{x},H,\Theta},P,Q\right)\ensuremath{\leq d_{x}\left(C\left(L\right)+1\right)^{2}\left(\binom{d_{x}}{2C\left(L\right)}\right)\left(\frac{2C\left(L\right)}{d_{x}}+1\right)^{d_{x}}}
		\]
	\end{theorem}
	Comparing this dependence to the upper bound in the theorem~2 of the main text, given in eq.~\eqref{eq:dx_upper_bound}, the above theorem implies that the effect of residual connections is insignificant to our analysis. 
	\begin{proof}
		We will adapt the proof of  theorem~\ref{theorem:dx_upper_bound}. All of the arguments remain unchanged, except that we obtain the network structure via lemma~\ref{lemma:sa_abstract_resnet_structre} instead of lemma~\ref{lemma:sa_abstract_structre}.
		Following the new structure we will have  two additional summations, one over $j$ and one over $\alpha$ (see lemma~\ref{lemma:sa_abstract_resnet_structre}), as well as an additional input $X$ factor.
		We will leave the summation over $j$ during the whole proof, thus multiplying the separation rank by at most $C\left(L\right)$.  
		Note that similarity to the $h$ summation, the summation over $\alpha$ has no influence on the separation rank, since it collapses into a single  coefficient $\mathcal{T}$.
		Finally the unput $X$ factor contribute at most $1$ to the separation rank, therefore we can bound the separation rank by $C\left(L\right)+1$ times the bound  in eq.~\eqref{eq:dx_upper_bound}.
	\end{proof}
	
	Finally, since the upper bounds undergo such minor increases in the presence of skip connections, the lower bounds can be left with no further tightening, without affecting the analysis and its conclusions.
	
	\subsection{Technical lemmas}
	
	\begin{lemma}
		\label{means_inequality}(inequality of arithmetic and geometric multiset
		coefficient means)
		
		Let $n,k\in\mathbb{N}$ and $\phi:\mathbb{N}^{k}\rightarrow\mathbb{N}\coloneqq r_{1},\dots r_{k}\vdash\prod_{j=1}^{k}\left(\binom{n}{r_{j}}\right)$
		then:
		\[
		\forall r_{z},\dots r_{k}\in\mathbb{N}\quad\phi\left(r_{1},\dots r_{k}\right)\leq\frac{\left(\prod_{t=1}^{n-1}\left(\frac{M}{k}+t\right)\right)^{k}}{\left(\left(n-1\right)!\right)^{k}}
		\]
		where $M\coloneqq\sum_{j=1}^{k}r_{j}$
	\end{lemma}
	\begin{proof}
		Define $f_{t}\coloneqq\prod_{j=1}^{k}\left(r_{j}+t\right)$ and $\psi\coloneqq\prod_{t=1}^{n-1}f_{t}$
		than by the inequality of arithmetic and geometric means 
		\[
		\forall t\in\text{\ensuremath{\left[k\right]}\ensuremath{\quad\ensuremath{f_{t}\leq\left(\frac{1}{k}\sum_{j=1}^{k}\left(r_{j}+t\right)\right)^{k}}=\ensuremath{\left(\frac{M}{k}+t\right)^{k}}}}
		\]
		Therefore 
		\begin{align*}
		\phi\left(r_{1},\dots,r_{k}\right) & =\prod_{j=1}^{k}\left(\binom{n}{r_{j}}\right)=\prod_{j=1}^{k}\binom{n+r_{j}-1}{r_{j}}=\prod_{j=1}^{k}\frac{\left(n+r_{j}-1\right)!}{r_{j}!\left(n-1\right)!}\\
		& =\frac{1}{\left(\left(n-1\right)!\right)^{k}}\prod_{j=1}^{k}\prod_{t=1}^{n-1}\left(r_{j}+t\right)=\frac{1}{\left(\left(n-1\right)!\right)^{k}}\prod_{t=1}^{n-1}f_{t}\leq\frac{\prod_{t=1}^{n-1}\left(\frac{M}{k}+t\right)^{k}}{\left(\left(n-1\right)!\right)^{k}}
		\end{align*}
		One can see that when $M$ divided by $k$ it hold that
		\[
		\phi\left(\overbrace{\frac{M}{k},\dots,\frac{M}{k}}^{k\:\text{times }}\right)=\frac{1}{\left(\left(n-1\right)!\right)^{k}}\prod_{t=1}^{n-1}f_{t}=\frac{1}{\left(\left(n-1\right)!\right)^{k}}\prod_{t=1}^{n-1}\left(\frac{M}{k}+t\right)^{k}=\left(\prod_{t=1}^{n-1}\left(\frac{M}{k}+t\right)\right)^{k}
		\]
		hence the name of this lemma. 
	\end{proof}
	
	\begin{lemma}\label{lemma:multiset_bound}
		$\ensuremath{\left(\binom{n}{k}\right)\leq\left(\frac{2e\left(n+k\right)}{n}\right)}^{n}$
	\end{lemma}
	\begin{proof}
		: by using the inequality $\binom{n}{k}\leq\left(\frac{en}{k}\right)^{k}$ we have$$\left(\binom{n}{k}\right)=\binom{n+k-1}{n-1}\leq\left(\frac{2e\left(n+k\right)}{n}\right)^{n}$$
	\end{proof}
	
	We now prove a lemma which reveals the alternation to the network structure as expressed in eq.~\eqref{eq:sa_abstract_structre} when taking skip connections into account.
	\begin{lemma}
		\label{lemma:sa_abstract_resnet_structre}Defining $C\left(L\right)\coloneqq\frac{3^{L}-1}{2}$, any depth $L$
		skip connection composition of the self-attention layers defined in eq.~\eqref{eq:our_layer} can be written as:
		
		\begin{align}\label{eq:sa_abstract_resnet_structre}
		Y^{L,d_{x},H,\Theta}=X+\sum_{j=1}^{C\left(L\right)}\sum_{\alpha=1}^{n_j}\sum_{h\in\left[H\right]^{\left[j\right]}}B^{\left(0,h,j,\alpha\right)T}M^{\left(1,h,j,\alpha\right)}\cdots M^{\left(j,h,j,\alpha\right)}A^{\left(0,h,j,\alpha\right)}X
		\end{align}
		
		where $\forall j\in\left[C\left(L\right)\right]\quad n_j\ge0$
		and $\forall\alpha\in\left[n_j\right]\, h\in\left[H\right]^{\left[j\right]}\,0\leq c\leq j:M^{\left(c,h\right)}=A^{\left(c,h\right)}XX^{T}B^{\left(c,h\right)T}$
		and $A^{\left(c,h\right)},B^{\left(c,h\right)}\in\mathbb{R}^{d_{a}\times d_{x}}$.
	\end{lemma}	
	\begin{proof}
		By Induction on $L$. Base case:
		
		\begin{align*}
		Y^{\left(1\right)}\left(X\right)=X+\sum_{h=1}^{H}\underbrace{W^{\textrm{O},h}}_{B^{T}}\underbrace{W^{\textrm{V},h}XX^{T}\left(W^{\textrm{K},h}\right)^{T}}_{M}\underbrace{W^{\textrm{Q},h}}_{A}X
		\end{align*}
		\begin{align*}
		Y^{\left(L+1\right)}\left(X\right)=\sum_{h=1}^{H}W^{\textrm{O},h}W^{\textrm{V},h}Y^{\left(L\right)}\left(X\right)Y^{\left(L\right)}\left(X\right)^{T}\left(W^{\textrm{K},h}\right)^{T}W^{\textrm{Q},h}Y^{\left(L\right)}\left(X\right)
		\end{align*}
		Now, rewriting $Y^{\left(L\right)}$ as $X + \left(Y^{\left(L\right)}-X\right)$ yields:
		\begin{align*}
		Y^{\left(L+1\right)}\left(X\right)=\sum_{E,F,G\in\left\{X,Y^{\left(L\right)}-X\right\}}\sum_{h=1}^{H}W^{\textrm{O},h}W^{\textrm{V},h}E F^{T}\left(W^{\textrm{K},h}\right)^{T}W^{\textrm{Q},h}G
		\end{align*}
		Now, substituting in the induction hypothesis on the structure of  $Y^{\left(L\right)}\left(X\right)$ yields:
		\begin{align*}
		Y^{\left(L+1\right)}\left(X\right)=\sum_{E,F,G\in\left\{X,\sum_{j=1}^{C\left(L\right)}\sum_{\alpha=1}^{n_j}\sum_{h\in\left[H\right]^{\left[j\right]}}B^{\left(0,h,j,\alpha\right)T}M^{\left(1,h,j,\alpha\right)}\cdots M^{\left(j,h,j,\alpha\right)}A^{\left(0,h,j,\alpha\right)}X\right\}}\\
		\sum_{h=1}^{H}W^{\textrm{O},h}W^{\textrm{V},h}E F^{T}\left(W^{\textrm{K},h}\right)^{T}W^{\textrm{Q},h}G
		\end{align*}
		Similarly to eq.~\eqref{eq:sa_abstract_structr_induction_step} each of the $8$ terms in the outer summation is of the required form, thus we complete the proof.
	\end{proof}

	\section{Lower bounds on the separation rank}
	\subsection{preliminaries}
	\subsubsection{Tensors and their matricization}
	We begin by laying out basic concepts in tensor theory required for
	the upcoming analysis. The core concept of a \emph{tensor} may be thought of as a
	multi-dimensional array. The \emph{order} of a
	tensor is defined to be the number of indexing entries in the array,
	referred to as \emph{modes}. The \emph{dimension} of a tensor in a particular
	mode is defined as the number of values taken by the index in that
	mode. If $\A$ is a tensor of order $N$ and dimension $M_i$ in each mode
	$i\in[N]$, its entries are denoted $\A_{d_1...d_N}$, where the index in each
	mode takes values $d_i\in [M_i]$.
	
	We will make use of the concept of the \emph{matricization of $\A$ \wrt~the balanced partition $(I,J)$}, denoted $\mat{\A}_{I,J}\in\R^{M^{\nicefrac{N}{2}}\times M^{\nicefrac{N}{2}}}$, which is essentially the arrangement of the tensor elements as a matrix whose rows correspond to $I$ and columns to $J$.
	Suppose $\A\in\R^{M{\times\cdots\times}M}$
	is a tensor of order $N$, and let $(I,J)$ be a balanced partition of $[N]$, \ie~$I$
	and~$J$ are disjoint size $\nicefrac{N}{2}$ subsets of $[N]$ whose union gives~$[N]$.
	The \emph{matricization of $\A$ \wrt~the partition $(I,J)$}, denoted
	$\mat{\A}_{I,J}$, is the $M^{{\nicefrac{N}{2}}}$-by-$M^{{\nicefrac{N}{2}}}$ matrix holding the entries of $\A$ such that $\A_{d_1{\ldots}d_N}$ is placed in row index $1+\sum_{t=1}^{{\nicefrac{N}{2}}}(d_{i_t}-1)M^{{\nicefrac{N}{2}}-t}$ and column index $1+\sum_{t=1}^{{\nicefrac{N}{2}}}(d_{j_t}-1)M^{{\nicefrac{N}{2}}-t}$.
	
	\subsubsection{Grid tensors provide lower bounds for the separation rank} 
	
	We now present the concept of grid tensors, which are a form of function discretization~\citep{hackbusch2012tensor}. Essentially, the function is
	evaluated for a set of points on an exponentially large grid in the
	input space and the outcomes are stored in a tensor. Formally, fixing a set of \emph{template} vectors
	$\x^{(1)},\ldots,\x^{(M)} \in \R^{d_x}$, the points on the grid are the set
	$\{(\x^{(d_1)},\ldots,\x^{(d_N)})\}_{d_1,\ldots,d_N=1}^M$. Given a function
	$y(\x^1,\ldots,\x^N)$, the set of its values on the grid arranged in the form of a tensor are
	called the grid tensor induced by $y$, denoted
	$\A(y)_{d_1,\ldots,d_N} \equiv y(\x^1=\x^{(d_1)},\ldots,\x^N=\x^{(d_N)})$.
	
	The following claim establishes a fundamental relation between a function's separation rank (see section~\ref{sec:sep_rank}) and the rank of the matrix obtained by the
	corresponding grid tensor matricization. This relation, which holds for all functions, is
	formulated below for functions realized by self-attention networks:
	\begin{claim}\label{claim:grid_sep_deep}
		For $p\in[d_x]$, let $y^{i, L, d_x, H, \Theta}_p$ be the scalar function computing the $p$th entry of an output vector at position $i\in[N]$ of the depth-$L$ self-attention network with hidden dimension $d_x$ and $H$ attention heads per layer, defined in eqs.~\eqref{eq:our_layer} and~\eqref{eq:our_network}. Then, for any integer $M$ and any set of template vectors $\x^{(1)},\ldots,\x^{(M)} \in \R^{d_x}$ it
		holds that:
		\begin{equation}
		\sep{I,J}{y^{i, L, d_x, H, \Theta}_p}\geq \rank{\mat{\A(y^{i, L, d_x, H, \Theta}_p)}_{I,J}},
		\end{equation}
		where $\A(y^{i, L, d_x, H, \Theta}_p)$ is the grid tensor of $y^{i, L, d_x, H, \Theta}_p$ with
		respect to the above template vectors.
	\end{claim}
	\ifdefined\SQUEEZE \vspace{-4mm} \fi
	\begin{proof}
		If $\sep{I,J}{y^{i, L, d_x, H, \Theta}_p} = \infty$ then the inequality is trivially
		satisfied. Otherwise, assume that
		$\sep{I,J}{y^{i, L, d_x, H, \Theta}_p} = K \in \N$, and let $\{g_\nu^I, g_\nu^J\}_{\nu=1}^K$
		be the functions of the respective decomposition to a sum of separable
		functions, i.e. that the following holds:
		\begin{align*}
		y^{i, L, d_x, H, \Theta}_p(\x^1,\ldots,\x^N)
		&= \sum_{\nu=1}^K g_\nu^I(\x^j:j\in I)
		\cdot g_\nu^J(\x^j:j\in J).
		\end{align*}
		Then, by definition of the grid tensor, for any template vectors $\x^{(1)},\ldots,\x^{(M)}\in \R^{d_x}$ the following
		equality holds:
		\begin{align*}
		\A(y^{i, L, d_x, H, \Theta}_p)_{d_1,\ldots,d_N} &=
		\sum_{\nu = 1}^K g_\nu^I(\x^{(d_j)}:j\in I)
		\cdot g_\nu^J(\x^{(d_j)}:j\in J) \\
		&\equiv \sum_{\nu=1}^K V^\nu_{d_j:j\in [I]} U^\nu_{d_j:j\in [J]},
		\end{align*}
		where $V^\nu$ and $U^\nu$ are the tensors holding the values of $g_\nu^I$
		and $g_\nu^J$, respectively, at the points defined by the template vectors.
		Under the matricization according to the $(I,J)$ partition, it holds that
		$\mat{V^\nu}_{I,J}$ and $\mat{U^\nu}_{I,J}$ are column and row vectors,
		respectively, which we denote by $\vv_\nu$ and $\uu_\nu^T$. It follows that the
		matricization of the grid tensor is given by:
		\begin{align*}
		\mat{\A(y^{i, L, d_x, H, \Theta}_p)}_{I,J} &= \sum_{\nu=1}^K \vv_\nu \uu_\nu^T,
		\end{align*}
		which means that
		$\rank{\mat{\A(y^{i, L, d_x, H, \Theta}_p)}_{I,J}}\leq K=\sep{I,J}{y^{i, L, d_x, H, \Theta}_p}$.
	\end{proof}
	\vspace{2mm}

	\subsubsection{Method for bounding the grid tensor's rank}
	Claim~\ref{claim:grid_sep_deep} assures us that
	the separation rank of the function realized by a self-attention network is lower bounded by the rank of the matrix obtained by the
	corresponding grid tensor matricization, for any choice of template
	vectors. Specifically:
	\begin{equation*}
	\sep{I,J}{y^{i, L, d_x, H, \Theta}_p}
	\geq \rank{\mat{\A(y^{i, L, d_x, H, \Theta}_p)}_{I,J}}.
	\end{equation*}
	Thus, proving that
	$\rank{\mat{\A(y^{i, L, d_x, H, \Theta}_p)}_{I,J}} $
	is higher than the lower bounds stated in theorems~\ref{theorem:depth_efficiency} and~\ref{theorem:width_bound}
	for all of the values of the parameters $\Theta$ but a set of Lebesgue measure
	zero, would satisfy the theorems.
	
	We note that since the network's operation is polynomial in $\Theta$, then the entries of the grid tensor are also polynomial. 
	\cite{sharirtractable} prove a claim regarding the prevalence of the maximal matrix rank
	for matrices whose entries are polynomial functions. Essentially, they show that it suffices to find a single configuration of the parameters, denoted $\theta\in \R^{K}$ (where $K$ is the number of scalar parameters), for which the resultant matrix is of rank $r$, in order to show the rank is at least $r$ for all configurations in $\R^K$ but a set of measure zero in $\R^K$. 
	For simplicity of the proof we will find a single configuration $\theta\in\C^K$ for which the resultant matrix is of the required rank. We therefore modify the original claim to fit this setting, still proving the rank is lower bounded for all configurations in $\R^K$ but a set of measure zero in $\R^K$:
	\begin{claim} \label{claim:rank_everywhere}
		Let $M, N, K \in \N$, $1 \leq r \leq \min\{M,N\}$ and an $M\times N$ matrix $A$ where each entry is a polynomial mapping $A_{ij}$ over $K$ variables
		for every $i \in [M]$ and $j\in [N]$. If there exists a
		point $\theta\in\mathbb{F}^K$, where $\mathbb{F}$ is either $\R$ or $\C$,  s.t. ${\textrm {rank}}{(A(\theta))} \geq r$, then the set
		$\{\theta \in \R^K : \textrm{rank}{(A(\theta))} < r\}$ has zero measure (w.r.t. the
		Lebesgue measure over $\R^K$).
	\end{claim}
	\begin{proof} (based on a proof in \cite{sharirtractable})
		Recall that $\rank{A(\theta)} \geq r $ iff there exits a non-zero $r \times r$ minor of $A(\theta)$. 
		Note that a minor of $A(\theta)$ is polynomial in the entries of $A(\theta)$, and so it is polynomial in $\theta$ as well. Let $c = {M \choose r} \cdot {N \choose r}$
		be the number of minors in $A$, denote the minors by $\{f_i(\theta)\}_{i=1}^c$, and define a new polynomial
		function $f(\theta) = \sum_{i=1}^c f_i(\theta)^2$. It thus holds that $f(\theta) = 0$ iff for all $i \in [c]$ it holds that $f_i(\theta) = 0$,
		i.e. $f(\theta) = 0$ iff $\rank{A(\theta)} < r$.
		
		Now, $f(\theta)$ is a polynomial in the entries of $\theta$, and so it either vanishes on a set of zero measure in $\R^K$, or
		it is the zero polynomial~(see \citet{caron2005zero} for proof). Since we assumed that there exists $\theta \in \mathbb{F}^K$ s.t.
		$\textrm{rank}(A(\theta)) \geq r$, the latter option is not possible. 
	\end{proof}

	\subsection{Proof of the lower bounds in theorems~\ref{theorem:depth_efficiency} and~\ref{theorem:width_bound}}
	
	In this section, we show there exists an assignment for the  weight matrices of a self-attention network, along with a specific choice of template vectors, for
	which $\rank{\mat{\A(y^{i, L, d_x, H, \Theta}_p)}_{I,J}}$  surpasses the lower bounds stated in theorems~\ref{theorem:depth_efficiency} and~\ref{theorem:width_bound} in the appropriate depth to width ratios.
	In accordance with Claim~\ref{claim:rank_everywhere}, the lower bounds in the theorems will follow since such an assignment implies this rank is achieved for all configurations of
	the self-attention network weights but a set of Lebesgue measure zero.
	
	\begin{proof}(of lower bounds in theorems~\ref{theorem:depth_efficiency} and~\ref{theorem:width_bound}).
		
		Relying on claim~\ref{claim:grid_sep_deep} we will bound the separation rank from below via the rank of the matricization \wrt~a partition $(I,J)$ of a grid tensor induced by $y^{i, L, d_x, H, \Theta}_p$, computed by any set of template vectors: $\sep{I,J}{y^{i, L, d_x, H, \Theta}_p}\geq \rank{\mat{\A(y^{i, L, d_x, H, \Theta}_p)}_{I,J}}$.
		Relying on claim~\ref{claim:rank_everywhere}, we ensure that the rank of $\mat{\A(y^{i, L, d_x, H, \Theta}_p)}_{I,J}$ is above a certain value almost everywhere by finding an assignment of the network parameters for which it achieves this value.
		
		Lemma~\ref{lemma:assignment} assures us that for any matrix $V \in \R^{\nicefrac{M}{2} \times \nicefrac{(d_x - H)}{2}}$ with $l^2$ normalized rows, there exists a choice of $M+1$ template vectors $\x^{(1)},\ldots,\x^{(M+1)} \in \R^{d_x}$, as well as an assignment to the self-attention network weights 
		for which:
		\begin{equation}
		\mat{\A(y^{i, L, d_x, H, \Theta}_p)}_{\tilde{I},\tilde{J}}=\textrm{Const.} \cdot \left(V V^T\right)^{\odot (3^{L-2})},
		\end{equation}
		where $\mat{\A(y^{i, L, d_x, H, \Theta}_p)}_{\tilde{I},\tilde{J}}$ is a sub-matrix of the grid tensor matricization $\mat{\A(y^{i, L, d_x, H, \Theta}_p)}_{I,J}$ of size   $\nicefrac{M}{2}\times\nicefrac{M}{2}$ and $\odot$ represents the Hadamard power operation, i.e., $\left(A^{\odot k}\right)_{ij} = A_{ij}^k.$
		Since proving the existence of a sub-matrix of a certain rank lower-bounds the rank of the full matrix by this rank, it suffices to find a matrix $V$ such that $\rank{\left(V V^T\right)^{\odot (3^{L-2})}}$ upholds the stated dependence.
		
		Noting that the operation of raising a rank $r$ matrix to the Hadamard power of $p$ results in a matrix upper bounded by $\multiset{r}{p}$ (see proof in~\cite{amini2012low} for example) 
		with the notation of the multiset coefficient $\multiset{n}{k}:=\binom{n+k-1}{k}$, and that the rank of $V V^T$ is upper bounded by $\nicefrac{(d_x - H)}{2}$, we choose the dimension 
		$\nicefrac{M}{2}=\multiset{\nicefrac{\left(d_{x}-H\right)}{2}}{3^{L-2}}$ to facilitate the rank increase.
		
		For this choice, observe that it suffices to prove that the  sub-matrix $\mat{\A(y^{i, L, d_x, H, \Theta}_p)}_{\tilde{I},\tilde{J}}\in R^{\nicefrac{M}{2}\times \nicefrac{M}{2}}$ is fully ranked in order to satisfy the theorems. This follows by using the identity $\binom{n}{k}\geq\left(\frac{n}{k}\right)^k$ we have:
		$\multiset{n}{k}=\binom{n+k-1}{k}=\binom{n+k-1}{n-1}\geq\max\left\{\left(\frac{n-1}{k}+1\right)^k,\left(\frac{k}{n-1}+1\right)^{n-1}\right\}$
		
		And accordingly:
		$$\multiset{\nicefrac{\left(d_{x}-H\right)}{2}}{3^{L-2}}\geq\max\left\{\left(\frac{\nicefrac{\left(d_{x}-H\right)}{2}-1}{3^{L-2}}+1\right)^{3^{L-2}},\left(\frac{3^{L-2}}{\nicefrac{\left(d_{x}-H\right)}{2}-1}+1\right)^{\nicefrac{\left(d_{x}-H\right)}{2}-1}\right\}$$
		and the log of this bounds the expressions in the theorems' lower bounds, where for each regime the tighter lower bound is used.
		
		Defining for brevity $d:=\nicefrac{\left(d_{x}-H\right)}{2}$ and $\lambda:=3^{L-2}$, it remains only to find a specific matrix $V \in \R^{\multiset{d}{\lambda} \times d}$ with $l^2$ normalized rows such that the operation of taking the rank $d$ matrix $VV^\top$ to the Hadamard power of $\lambda$ would result in a fully ranked matrix.
		We will provide such a matrix, and prove for it that:
		\begin{equation}\label{eq:schmatic}
		\left(VV^\top\right)^{\odot\lambda}=\sum_{k=1}^{\multiset{d}{\lambda}} \aaa^{(k)} \otimes \bb^{(k)}\end{equation}
		for $\{\aaa^{(k)}\}_{k=1}^{\multiset{d}{\lambda}}$ and $\{\bb^{(k)}\}_{k=1}^{\multiset{d}{\lambda}}$ which are two	
		sets of linearly independent vectors.
		
		For $\alpha,\beta\in[\multiset{d}{\lambda}]$, observing an entry of  $\left(VV^\top\right)^{\odot\lambda}$:
		\begin{align}
		\ensuremath{\left(\left(VV^{\top}\right)^{\odot\lambda}\right)_{\alpha\beta}=}&\left(VV^{\top}\right)_{\alpha\beta}^{\lambda}=\left(\sum_{r=1}^{d}v_{r}^{\left(\alpha\right)}v_{r}^{\left(\beta\right)}\right)^{\lambda}=\\\label{eq:multinomial}\sum_{k_{1}+\cdots+k_{d}=\lambda}\left(\begin{matrix}\lambda\\
		k_{1},\ldots,k_{d}
		\end{matrix}\right)&\left[\prod_{r=1}^{d}\left(v_{r}^{\left(\alpha\right)}\right)^{k_{r}}\right]\left[\left[\prod_{r=1}^{d}\left(v_{r}^{\left(\beta\right)}\right)^{k_{r}}\right]\right]
		\end{align}	
		where the first equality follows from the definition of the Hadamard power, in the section we denoted $v_{r}^{\left(\alpha\right)},v_{r}^{\left(\beta\right)}$ as the $r$th entries in rows $\alpha$ and $\beta$ of $V$, and in the second line we expanded the power with the multinomial identity. 
		Identifying the form of eq.~\eqref{eq:multinomial} with the schematic form of eq.~\eqref{eq:schmatic}, it remains to find a specific matrix $V \in \R^{\multiset{d}{\lambda} \times d}$ with $l^2$ normalized rows for 
		which the size $\multiset{d}{\lambda}$ set $\left\{\aaa^{(k_{1},\ldots,k_{d})}\right\}_{k_{1}+\cdots+k_{d}=\lambda}$ is linearly independent, where $a_\alpha^{(k_{1},\ldots,k_{d})}= \prod_{r=1}^{d}\left(v_{r}^{\left(\alpha\right)}\right)^{k_{r}}$. 
		
		We show this is the case for $V$ in which the rows are each associated with one of $\multiset{d}{\lambda}$ configurations of distributing $d$ integer numbers that sum up to $\lambda$, \ie, in which each row is associated with specific $\left\{q_{1}^{\alpha},\ldots,q_{d}^{\alpha}\geq0,\sum_{r=1}^{d}q_{r}^{\alpha}=\lambda\right\} $. Explicitly, we take the rows $\vv_{r}^{\left(\alpha\right)}$ to be: 
		$$\forall r\in[d]:v_{r}^{\left(\alpha\right)}=\nicefrac{\Omega^{q_{r}^{\alpha}}}{\sqrt{\sum_{r'=1}^{d}\Omega^{2q_{r'}^{\alpha}}}}$$
		
		Given this $V$, each vector in the above defined set $\left\{\aaa^{(k_{1},\ldots,k_{d})}\right\}_{k_{1}+\cdots+k_{d}=\lambda}$  is equal to: 
		\begin{align*}
		a_{\alpha}^{(k_{1},\ldots,k_{d})}=\prod_{r=1}^{d}\left(v_{r}^{\left(\alpha\right)}\right)^{k_{r}}=&\prod_{r=1}^{d}\left(\frac{\Omega^{q_{r}^{\alpha}}}{\sqrt{\sum_{r'=1}^{d}\Omega^{2q_{r'}^{\alpha}}}}\right)^{k_{r}}=\frac{\prod_{r=1}^{d}\Omega^{q_{r}^{\alpha}k_{r}}}{\prod_{r=1}^{d}\left(\sum_{r'=1}^{d}\Omega^{2q_{r'}^{\alpha}}\right)^{\frac{k_{r}}{2}}}\\&={\left(\sum_{r'=1}^{d}\Omega^{2q_{r'}^{\alpha}}\right)^{-\frac{\lambda}{2}}}\cdot\left[\Omega^{\sum_{r=1}^{d}q_{r}^{\alpha}k_{r}}\right]
		\end{align*}
		Observing that the factor attained from the normalization depends only on the rows and doesn't vary with the different vectors labeled by $(k_{1},\ldots,k_{d})$, we note it does not affect their linear dependence (amounts to a multiplication by a diagonal matrix with non-zero entries on the diagonal - does not affect the rank). 
		
		We prove that the set$\left\{\hat{\aaa}^{(k_{1},\ldots,k_{d})}\right\}_{k_{1}+\cdots+k_{d}=\lambda}$ for $\hat{a}_{\alpha}^{(k_{1},\ldots,k_{d})}=\Omega^{\sum_{r=1}^{d}q_{r}^{\alpha}k_{r}}$ is linearly independent by arranging it as the columns of the matrix
		$A\in \R^{\multiset{d}{\lambda} \times \multiset{d}{\lambda}}$, and showing that $A$ is fully ranked.
		
		Since the elements of $A$ are polynomial in $\Omega$, then as
		lemma~\ref{lemma:poly_full_rank} shows, it is sufficient to show that there exists
		a single contributor to the determinant of $A$ that has the highest degree
		of $\Omega$ in order to ensure that the matrix is fully ranked for all values of $\Omega$
		but a finite set, so $\Omega$ should simply be chosen to be any number that is outside of this set. Observing the summands of the determinant, i.e.
		$ \Omega^{\sum_{q_{1}+\cdots+q_{d}=\lambda}
			\inprod{\q}{\sigma(\q)}}$,
		where $\sigma$ is a permutation on the columns of $A$, 
		lemma~\ref{lemma:rearrange} assures us the existence of a strictly maximal
		contributor, satisfying the conditions of lemma~\ref{lemma:poly_full_rank}, thus the set $\left\{\hat{\aaa}^{(k_{1},\ldots,k_{d})}\right\}_{k_{1}+\cdots+k_{d}=\lambda}$ is linearly
		independent, and the lower bounds in the theorems follow.
	\end{proof}

	\subsection{Technical lemmas}
	The following lemma details the assignment of the self-attention network weights and the choice of template vectors which help us establish theorem~\ref{theorem:depth_efficiency}.
	
	\begin{lemma}\label{lemma:assignment}
		For any balanced partition of $[N]$, denoted $(I,J)$, for any even $M$, and for any matrix $V \in \R^{\nicefrac{M}{2} \times \nicefrac{(d_x - H)}{2}}$ with rows that are $l^2$ normalized, there exists a choice of $M+1$ template vectors $\x^{(1)},\ldots,\x^{(M+1)} \in \R^{d_x}$, as well as an assignment to the self-attention network weights,
		for which:
		\begin{equation}\label{eq:submatrix}
		\mat{\A(y^{i, L, d_x, H, \Theta}_p)}_{\tilde{I},\tilde{J}}=\textrm{Const.} \cdot \left(V V^T\right)^{\odot 3^{L-2}},
		\end{equation}
		where $\mat{\A(y^{i, L, d_x, H, \Theta}_p)}_{\tilde{I},\tilde{J}}$ is a sub-matrix of the grid tensor matricization $\mat{\A(y^{i, L, d_x, H, \Theta}_p)}_{I,J}$ of size   $\nicefrac{M}{2}\times\nicefrac{M}{2}$ and $\odot$ represents the Hadamard power operation, i.e., $\left(A^{\odot k}\right)_{ij} = A_{ij}^k$.
	\end{lemma}
	\begin{proof}
		We present below a choice of weights and template vectors that yields the stated form for a sub-matrix of $\mat{\A(y^{i, L, d_x, H, \Theta}_p)}_{I,J}$. Subsequently we will plug these values into the self-attention operation stated in eq.~\eqref{eq:our_layer}, and prove that this form follows.
		
		Though the proof has many technical details, it has 3 essential parts. We first choose the weights of the first layer so that the outputs in all locations are the same and equal to a summation of the input vectors.
		Because the weight matrices are not $d_x \times d_x$ but are decomposed through the attention dimension $d_a \times d_x$ or $d_x \times d_a$, then we divide the coordinates of the $d_x$-dimensional vectors into contiguous segments of length $d_a$, and set the weights to either project these segments to the $d_a$-dimensional space or invert this mapping with added zero-padding.
		For the second part, we set the key and query matrices to use the same ``projections'' we used in the first layer to compute inner-products between each segment, while setting the value and output matrices to preserve each head's segment (with zero-padded coordinates).
		For the remainder of the network's layers, we use the previous step to compute increasingly larger powers of the norm of the vector computed in the first layer, by reconstructing the squared-norm from the inner products of each segment.
		The template vectors (and parameters) are chosen such that the square of this norm will be equal to $VV^T$.
		
		The assignment to the network weights:
		\begin{align*}
		W_{i,j}^{V,1,h} & =\frac{1}{N}\cdot \begin{cases}
		1_{i=j-d_{a}\cdot (h{-}1)} & \begin{matrix}d_{a}(h{-}1) < j\leq d_{a}(h{-}1)+\frac{d_{a}-1}{2} \\ 0 < i\leq\frac{d_{a}-1}{2}\end{matrix}\\
		\mathbf{i}\cdot1_{i=j-d_{a}\cdot (h{-}1)-\frac{d_{a}-1}{2}} & \begin{matrix}d_{a}(h{-}1)+\frac{d_{a}-1}{2} < j\leq d_{a}h-1\\0 < i\leq\frac{d_{a}-1}{2}\end{matrix}\\
		-1_{i=j-d_{a}\cdot (h{-}1)} & \begin{matrix}d_{a}(h{-}1)< j\leq d_{a}(h{-}1)+\frac{d_{a}-1}{2}\\\frac{d_a-1}{2} < i \leq d_a - 1\end{matrix}\\
		-\mathbf{i}\cdot1_{i=j-d_{a}\cdot (h{-}1)-\frac{d_{a}-1}{2}} & \begin{matrix}d_{a}(h{-}1)+\frac{d_{a}-1}{2} < j\leq d_{a}h-1\\\frac{d_{a}-1}{2}  < i \leq d_{a}-1\end{matrix}\\
		1 & j=d_{a}h,\frac{d_{a}-1}{2} < i \leq d_{a}\\
		0 & \textrm{Otherwise}
		\end{cases}\\
		W_{i,j}^{O,l,h} & =\begin{cases}
		1_{j=i-d_{a}(h{-}1)} & d_{a}(h{-}1) < i \leq d_{a}h\\
		0 & \textrm{Otherwise}
		\end{cases}\\
		\forall 1{<}l{<}L,W_{i,j}^{V,l,h} & =\begin{cases}
		1_{i=j-d_{a}\cdot (h{-}1)} & d_{a}(h{-}1) < j \leq d_{a}h\\
		0 & \textrm{Otherwise}
		\end{cases}\\
		W_{i,j}^{V,L,h} & =\mathbf{i} \cdot 1_{j=d_{a}}\\
		W_{i,j}^{K,1,h}&=W_{i,j}^{Q,1,h}=1_{i=1 \wedge j=d_{a}}\\
		W_{i,j}^{K,2,h}&=W_{i,j}^{Q,2,h}=\begin{cases}
		1_{i=j-d_{a}\cdot (h{-}1)} & \begin{matrix}d_{a}(h{-}1) < j\leq d_{a}(h{-}1)+\frac{d_{a}-1}{2}\\0<i\leq\frac{d_{a}-1}{2}\end{matrix}\\
		0 & \textrm{Otherwise}
		\end{cases}\\
		\forall l{>}2,W_{i,j}^{K,l,h}&=W_{i,j}^{Q,l,h} =\begin{cases}
		1 & i=1 \wedge j \bmod d_a \neq 0\\
		0 & \textrm{Otherwise}
		\end{cases}
		\end{align*}
		In the above, we denoted the complex root of $-1$ as $\mathbf{i}$, to differentiate it from the index $i$. The choice of template vectors:
		\begin{align*}
		x^{(i)}_j &= \begin{cases}
		V_{i,\phi(j)} & i \leq \nicefrac{M}{2} \wedge (j - 1) \bmod d_a < \frac{d_a - 1}{2} \\
		V_{i - \nicefrac{M}{2} + 1, \phi\left(j - \frac{d_a - 1}{2}\right)} & \frac{M}{2} < i \leq M \wedge \frac{d_a - 1}{2} \leq (j - 1) \bmod d_a < d_a - 1 \\
		1 & (j - 1) \bmod d_a = d_a - 1 \\
		0 & \text{Otherwise}   
		\end{cases}
		\end{align*}
		where $\phi(j) \equiv \left\lfloor \nicefrac{j - 1}{d_a} \right\rfloor \cdot (d_a - 1) + (j - 1 \bmod d_a) + 1$.
		
		W.l.o.g. we can assume that $I=\{1,\ldots,\nicefrac{N}{2}\},J=\{\nicefrac{N}{2} + 1, \ldots, N\}$. We examine the sub-matrix defined by the following indices:
		\begin{align}
		\tilde{I} &= \{(i_1,\ldots,i_{\nicefrac{N}{2}}) : 1\leq i_1 \leq \nicefrac{M}{2} \wedge \forall k > 1, i_k = M + 1\}\\
		\tilde{J} &= \{(j_1,\ldots,j_{\nicefrac{N}{2}}) : \nicefrac{M}{2} < j_1 \leq M \wedge \forall k > 1, j_k = M + 1\}
		\end{align}
		
		With all of the above in place, we are ready to prove that the resulting sub-matrix has the form of eq.~\eqref{eq:submatrix}. We begin with the output of the first self-attention layer:
		\begin{align}
		\y^{(1,i)}(\x^{(d_1)},\ldots,\x^{(d_N)})_k &= \sum_{j=1}^N \sum_{h=1}^H \left\langle W^{Q,1,h} \x^{(d_i)}, W^{K,1,h} \x^{(d_j)} \right\rangle (W^{O,1,h} W^{V,1,h} \x^{(d_j)})_k \\
		&\overset{1}{=} \sum_{j=1}^N \sum_{h=1}^H  \overbrace{x^{(d_i)}_{d_a}}^{=1} \cdot \overbrace{x^{(d_j)}_{d_a}}^{=1} (W^{O,1,h} W^{V,1,h} \x^{(d_j)})_k \\
		&\overset{2}{=} \left(\left(\sum_{h=1}^H W^{O,1,h} W^{V,1,h} \right) \left(\x^{(i_1)} + \x^{(j_1)} + (N-2)\x^{(M+1)}\right)\right)_k \\
		&\overset{3}{=} \begin{cases}
		1 & (k{-}1) \bmod d_a = d_a {-} 1 \\
		V_{i_1, \phi(k)} + \mathbf{i} V_{j_1,\phi(k)} & (k{-}1) \bmod d_a < \frac{d_a {-} 1}{2} \\
		1 {-} V_{i_1, \phi(k {-} \frac{d_a {-} 1}{2})} {-} \mathbf{i} V_{j_1,\phi(k {-} \frac{d_a {-} 1}{2})} & \textrm{Otherwise}
		\end{cases}
		\end{align}
		where $(1)$ is because $W^{Q,1,h} = W^{K,1,h}$ are matrices that are zero everywhere except for entry $(1,d_a)$, $(2)$ because when summing over the locations, only $i_1$ and $j_1$ are different from $M+1$, and $(3)$ because applying the value and output matrices on any template vector $\uu$ results in:
		\begin{align}
		\left(W^{O,1,h} W^{V,1,h} \uu \right)_k &= \sum_{\alpha=1}^{d_a} W^{O,1,h}_{k,\alpha}  \sum_{\beta=1}^{d_x} W^{V,1,h}_{\alpha,\beta} u_\beta \\
		&= \sum_{\alpha=1}^{d_a} W^{O,1,h}_{k,\alpha}  \overbrace{\begin{cases}
			u_{d_a h {+} \alpha {-} 1} {+} \mathbf{i} \cdot  u_{d_a h {+} \alpha {-} 1 {+} \frac{d_a {-} 1}{2}} & \alpha {\leq} \frac{d_a {-} 1}{2} \\
			\frac{1}{N} {-} u_{d_a h {+} \alpha {-} 1} {-} \mathbf{i} \cdot  u_{d_a h {+} \alpha {-} 1 {+} \frac{d_a {-} 1}{2}} & \frac{d_a {-} 1}{2} {<} \alpha {\leq} d_a {-} 1 \\
			\frac{1}{N} & \textrm{Otherwise}
			\end{cases}}^{\equiv \hat{u}_\alpha}\\
		&= \begin{cases}
		\hat{u}_{((k-1) \bmod d_a) + 1} & d_a (h{-}1) \leq k < d_a h \\
		0
		\end{cases}
		\end{align}
		
		At this point, notice that for any $i\in [N]$, $\y^{(1,i)}$ is the same, and we denote it with $\vv$. Note that it is a vector composed of $H$ $d_a$-dimensional sub-vectors, each composed of a $\frac{d_a - 1}{2}$-dimensional sub-vector and its complement in the next $\frac{d_a - 1}{2}$ indices, followed by a fixed value of $1$.
		
		Next, we will compute the result of the second layer, where we use the fact that every position is equal to $\vv$ to drop the reference to a specific location $i$, i.e., $\y^{(l,i)} = \y^{(l)}$:
		\begin{align}
		\y^{(2)}_k &= N \sum_{h=1}^H \left\langle W^{Q,2,h} \vv, W^{K,2,h} \vv \right\rangle (W^{O,2,h} W^{V,2,h} \vv)_k \\
		&= N \sum_{h=1}^H \left\langle \tilde{\vv}^{(h)}, \tilde{\vv}^{(h)} \right\rangle \vv^{(h)},
		\end{align}
		where we used the notation $v^{(h)}_k = v_k \cdot 1_{d_a (h{-}1) \leq k < d_a h}$, i.e., a vector that is equal to $v_k$ on the $h$'th $d_a$-dimensional segment and otherwise filled with zeros, as well as the notation $\tilde{v}^{(h)}_k = v_k \cdot 1_{d_a (h{-}1) \leq k \leq d_a (h-1) + \frac{d_a - 1}{2}}$. The last equality is because all matrices in this layer essentially just project the $d_a$-dimensional sub-vector of $\vv$ for its respective head $h$.
		
		For the third layer we get:
		\begin{align}
		\y^{(3)} &= N \sum_{h=1}^H \left\langle W^{Q,2,h} \y^{(2)}, W^{K,2,h} \y^{(2)} \right\rangle (W^{O,2,h} W^{V,2,h} \y^{(2)}) \\
		&\overset{1}{=} N\sum_{h=1}^H  \left(\sum_{r \bmod d_a \neq 0} y^{(2)}_r \right)^2  \y^{(2),h} \\
		&\overset{2}{=} N\sum_{h=1}^H  \left(N \sum_{h'=1}^H \left\langle \tilde{\vv}^{(h')}, \tilde{\vv}^{(h')}\right\rangle \right)^2  N \left\langle \tilde{\vv}^{(h)}, \tilde{\vv}^{(h)} \right\rangle v^{(h)} \\
		&\overset{3}{=} N^4 \norm{\tilde{\vv}}^4 \sum_{h=1}^H \left\langle \tilde{\vv}^{(h)}, \tilde{\vv}^{(h)} \right\rangle v^{(h)},
		\end{align}
		where we define $\tilde{\vv} = \sum_{h=1}^H \tilde{\vv}^{(h)}$. Equality $(1)$ is because in both $W^{K,3,h}$ and $W^{Q,3,h}$ on the first row is nonzero, and it has ones everywhere except in coordinates that are multiples of $d_a$, resulting in summing over all of these non-zero elements of the vector $\y^{(2)}$. Equality $(2)$ is because in the vector $\vv^{(h)}$ every entry has a corresponding entry equal to its complement, which upon summation is equal to one, leaving only the $\left\langle \tilde{\vv}^{(h')}, \tilde{\vv}^{(h')}\right\rangle$ coefficients of the vector $\y^{(2)}$. Equality $(3)$ is because \begin{align}\norm{\tilde{\vv}}^2 &= \left\langle\tilde{\vv}, \tilde{\vv}\right\rangle = \sum_{h_1,h_2} \left\langle\tilde{\vv}^{(h_1)}, \tilde{\vv}^{(h_2)}\right\rangle = \sum_{h=1}^H \left\langle \tilde{\vv}^{(h)}, \tilde{\vv}^{(h)} \right\rangle,\end{align} where the last equality stems from the fact that every $\tilde{\vv}^{(h)}$ is non-zero on a different segment of its $d_x$ coordinates.
		
		For any subsequent layer $l < L$ we use the same set of parameters, and since the input of each preceding layer has the same form of $\y^{(l)} = N^{\alpha_l} \cdot \norm{\tilde{\vv}}^{2\beta_l} \sum_{h=1}^H \left\langle \tilde{\vv}^{(h)}, \tilde{\vv}^{(h)} \right\rangle \vv^{(h)}$, then we can just compute its recurrence relation:
		\begin{align}
		\y^{(l+1)} &= N\sum_{h=1}^H  \left(N^{\alpha_l} \norm{\tilde{\vv}}^{2\beta_l} \sum_{h'=1}^H \left\langle \tilde{\vv}^{(h')}, \tilde{\vv}^{(h')}\right\rangle \right)^2 N^{\alpha_l} \norm{\tilde{\vv}}^{2\beta_l}  \left\langle \tilde{\vv}^{(h)}, \tilde{\vv}^{(h)} \right\rangle v^{(h)} \\
		&= N^{1+3\alpha_l} \norm{\tilde{\vv}}^{6\beta_l} \sum_{h=1}^H  \left(\sum_{h'=1}^H \left\langle \tilde{\vv}^{(h')}, \tilde{\vv}^{(h')}\right\rangle \right)^2 \left\langle \tilde{\vv}^{(h)}, \tilde{\vv}^{(h)} \right\rangle v^{(h)} \\
		&= N^{3\alpha_l+1} \norm{\tilde{\vv}}^{2\cdot(3\beta_l+2)} \sum_{h=1}^H \left\langle \tilde{\vv}^{(h)}, \tilde{\vv}^{(h)} \right\rangle v^{(h)}\\
		\Rightarrow \alpha_{l+1} &= 3\alpha_l +1,\beta_{l+1} = 3\beta_l + 2
		\end{align}
		Using the initial conditions of $\alpha_3 = 4$ and $\beta_3 = 2$, we get that $\alpha_l=\frac{3^{l-1}-1}{2},\beta_l = 3^{l-2}-1$. For the $L$'th layer, the only difference is that $W^{V,L,h}$ is defined such that it returns a 1-hot vector that picks the $d_a$'th element of the previous step. Putting it all together we get:
		\begin{align}
		y^{(L)}_k &=  N^{\frac{3^{L-1}-1}{2}} \cdot \norm{\tilde{\vv}}^{2\cdot (3^{l-2}-1)} \sum_{h=1}^H \left\langle \tilde{\vv}^{(h)}, \tilde{\vv}^{(h)} \right\rangle \mathbf{i} \cdot v^{(h)}_{d_a} \\
		y^{(L)}_k &=  N^{\frac{3^{L-1}-1}{2}} \cdot \mathbf{i} \cdot\norm{\tilde{\vv}}^{2\cdot3^{l-2}} 
		\end{align}
		Finally, we can evaluate $\norm{\tilde{\vv}}^2$:
		\begin{align}
		\norm{\tilde{\vv}}^2 &= \sum_{k=1}^{d_x} \tilde{v}^2_k = \sum_{h=1}^H \sum_{k=1}^{\nicefrac{d_a - 1}{2}} (V_{i_1,(d_a - 1) \cdot (h - 1) + k} + \mathbf{i} \cdot V_{j_1,(d_a - 1) \cdot (h - 1) + k})^2 \\
		&=  \overbrace{\sum_{h=1}^H \sum_{k=1}^{\nicefrac{d_a - 1}{2}} V_{i_1,(d_a - 1) \cdot (h - 1) + k}^2}^{\textrm{normalized} \Rightarrow = 1} -  \overbrace{\sum_{h=1}^H \sum_{k=1}^{\nicefrac{d_a - 1}{2}} V_{j_1,(d_a - 1) \cdot (h - 1) + k}^2}^{\textrm{normalized} \Rightarrow = 1} \\
		&\phantom{==}  2\mathbf{i}\cdot \sum_{h=1}^H \sum_{k=1}^{\nicefrac{d_a - 1}{2}} V_{i_1,(d_a - 1) \cdot (h - 1) + k} V_{j_1,(d_a - 1) \cdot (h - 1) + k} \\
		&=2\mathbf{i} (VV^T)_{i_1,j_1},
		\end{align}
		which concludes the proof.
	\end{proof}
	
	Next, we show two lemmas that aid in the proof of the lower bound. We first quote an identity by which for a matrix with entries that are polynomials in $x$, if a single
	contributor to the determinant has the highest degree of $x$, then the matrix is
	fully ranked for all values of $x$ but a finite set. 
	
	\begin{lemma}\label{lemma:poly_full_rank}
		(from \cite{levine2018benefits}).
		Let $A\in\R^{N\times N}$ be a matrix whose entries are polynomials in
		$x\in\R$. In this case, its determinant may be written as
		$\det(A)=\sum_{\sigma\in S_N}sgn(\sigma)p_\sigma(x)$, where $S_N$ is the
		symmetric group on $N$ elements and $p_\sigma(x)$ are polynomials defined by
		$p_\sigma(x)\equiv\prod_{i=1}^{N} A_{i\sigma(i)}(x),~\forall{\sigma\in S_n}$.
		Additionally, let there exist $\bar{\sigma}$ such that
		$\deg(p_{\bar{\sigma}}(x)) > \deg(p_{\sigma}(x)) ~ \forall \sigma
		\neq \bar{\sigma}$. Then, for all values of $x$ but a finite set, $A$ is
		fully ranked.
	\end{lemma}
	\begin{proof}
		We show that in this case $\det(A)$, which is a polynomial in $x$ by its
		definition, is not the zero polynomial. Accordingly, $\det(A)\neq 0$ for all
		values of $x$ but a finite set. Denoting $t\equiv\deg(p_{\bar{\sigma}}(x))$,
		since $t>\deg(p_{\sigma}(x))~\forall\sigma\neq\bar{\sigma}$, a monomial of
		the form $c\cdot x^t,c\in\R \setminus \{0\}$ exists in $p_{\bar{\sigma}}(x)$ and
		doesn't exist in any $p_\sigma(x),~\sigma\neq\bar{\sigma}$. This implies
		that $\det(A)$ is not the zero polynomial, since its leading term has a
		non-vanishing coefficient $sgn(\bar{\sigma})\cdot c\neq 0$, and the lemma
		follows from the basic identity: $\det(A)\neq 0 \iff$ $A$ is fully ranked.
	\end{proof}
	The following quoted lemma, establishes a relation referred to as the \emph{vector rearrangement inequality},  which helped us ensure that our matrix of interest upholds the conditions of lemma~\ref{lemma:poly_full_rank} and is thus fully ranked. 
	\begin{lemma}\label{lemma:rearrange} (from \cite{levine2018benefits}).
		Let $\{\vv^{(i)}\}_{i=1}^{N}$ be a set of $N$ different vectors in $\R^{\bar{R}}$
		such that $\forall i\in[N],~j\in[\bar{R}]:~v^{(i)}_j\geq 0$. Then, for all
		$\sigma\in S_N$ such that $\sigma\neq\mathbb{I}_N$, where $S_N$ is the
		symmetric group on $N$, it holds that:
		\begin{equation*}
		\sum_{i=1}^N \inprod{\vv^{(i)}}{\vv^{(\sigma(i))}} < \sum_{i=1}^{N} \norm{\vv^{(i)}}^2.
		\end{equation*}
	\end{lemma}
	\begin{proof}
		We rely on theorem 368 in~\citep{hardy1952inequalities}, which implies that
		for a set of non-negative numbers $\{a^{(1)},\ldots,a^{(N)}\}$ the following
		holds for all $\sigma\in S_N$:
		\begin{equation}\label{eq:rearrange}
		\sum_{i=1}^{N}a^{(i)}a^{({\sigma(i)})}\leq\sum_{i=1}^{N}(a^{(i)})^2,
		\end{equation}
		with equality obtained only for $\sigma$ which upholds
		$\sigma(i)=j\iff a^{(i)}=a^{(j)}$. The above relation, referred to as the
		\emph{rearrangement inequality}, holds separately for each component
		$j\in[\bar{R}]$ of the given vectors:
		\begin{equation*}
		\sum_{i=1}^{N}v_j^{(i)}v_j^{(\sigma(i))}\leq\sum_{i=1}^{N}(v_j^{(i)})^2.
		\end{equation*}
		We now prove that for all $\sigma\in S_N$ such that
		$\sigma\neq\mathbb{I}_N$, $\exists \hat{j}\in[\bar{R}]$ for which the above inequality
		is hard, \ie:
		\begin{equation}\label{hard_ineq}
		\sum_{i=1}^{N}v_{\hat{j}}^{(i)}v_{\hat{j}}^{(\sigma(i))}<\sum_{i=1}^{N}(v_{\hat{j}}^{(i)})^2.
		\end{equation}
		By contradiction, assume that $\exists\hat{\sigma}\neq\mathbb{I}_N$ for
		which $\forall j \in [\bar{R}]$:
		\begin{equation*}
		\sum_{i=1}^{N}v_j^{(i)}v_j^{(\hat{\sigma}(i))}=\sum_{i=1}^{N}(v_j^{(i)})^2.
		\end{equation*}
		From the conditions of achieving equality in the rearrangement inequality
		defined in Equation~\eqref{eq:rearrange}, it holds that
		$\forall j \in [\bar{R}]:~v_j^{(\hat{\sigma}(i))}= v_j^{(i)}$, trivially
		entailing: $\vv^{(\hat{\sigma}(i))}=\vv^{(i)}$. Thus,
		$\hat{\sigma}\neq\mathbb{I}_N$ would yield a contradiction to
		$\{\vv^{(i)}\}_{i=1}^{N}$ being a set of $N$ different vectors in $\R^{\bar{R}}$.
		Finally, the hard inequality of the lemma for $\sigma\neq\mathbb{I}_N$ is
		implied from Equation~\eqref{hard_ineq}:
		\begin{equation*}
		\sum_{i=1}^N \inprod{\vv^{(i)}}{\vv^{(\sigma(i))}}
		\equiv \sum_{i=1}^N \left(\sum_{j=1}^{\bar{R}} v_j^{(i)} v_j^{(\sigma(i))}\right)
		= \sum_{j=1}^{\bar{R}} \left(\sum_{i=1}^N v_j^{(i)} v_j^{(\sigma(i))}\right)
		< \sum_{j=1}^{\bar{R}} \left(\sum_{i=1}^N (v_j^{(i)})^2 \right)
		= \sum_{i=1}^N \norm{\vv^{(i)}}^2.
		\end{equation*}
	\end{proof}
	
	\section{Proof of Proposition 1 on the separation rank symmetry}\label{sec:sep}
	
	\begin{claim}
		For any depth $L\ge1$ input size $N>1$ and output locations $i\in\left[N\right],\:p\in\left[d_{x}\right]$ The
		separation rank \wrt~balanced partitions, which obey $A\cupdot B=[N], \abs{A},\abs{B}=\nicefrac{N}{2}$, is invariant to the identity of the partition, \ie, $\forall A\cupdot B=[N], \tilde{A}\cupdot \tilde{B}=[N],~~s.t.~ \abs{A},\abs{B},|{\tilde{A}}|,|{\tilde{B}}|=\nicefrac{N}{2}$:
		\begin{equation}
		sep(y^{i, L, d_x,H,\Theta}_p;A,B)=sep(y^{i, L, d_x, H, \Theta}_p;\tilde{A},\tilde{B})	
		\end{equation}
	\end{claim} 
	
	\begin{proof}
		We will denote $A=\left(a_{1},\dots,a_{\frac{N}{2}}\right)$,$B=\left(b_{1},\dots,b_{\frac{N}{2}}\right)$,$\tilde{A}=\left(\tilde{a}_{1},\dots,\tilde{a}_{\frac{N}{2}}\right)$,$\tilde{B}=\left(\tilde{b}_{1},\dots,\tilde{b}_{\frac{N}{2}}\right)$
		and by $\pi\in S_{N}$ the unique permutation that satisfy
		\[
		\forall m\in\left[\frac{N}{2}\right]\quad\pi\left(a_{m}\right)=\tilde{a}_{m}\wedge\pi\left(b_{m}\right)=\tilde{b}_{m}
		\]
		
		w.l.o.g we will assume that $a_{1}=\tilde{a_{1}}=i$.
		
		Assuming that $sep(y;A,B)=R$, then there exist $g_{1},\dots,g_{R},g'_{1},\dots,g'_{R}$
		s.t.
		\[
		\forall \x^{\left(1\right)},\dots,\x^{\left(N\right)}\in\mathbb{\mathbb{R}}^{d_{x}}\quad y_{p}^{i,L,d_{x},H,\Theta}\left(\x^{\left(1\right)},\dots,\x^{\left(N\right)}\right)=\sum_{v=1}^{R}g_{v}\left(\x^{\left(a_{1}\right)},\dots,\x^{\left(a_{\frac{N}{2}}\right)}\right)g'_{v}\left(\x^{\left(b_{1}\right)},\dots,\x^{\left(b_{\frac{N}{2}}\right)}\right)
		\]
		$i=\pi\left(a_{1}\right)=a_{1}$ therefore the summations over $j_{1},\dots,j_{N}$
		in eq.~\eqref{eq:gl_explicit_form} implies that for any $x^{\left(1\right)},\dots,x^{\left(N\right)}\in\mathbb{\mathbb{R}}^{d_{x}}$
		we have 
		\[
		y_{p}^{i,L,d_{x},H,\Theta}\left(\x^{\left(1\right)},\dots,\x^{\left(N\right)}\right)=y_{p}^{i,L,d_{x},H,\Theta}\left(\x^{\left(\pi\left(1\right)\right)},\dots,\x^{\left(\pi\left(N\right)\right)}\right)
		\]
		And therefore 
		\begin{align*}
		& =\sum_{v=1}^{R}g_{v}\left(\x^{\left(\pi\left(a_{1}\right)\right)},\dots,\x^{\left(\pi\left(a_{\frac{N}{2}}\right)\right)}\right)g'_{v}\left(\x^{\left(\pi\left(b_{1}\right)\right)},\dots,\x^{\left(\pi\left(b_{\frac{N}{2}}\right)\right)}\right)\\
		& =\sum_{v=1}^{R}g_{v}\left(\x^{\left(\tilde{a_{1}}\right)},\dots,\x^{\left(\tilde{a}_{\frac{N}{2}}\right)}\right)g'_{v}\left(\x^{\left(\tilde{b}_{1}\right)},\dots,\x^{\left(\tilde{b}_{\frac{N}{2}}\right)}\right)
		\end{align*}
		So we proved that
		\[
		sep(y_{p}^{i,L,d_{x},H,\Theta};\tilde{A},\tilde{B})\leq sep(y_{p}^{i,L,d_{x},H,\Theta};A,B)
		\]
		Finally by switching the roles of $\tilde{A},\tilde{B}$ and $A,B$
		we can get the inverse inequality so we conclude that 
		\[
		sep(y_{p}^{i,L,d_{x},H,\Theta};\tilde{A},\tilde{B})=sep(y_{p}^{i,L,d_{x},H,\Theta};A,B)
		\]
	\end{proof}
	\section{Experimental details}
	We conducted the network training described in section~\ref{sec:exp} of the main text with Adam optimizer for $1M$ steps and a batch size of $512$ sequences of $128$ tokens. All experiments used a learning rate schedule with a $12000$ step linear warm-up followed by a cosine decay to zero. 
	In order to increase width without changing other architectural parameters, we kept the number of heads per layer constant at $2$ (experimental evidence indicates that many heads per layer are not crucial~\citep{michel2019sixteen,kaplan2020scaling}, as does our theoretical analysis which shows that the number of heads per layer affects the separation rank logarithmically).
	
	Table~\ref{tab:widths} shows the per-depth widths of the trained architecture. More experiments were conducted per adjacent depth pairs in order to identify the transition point accurately, and reduce the error bars in figure~\ref{fig:exponential}. Table~\ref{tab:1} details the different standard deviation of repeating the training and evaluation experiment $3$ times per the given architectures.

\begin{table}[h]
	\begin{center}
		\vspace*{-3.5cm}\begin{tabular}{@{}ccccccc@{}}
			\toprule
			\multicolumn{1}{l}{L=6} & \multicolumn{1}{l}{L=12} &
			\multicolumn{1}{l}{L=18} & \multicolumn{1}{l}{L=24} &
			\multicolumn{1}{l}{L=30} & \multicolumn{1}{l}{L=36} &
			\multicolumn{1}{l}{L=48} \\ \midrule
			128                     & 88                       & 
			-                       & 64					   &
			-						& -						   & 
			44						\\
			168                     & 120                      & 
			-                       & 88					   &
			-						& -				   	       & 
			60						\\
			184                     & 130                      & 
			-                       & -					   &
			-						& -				   	       & 
			-						\\
			192                     & 136                      & 
			-                       & -					   &
			-						& -				   	       & 
			-						\\
			200                     & 142                      & 
			-                       & -					   &
			-						& -				   	       & 
			-						\\
			208                     & 148                      & 
			116                       & -					   &
			-						& -				   	       & 
			-						\\
			216                     & 152                      & 
			124                      & 104					   &
			-						& -					       & 
			72						\\
			220                     & 156                      & 
			-                       & -						   &
			-						& -					       & 
			-						\\
			224                     & 158                      & 
			130                      & 112					   &
			-						& -					       & 
			80						\\
			236                     & 168                      & 
			-                       & 88					   &
			-						& -				   	       & 
			60						\\
			248                     & 176                      & 
			144                     & 128					   &
			-						& -						   & 
			88						\\
			272                     & 192                      & 
			-                       & 136					   &
			-						& -					       & 
			96						\\
			296                     & 208                      & 
			-                       & 144					   &
			-						& -					       & 
			104						\\
			320                     & 224                      & 
			184                     & 160					   &
			144						& 128					   & 
			112						\\
			376                     & 264                      & 
			216                     & 184					   &
			168						& 152					   & 
			128						\\
			-	                    & 272                      & 
			244                     & -						   &
			-						& -						   & 
			-						\\
			-	                    & 280                      & 
			228                     & -						   &
			-						& -						   & 
			-						\\
			408                     & 288                      & 
			232                     & 200					   &
			176						& 160					   & 
			144						\\
			-	                    & 296                      & 
			240                     & -						   &
			-						& -						   & 
			-						\\
			-	                    & 304                      & 
			248                     & -						   &
			-						& -						   & 
			-						\\
			-	                    & 308                      & 
			252                     & -						   &
			-						& -						   & 
			-						\\
			-	                    & 314                      & 
			256                     & -						   &
			-						& -						   & 
			-						\\
			456                     & 320                      & 
			264                     & 224					   &
			200						& 184					   & 
			160						\\
			-	                    & 328                      & 
			268                     & -						   &
			-						& -						   & 
			-						\\
			-	                    & 240                      & 
			278                     & -						   &
			-						& -						   & 
			-						\\
			496                     & 352                      & 
			288                     & 248					   &
			224						& 200					   & 
			176						\\
			568                     & 400                      & 
			320                     & 280					   &
			248						& 232					   & 
			200						\\
			- 		                & -                 	   & 
			360                     & 312					   &
			-						& -						   & 
			-						\\
			680                     & 480                      & 
			384                     & 336					   &
			304						& 272					   & 
			240						\\
			- 		                & -                 	   & 
			406                     & 352					   &
			-						& -						   & 
			-						\\
			- 		                & -                 	   & 
			416                     & 360					   &
			-						& -						   & 
			-						\\
			- 		                & -                 	   & 
			424                     & 368					   &
			-						& -						   & 
			-						\\
			- 		                & -                 	   & 
			434                     & 376					   &
			-						& -						   & 
			-						\\
			- 		                & -                 	   & 
			440                     & 376					   &
			-						& -						   & 
			-						\\
			- 		                & -                 	   & 
			448                     & 388					   &
			-						& -						   & 
			-						\\
			- 		                & -                 	   & 
			456                     & 396					   &
			-						& -						   & 
			-						\\
			- 		                & -                 	   & 
			464                     & 402					   &
			-						& -						   & 
			-						\\
			816                     & 576                      & 
			472                     & 408					   &
			368						& 336					   & 
			288						\\
			960                     & 680                      &
			560                     & 480					   &
			432						& 392					   & 
			336						\\
			1088                    & 768                      & 
			624                     & 544					   &
			484						& 440					   & 
			384						\\
			-	                    & -	                       & 
			-	                    & 552					   &
			494						& -						   & 
			-						\\
			-	                    & -	                       & 
			-	                    & 560					   &
			504						& -						   & 
			-						\\
			-	                    & -	                       & 
			-	                    & 568					   &
			508						& -						   & 
			-						\\
			-	                    & -	                       & 
			-	                    & 576					   &
			512						& -						   & 
			-						\\
			-	                    & -	                       & 
			-	                    & 584					   &
			522						& -						   & 
			-						\\
			-	                    & -	                       & 
			-	                    & 592					   &
			530						& -						   & 
			-						\\
			-	                    & -	                       & 
			-	                    & 600					   &
			536						& -						   & 
			-						\\
			1416                    & 1000                     & 
			816                       & 704					   &
			632						& 576					   & 
			496						\\
			-	                    & -	                       & 
			-	                    & -						   &
			712						& 648					   & 
			-						\\
			-	                    & -	                       & 
			-	                    & -						   &
			760						& 696					   & 
			-						\\
			-	                    & -	                       & 
			-	                    & -						   &
			808						& 736					   & 
			-						\\
			-	                    & -	                       & 
			-	                    & -						   &
			840						& 768					   & 
			-						\\
			-	                    & -	                       & 
			-	                    & -						   &
			896						& 816					   & 
			-						\\
			2128                    & 1504                     & 
			1232                    & 1064					   &
			952						& 872					   & 
			752						\\
			-	                    & -	                       & 
			-	                    & -						   &
			992						& 904					   & 
			-						\\
			2832                    & 2000                     & 
			-                       & 1416					   &
			1264					& 1160					   & 
			1000					\\ \bottomrule
		\end{tabular}\vspace{1em}\caption{The widths $d_x$ of the different trained networks. In order to improve the estimation of the data points and their empirical error for the fit in section 5.2.2, we performed dense measurements around potential transition points. }\label{tab:widths}\vspace*{-3.5cm}
	\end{center}
\end{table}

	\begin{table}[h]
		\begin{center}
					\begin{subtable}{\linewidth}\centering
				{\begin{tabular}{ccc}
						\toprule
						\multicolumn{1}{l}{$d_x=320$} & \multicolumn{1}{l}{$d_x=680$} &
						\multicolumn{1}{l}{$d_x=800$} \\ \midrule
						1.92E-03 & 2.06E-03 & 6.51E-04 \\
				\end{tabular}}
				\caption{L = 6}\label{tab:1a}
			\end{subtable}
		\begin{subtable}{\linewidth}\centering
	{\begin{tabular}{cccc}
			\toprule
			\multicolumn{1}{l}{$d_x=224$} & \multicolumn{1}{l}{$d_x=400$} &
			\multicolumn{1}{l}{$d_x=680$} &
			\multicolumn{1}{l}{$d_x=1000$} \\ \midrule
			2.08E-03 & 1.65E-03 & 1.33E-03 & 1.20E-03 \\
	\end{tabular}}
	\caption{L = 12}\label{tab:1b}
\end{subtable}
\begin{subtable}{\linewidth}\centering
	{\begin{tabular}{cccc}
			\toprule
			\multicolumn{1}{l}{$d_x=160$} & \multicolumn{1}{l}{$d_x=280$} &
			\multicolumn{1}{l}{$d_x=480$} &
			\multicolumn{1}{l}{$d_x=704$} \\ \midrule
			7.36E-04 & 1.02E-03 & 1.48E-03 & 7.76E-04 \\
	\end{tabular}}
	\caption{L = 24}\label{tab:1c}
\end{subtable}
			\caption{The standard deviation of the test loss for networks of varying widths and depths, when repeating the training and evaluation experiment $3$ times per point.}\label{tab:1}
		\end{center}
	\end{table}

\subsection{Fit details}\label{app:fit_tests}
The estimated experimental transition points between the two depth-efficiency regimes that were collected according to the procedure described in section~5.2.2 are given in table~\ref{tab:data}.
For the linear fit we set $x_i$ to be the depth and $y_i$ the log of the estimated width at the measured transition point (with an empirical error $\sigma_i$ calculated as in table~\ref{tab:data}).  

The 
$\chi^{2}_{red}$ measure is calculated by:
\begin{align}\label{eq:chi}
\chi^{2}&=\sum_{i=1}^{n}{\frac{(y_i-\hat{y_i})^{2}}{\sigma_i^{2}}}\\\nonumber
\chi^{2}_{red}&=\frac{\chi^{2}}{df}
\end{align}	
where $\hat{y_i}$ is the predicted value of the $i$-th sample according to the fitting function given by the fit parameters $a$ and $b$ in eq.~\ref{eq:fit_params}, and  $df=n-m=3$ is the number of observations $n=5$ minus the number of fitted parameters $m=2$.
The attained value of $\chi_{\textrm{red}}^2=0.854$ 
indicates a good fit for such a low $n$, though hinting at a slight overestimation of the empirical errors. This may arise due to the limitations in attaining very dense measurements around the transition points (though as can be seen in table~\ref{tab:widths}, we made an effort to sample the loss densely around the transitions).

The 
$R^{2}$ measure is calculated by:
\begin{align}\label{eq:r_squared}
R^{2}=1-\left(\frac{\sum_{i=1}^{n}{\left(\frac{1}{\sigma_i^{2}}(y_i-\hat{y_i})^{2}\right)}}{\sum_{i=1}^{n}{\left(\frac{1}{\sigma_i^{2}}(y_i-\bar{y})^{2}\right)}}\right)
\end{align}	
where  $\bar{y}=\frac{1}{n}\sum_{i=1}^{n}{\left(\frac{1}{\sigma_i^{2}}y_i\right)}$. The attained value of $R^{2}=0.998$ indicates a good linear fit of the data.

\begin{table}[t]
	\begin{subtable}{.4\linewidth}\centering
		\begin{tabular}{||c c c||} 
			\hline 
			$L$ & $d_x$ & $\Delta d_x$ \\ [0.5ex] 
			\hline\hline
			6 & 214 & 6 \\ 
			\hline
			12 & 308 & 12 \\
			\hline
			18 & 436 & 20 \\
			\hline
			24 & 572 & 12 \\
			\hline
			30 & 824 & 16 \\
			\hline
		\end{tabular}
		\caption{The identified transition points.}\label{tab:data_a}
	\end{subtable}%
	$\Delta\left(\log d_x\right)=\frac{1}{d_x}\Delta d_x$
	\begin{subtable}{.4\linewidth}\centering
		\begin{tabular}{||c c c||} 
			\hline
			$L$ & $\log d_x$ & $\Delta\left(\log d_x\right)$ \\ [0.5ex] 
			\hline\hline
			6 & 5.37 & 2.80e-2 \\ 
			\hline
			12 & 5.73 & 3.90e-2 \\
			\hline
			18 & 6.08 & 4.59e-2 \\
			\hline
			24 & 6.35 & 2.10e-2 \\
			\hline
			30 & 6.71 & 1.94e-2 \\
			\hline
		\end{tabular}
		\caption{Log space conversion for the linear fit.}\label{tab:data_b}
	\end{subtable}
	\caption{The collected depth-efficiency regime transition widths per depth.}
	\label{tab:data}
\end{table}

\end{document}